\documentclass{article}

\usepackage[final]{ewrl_2025}
\usepackage{comment}

\usepackage[utf8]{inputenc} % allow utf-8 input
\usepackage[T1]{fontenc}    % use 8-bit T1 fonts
\usepackage{hyperref}       % hyperlinks
\usepackage{url}            % simple URL typesetting
\usepackage{booktabs}       % professional-quality tables
\usepackage{amsfonts}       % blackboard math symbols
\usepackage{nicefrac}       % compact symbols for 1/2, etc.
\usepackage{microtype}      % microtypography
\usepackage{xcolor}         % colors
\usepackage{command}

\usepackage{command}
\usepackage{tikz}
\usetikzlibrary{calc, positioning, shapes, patterns}

\usepackage{thmtools}
\usepackage{amsthm}

\newtheorem{theorem}{Theorem}
\newtheorem{proposition}[theorem]{Proposition}
\newtheorem{lemma}[theorem]{Lemma}

\newtheorem{definition}[theorem]{Definition}
\newtheorem{assumption}[theorem]{Assumption}

\usepackage{thmtools,thm-restate}
\usepackage[export]{adjustbox}  % For valign option
\usepackage{dsfont}

\title{Decentralized Online Convex Optimization with Unknown Feedback Delays}

\author{%
  Hao Qiu \\
  Università degli Studi di Milano\\
  \texttt{hao.qiu@unimi.com} \\
  \And
  Mengxiao Zhang \\
  University of Iowa \\
  \texttt{mengxiao-zhang@uiowa.edu} \\
  \AND
  Juliette Achddou \\
  UMR 9189 - CRIStAL, Université de Lille, CNRS, Inria, Centrale Lille \\
  \texttt{juliette.achdou@gmail.com} \\
}

\begin{document}

\maketitle

\begin{abstract}
Decentralized online convex optimization (D-OCO), where multiple agents within a network collaboratively learn optimal decisions in real-time, arises naturally in applications such as federated learning, sensor networks, and multi-agent control. In this paper, we study D-OCO under unknown, time- and agent-varying feedback delays. While recent work has addressed this problem~\citep{nguyen2024handling}, existing algorithms assume prior knowledge of the total delay over agents and still suffer from suboptimal dependence on both the delay and network parameters. 
To overcome these limitations, we propose a novel algorithm that achieves an improved regret bound of $\widetilde{\mathcal{O}}\Big(N\sqrt{\dtot} + \frac{N\sqrt{T}}{(1-\sigma_2)^{1/4}}\Big)$, where $T$ is the total horizon, $\dtot$ denotes the average total delay across agents, $N$ is the number of agents, and $1 - \sigma_2$ is the spectral gap of the network. Our approach builds upon recent advances in D-OCO~\citep{wan2024nearly}, but crucially incorporates an adaptive learning rate mechanism via a decentralized communication protocol. This enables each agent to estimate delays locally using a gossip-based strategy without the prior knowledge of the total delay. 
We further extend our framework to the strongly convex setting and derive a sharper regret bound of $\order\Big(\frac{N\delta_{\max}\ln T}{\alpha} + \frac{N\ln(N)\ln(T)}{\alpha\sqrt{1-\sigma_2}}\Big)$, where $\alpha$ is the strong convexity parameter and $\delta_{\max}$ is the maximum number of missing observations averaged
over agents. We also show that our upper bounds for both settings are tight up to logarithmic factors. Experimental results validate the effectiveness of our approach, showing improvements over existing benchmark algorithms.
\end{abstract}

\section{Introduction}\label{sec:intro}
Decentralized online convex optimization (D-OCO) provides a powerful framework for distributed learning systems where multiple agents collaboratively optimize a global objective while processing local data streams. Specifically, in D-OCO, agents make sequential decisions based on local information and coordinate through peer-to-peer communication networks without relying on a central coordinator. This paradigm has become increasingly important in modern applications including federated learning~\citep{kairouz2021advances}, wireless sensor networks \citep{hosseini2013online, akbari2015distributed}, real-time control systems \citep{lesage2020dynamic}, and multi-agent robotic systems~\citep{liu2018multiagent}, where centralized processing is either infeasible due to communication constraints or undesirable due to privacy concerns.

While immediate feedback is ideal, in practical distributed systems, local delays are ubiquitous and stem from factors such as fluctuating connectivity reliability, varying processing and computation times across heterogeneous devices, queuing latency in congested network links, or even delays introduced by human-in-the-loop feedback.

These delays can significantly degrade learning performance and raise fundamental challenges for algorithm design. While the impact of delays has been extensively studied in centralized online learning settings~\citep{weinberger2002delayed,joulani2013online}, the interplay between decentralization and delayed feedback introduces unique complexities that remain less understood. Several works have considered delays in decentralized settings, but most assume either bounded time-invariant \citep{CaoB22} or known delays \citep{nguyen2024handling}, which fail to capture the uncertainty and variability encountered in real-world systems. For example, in sensor networks, each node may incur delays both when acquiring measurements and when processing data \citep{rabbat2004distributed, olfati2007distributed}.
Recently, \citet{nguyen2024handling} made progress by proposing a decentralized algorithm that handles arbitrary delays in D-OCO. However, their approach suffers from two limitations: (i) it requires prior knowledge of the total delay to set the learning rate appropriately, which is usually unavailable in practice, and (ii) even with this knowledge, their regret bounds suffer from suboptimal dependencies on both the total delay and network-dependent parameters. This raises a fundamental question: 
\begin{center}
\emph{Can we design decentralized online learning algorithms that adapt to unknown, time- and agent-varying delays while maintaining near-optimal regret guarantees?}
\end{center}

In this paper, we answer this question affirmatively by developing novel decentralized online learning algorithms that achieve improved regret bounds under unknown, agent- and time-varying feedback delays. Specifically,
\begin{itemize}
    \item For general convex losses, we derive an algorithm that achieves a regret bound of $\widetilde{\mathcal{O}}\big(N\sqrt{\dtot} + \frac{N\sqrt{T}}{(1-\sigma_2)^{1/4}}\big)$, where $\dtot$ denotes the average total delay across agents, $N$ is the number of agents, $T$ is the time horizon, and $1 - \sigma_2$ is the spectral gap of the communication network.\footnote{A formal definition of the communication network is introduced in Section Preliminary. We use $\otil(\cdot)$ to hide logarithmic factors of $N$ and $T$.} Our algorithm is inspired by the recent advance in D-OCO~\citep{wan2024nearly} but with an important adaptive learning rate mechanism combined with a decentralized communication protocol, where agents use gossip-based strategies to locally estimate delays without centralized coordination or prior knowledge of the total delay. Comparing to the results in \citet{nguyen2024handling} whose regret bound is no better than $ \mathcal{O}\left( \frac{N^2}{(1 - \sigma_2)^2}\sqrt{ \dtot  }+  \frac{ \sqrt {N^3T}}{1 - \sigma_2} \right) $, our result not only improves upon the regret bound dependency on $N$ and $\sigma_2$ but also eliminates the need for prior knowledge of delays.\footnote{We also remark that \citet{nguyen2024handling} requires $\beta$-smoothness for the loss functions for all agents, which is not assumed in our work.} We further complement with a $\Omega\big(N\sqrt{\dtot}+\frac{N\sqrt{T}}{(1-\sigma_2)^{1/4}}\big)$ lower bound, demonstrating that our algorithm's regret dependencies on $N, \dtot$, $T$, and $1-\sigma_2$ are tight up to logarithmic factors. 
    \item We then consider the case where the loss functions are all strongly convex, and extend our framework to derive regret bounds of ${\mathcal{O}}\left(  \frac{N}{\alpha} \delta_{\max} \ln{T} + \frac{N \ln N\ln T}{\alpha \sqrt{1-\sigma_2}}\right)$, where $\alpha$ is the strong convexity parameter and $\delta_{\max}$ is the maximum number of missing observations averaged over agents, showing that strong convexity enables improved regret guarantee under D-OCO with delayed feedback. We also provide a matching lower bound to show that our obtained guarantees are tight up to logarithmic factors. We remark again that our algorithm does not require the knowledge of the total delay.
    \item Finally, we implement extensive experiments on various network structures and loss functions, demonstrating superior empirical performances of our proposed algorithms comparing to existing baselines.
\end{itemize}

\section{Related Works}

\paragraph{Decentralized online convex optimization} D-OCO is a framework in which multiple agents cooperatively solve an online optimization problem over a network, without relying on a central coordinator. Early foundational work in decentralized optimization focused on offline settings, leveraging techniques from gossip algorithms --- originally used to achieve consensus to enable distributed optimization \cite{boyd2011distributed,Nedic2009}.
The first formal treatment of the online counterpart was given by \cite{hosseini2013online}, who analyzed a dual averaging algorithm and established sublinear regret guarantees. Specifically, they showed that a regret bound of $\order(N^{5/4}\sqrt{T}/(1-\sigma_2)^{1/2})$ is achievable, where $\sigma_2$ is the second highest singular value of the communication matrix $W$, whose definition is shown in later sections. Since then, various algorithmic approaches have been developed, including decentralized mirror descent \cite{Shahrampour2018}, for which a similar regret rate is provable  and accelerated gossiping for D-OCO \cite{wan2024nearly}. The method from \citep{wan2024nearly} notably improves the previous regret bound by a factor of $(1-\sigma_2(W))^{-1/4} N^{1/4}/{\sqrt {\log(N)}} $. The D-OCO framework has seen various extensions, including work on settings with dynamic networks \cite{hosseini2016online,lei2020online}. For a comprehensive overview of such developments, we refer the reader to the recent monograph by \citet{yuan2024multi}. 

\paragraph{Online learning with delayed feedbacks} Our work is closely related to the literature on online learning with delayed feedback, initiated by \citet{weinberger2002delayed}. They considered the setting with uniform, known per-round delays and proposed a general reduction to non-delayed online learning. Subsequent studies extended these results to handle non-uniform delays~\citep{joulani2013online}. Various aspects of delayed feedback have been explored, including adaptive regret guarantees~\citep{DBLP:conf/aaai/JoulaniGS16}, diverse delay structures~\citep{gatmiry2024adversarial,baron2025nonstochastic,ryabchenko2025capacity}, and limited-feedback scenarios~\citep{bandit-pmlr-v49-cesa-bianchi16,bandit-cella2020stochastic,zimmert2020optimal,mdp-lancewicki2022learning,bandit-pmlr-v195-hoeven23a}.

\paragraph{D-OCO with delayed feedbacks}
In D-OCO with local feedback delays, agents receive the gradient of their decision after a certain lag. 
For settings involving time-invariant but agent-specific delays, \citet{CaoB22} proposed an online decentralized gradient descent algorithm, accommodating such delays for both convex and strongly convex loss functions. Meanwhile, \citet{mao2025online} studied online distributed convex optimization under delayed feedback within unbalanced, time-varying communication graphs. Additionally, \citet{xiong2023distributed,xiong2023event} considered D-OCO and its bandit counterpart with event-triggered communications and delayed feedback. For the more challenging setting with time- and agent-varying delays, \citet{nguyen2024handling} introduced a projection-free approach; however, their method relies on prior knowledge of the cumulative delay to appropriately set the learning rate. Beyond local feedback delays, communication delay is also considered in the literature. For example, \citet{tsianos2012distributed} analyzed distributed optimization under fixed communication delays.

\section{Preliminary}\label{sec:pre}
Throughout this paper, we denote the set $\{1, 2,\dots,m\}$ for some positive integer $m$ by $[m]$ and let $\mathbf{1}$ be an all-one vector in an appropriate dimension. For a vector $v\in\R^m$, denote its $i$-th entry by $v(i)$ and for a matrix $M\in\R^{m\times n}$, denote its $(i,j)$-th entry by $M(i,j)$. In this section, we introduce the preliminary of our problem.

\paragraph{Protocol}
In our model of decentralized online convex optimization, agents are organized in a communication network defined by a connected and undirected  graph $G=(V, E)$. The node set $V=[N]$ corresponds to the $N$ agents, and $E$ denotes the set of edges indicating permissible communication among agents. We use $V$ and $[N]$ interchangeably throughout the paper. Each agent \( u \in V \) is associated with an arbitrary and unknown sequence of \textit{local loss functions} \( f_1(u,\cdot), f_2(u,\cdot), \ldots f_T(u, \cdot) \) decided by an adversary, where \( f_t(u, \cdot) : \calX\subseteq\R^n\to \R \) for $t\in[T]$ has a bounded feasible domain and is $L$-Lipschitz with respect to $\ell_2$ norm. 
\begin{assumption}[Bounded domain]\label{asm:bounded} The common decision space $\X \subseteq \mathbb R^n$ is convex and closed.  Let \( D = \sup_{x, y \in \X} \|x - y\|_2 \) be the diameter of \( \X \) and $\mathbf{0}\in \calX$.
\end{assumption}

\begin{assumption}[Lipschitzness]\label{asm:Lipschitz} For every \( t \in [T] \), we assume that \( f_t(u, \cdot) \) is convex and \( L \)-Lipschitz with respect to \(\|\cdot\|_2\) for all $u\in V$.
\end{assumption}

The learning protocol of D-OCO with time- and agent-varying feedback delays is defined as follows. The interaction between the agents and the environment proceeds in $T$ rounds. At each round $t$, each agent $u\in V$ selects an action $x_t(u)\in\calX$ simultaneously and suffers a loss $f_t(u,x_t(u))$. For each agent $u$, instead of observing the gradient $\nabla f_t(u,x_t(u))$ immediately in the standard OCO setting, agent $u$ observes this gradient information at the end of round $t+d_t(u)$. Without loss of generality, we assume that \( t + d_{t}(u) \leq T \), for all \(u \in V,~ t \in [T] \) since any feedback received at round \( T \) will never be used in the learning process. In addition, here we consider the \emph{anonymous delayed feedback} setting where the agent does not know the time stamp of the received gradient. After receiving feedback, each agent shares the information it received with its neighbors in $G$. Each agent's goal is to minimize  their regret defined as follows, which is in terms of the \emph{global loss function} $\sum_{v \in V}f_t(v,x)$ :
\begin{align}
    \Reg_T(u) \triangleq \max_{x \in \mathcal{X}}\Big(\sum_{t=1}^T\sum_{v\in V} (f_t(v,x_t(u)) - f_t(v,x))\Big).
\end{align}
We also define $\Reg_T\triangleq\max_{u\in V}\Reg_T(u)$.

It remains to introduce how agents communicate their information with each other in this network. Specifically, following previous works of D-OCO~\citep{yan2012distributed,hosseini2013online, wan2024nearly}, we consider a gossip mechanism, or more specifically, an accelerated one defined as follows. This mechanism is defined by a communication matrix $W$ constructed based on $G$.
\begin{definition}\label{def:com}
    A matrix $W\in[0,1]^{N\times N}$ is a valid communication matrix with respect to $G=(V,E)$ if $W$ satisfies that (i) $W(u,v)=0$ if $u \neq v$ and $(u,v) \notin E$;  W is symmetric and doubly-stochastic meaning that (ii) $W(u,v)\ge 0, ~ \forall u,v \in  V$; (iii) $ W(u,v) = W(v,u), ~ \forall u,v \in  V$ (iv) $\sum_{v\in V} W(u,v)=1, ~\forall u \in V$. Consequently, a valid communication $W$ is positive semi-definite with $0\leq\sigma_2(W)<1$ where $\sigma_2(W)$ is the second-largest eigenvalue of $W$.
\end{definition}
A typical construction of this matrix is as follows: 
\begin{align}
    W = I_N - c\cdot\text{Lap}(G),\label{eqn:def_w}
\end{align}
where $I_N\in\R^{N\times N}$ denotes the identity matrix and \( \text{Lap}(G) \) denotes the Laplacian of the graph \( G \) with $\text{Lap}(G)(i,i)=\deg(i)$ for all $i\in V$, $\text{Lap}(G)(i,j)=-1$ if $i\neq j, (i,j)\in E$, and $\text{Lap}(G)(i,j)=0$ if $i\neq j, (i,j)\notin E$. \( c \) is a certain constant such that \( 0<c \leq 1/\sigma_1(\text{Lap}(G)) \), with \( \sigma_1(\text{Lap}(G))\) being the largest eigenvalue of the Laplacian $\text{Lap}(G)$. In particular, building row $W(u, \cdot)$ defined in \eref{eqn:def_w} only requires knowing agent $u$'s direct neighbors.

Based on this communication matrix $W$, whose $u$-th row is given to each agent $u$ at the beginning of the learning process, the gossip communication process is defined as follows. Suppose there are $N$ vectors $\{x(u)\}_{u\in V}$ for each agent where $x(u)\in\R^n$ represents the information agent $u$ wants to communicate. In the context of D-OCO, this information can correspond to various quantities such as predictions \citep{Shahrampour2018} or loss gradients \citep{hosseini2013online}. In order to approximate the averaged vector $\bar{x}=\frac{1}{N}\sum_{u\in V}x(u)$, \citet{liu2011accelerated} considers the following accelerated gossip process:
\begin{align}
    x^{k+1}(u) = (1+\theta)\sum_{v\in\calN_u}W(u,v) x^k(v) - \theta x^{k-1}(u), \label{eq:matrix_iterations}
\end{align}
for $k\geq 0$ where $x^{0}(u)=x^{-1}(u)=x(u)$ for all $u\in V$, $\N_u = \braces{v: (u,v)\in E}\cup \braces{u}$ the set of neighbors of $u$ according to $G$, and $\theta>0$ is the mixing coefficient. Let $X^k\in \R^{N\times n}$ be a concatenation of $\{x^k(u)\}_{u\in V}$ and $\bar{X}=\bar{x}\mathbf{1}^\top$. \citet{ye2023multi} shows that $X^k$ converges to $\bar{X}$ in a linear rate.
\begin{proposition}[Proposition 1 in \citet{ye2023multi}]\label{lemma: acc_gossip}
    The iterations of \eqref{eq:matrix_iterations} with $\theta=\left(1+\sqrt{1-\sigma_2^2(W)}\right)^{-1}$ ensure that
    \begin{align*}
    \left\|X^k-\bar{X}\right\|_F \leq \sqrt{14} b^k \left\|X^0-\bar{X}\right\|_F
\end{align*}
for any $k\in \mathbb{N}$, where $b = \left(1-(1-1/\sqrt{2})\sqrt{1-\sigma_2(W)}\right)$ and $\|\cdot\|_F$ denotes the Frobenius norm of a matrix. 
\end{proposition}

\paragraph{Other Notations} Let $\mathbf{0}$ be an all-zero vector in an appropriate dimension. For each agent $u\in V$, define set $o_t(u) = \{\tau\in \mathbb{N} : \tau+d_\tau(u) < t\} \subseteq [t-1]$ to be the set of rounds for agent $u$ whose gradients are observed before round $t$, and let $m_t(u) = [t-1] \setminus o_t(u)$ be the set of rounds for agent $u$ whose observation is yet to be received at the beginning of round $t$. Define $\delta_{\max}=\max_{\substack{t \in [T]}}\frac{1}{N}\sum_{u \in V}|m_t(u)|$ to be the maximum number of per-round missing observations averaged over all agents and $\dtot = \frac{1}{N}\sum_{t\in[T]}\sum_{u \in V} d_t(u)$ to be the total delay averaged over all agents.

\section{D-OCO with General Convex Loss Functions}
\label{sec: D-OCO with General Convex Loss Functions}

In this section, we study the setting where the loss functions for each agent at each round are convex. We first consider the case where the total delay \( \dtot \) is known and propose an algorithm that achieves an \( {\widetilde{\mathcal{O}}}\left(N\sqrt{\dtot} + \frac{N\sqrt{T}}{(1-\sigma_2(W))^{1/4}}\right) \) regret guarantee. We then extend this approach to the more realistic case where \( \dtot \) is unknown, using a specific adaptive learning rate tuning. Finally, we provide a lower bound of \( \Omega\left(N\sqrt{\dtot} + \frac{N\sqrt{T}}{(1-\sigma_2(W))^{1/4}}\right) \), showing that our upper bound is tight in its dependence on \( T \) , $\dtot$, and \( 1-\sigma_2(W) \).

\subsection{Non-Adaptive Algorithm with Known Total Delay}
\label{sec: Non-Adaptive Algorithm with Known Total Delay}
\begin{algorithm}[ht]
\caption{Accelerated Decentralized Follow the Regularized Leader with Delayed Feedback (AD-FTRL-DF) for Agent $u$.}
\label{alg:dist_da_cvx_acc}

\textbf{Initialize:} $x_1(u) =z_1^{-1}(u) = z_1^0(u)=\mathbf{0}$.

\begin{algorithmic}
\For{$s = 1, 2, \dots, T/B $}

    \State Define $\calT_s = \{(s-1)B+1, \dots, sB\}$
    \For{$t \in\calT_s$}.
        
        \State Play $x_s(u)$ and set  $k \leftarrow t-(s-1)B-1$.
        \State Update $z_{s}^{k+1}(u) $ using accelerated gossiping:
            \begin{align} 
                z_{s}^{k+1}(u) &= (1 + \theta) \sum_{v \in V} W(u,v)z_{s}^{k}(v)~ \nonumber\\
                & \quad - \theta z_{s}^{k-1}(u). \label{eq:accgo}
                \end{align}%\EndIf
         \State Send $z_{s}^{k+1}(u) $ to every neighbor $v\in \mathcal N_u$.
    \EndFor
    \State Compute $x_{s+1}(u)$ for next block as follows:
    \begin{align}\label{eq:ftrl}  x_{s+1}(u) = \argmin_{x \in \mathcal{X}} \langle z_{s}^B(u), x \rangle + \frac{1}{\eta_{s+1}(u)} \| x \|_2^2.\end{align}
    \State Aggregate gradients observed during the block: 
    \[y_s(u) = \sum_{\tau \in o_{sB+1}(u) \backslash o_{(s-1)B+1}(u)} g_{\tau}(u),\] 
    with $g_\tau(u) \triangleq \nabla f_\tau(x_{s(\tau)}(u))$, $s(\tau)$ is the block $\tau$ lies in.
    \State Compute $z_{s+1}^{-1}(u)$ and $z_{s+1}^{0}(u)$ for next block:
     \begin{align*}
        z_{s+1}^{-1}(u) &= z_{s}^{B-1}(u) + y_{s}(u),\\
     z_{s+1}^0(u) &= z_{s}^B(u) + y_{s}(u).\end{align*}
    \EndFor
\end{algorithmic}
\end{algorithm}
 
When the total delay is known, our algorithm is built upon the algorithm proposed in~\citet{wan2024nearly}, whose idea is to incorporate the accelerated gossiping process into a blocking update mechanism to estimate the gradient of the global loss function. Specifically, the algorithm operates in blocks of size \( B \). Without loss of generality, we assume that $T/B$ is an integer such that each block contains exact $B$ time steps. Following~\citet{wan2024nearly}, within each block \( s\in [T/B] \), every agent \( u \) uses a fixed decision \( x_s(u) \) and iteratively updates an auxiliary variable \( z_s^{k+1}(u) \) using the accelerated gossip procedure defined in \eref{eq:accgo}. From a high level, $z_s^{k+1}(u)$ aims to approximate the gradient of the global loss function collected from all previous epochs. 
The parameters \( \theta \) and \( B \) are chosen based on the spectral gap of the communication matrix \( W \), specifically:
\begin{equation}
\theta = \frac{1}{1 + \sqrt{1 - \sigma_2^2(W)}}, \quad
B = \left\lceil \frac{\sqrt{2} \ln(N\sqrt{14N})}{(\sqrt{2} - 1)\sqrt{1 - \sigma_2(W)}} \right\rceil.
\label{eq:choice_B_theta}
\end{equation}

After completing all iterations within block $s$, each agent updates her decision for the next block by solving a Follow-the-Regularized-Leader problem \eref{eq:ftrl} with learning rate $\eta_s(u)$. Then, different from~\citet{wan2024nearly} which aggregates the received gradient within this block, due to the feedback delay, we compute $y_s(u)$ which only aggregates all gradients \( g_\tau(u) \) received during block \( s \). This is formalized through the difference set \( o_{sB+1}(u) \setminus o_{(s-1)B+1}(u) \), which captures newly received gradients within the block. Finally, we compute the first two iterates of the subsequent block using the prior iterates and the aggregated gradient \( y_s(u) \). In the absence of delay, our algorithm exactly recovers the algorithm proposed in \citet{wan2024nearly}. 

The pseudo code of our algorithm is formally shown in \alref{alg:dist_da_cvx_acc} and the following theorem shows that our algorithm achieves $\order(N\sqrt{\dtot}+N\sqrt{T}/(1-\sigma_2(W))^{1/4})$ when $\eta_s(u)$ is fixed over all blocks and is dependent on $\dtot$.

\begin{restatable}{theorem}{fstupperbound}
Assume each agent \( u \in {V} \) runs an instance of \alref{alg:dist_da_cvx_acc} with a valid communication matrix \( W \), parameters $\theta$ and $B$ defined in \eref{eq:choice_B_theta}, and a fixed learning rate
\begin{equation}\label{eq:fixedeta}
    \eta_s(u) = \eta = \frac{D}{L \sqrt{\dtot + B T}},
    \quad \forall\, s \in [T/B].
\end{equation}
Then, under \asref{asm:bounded} and \ref{asm:Lipschitz}, the regret is bounded as
    \begin{equation*}
        \Reg_T  = {\mathcal{O}}\Big( DLN \Big(\sqrt{\dtot} + \frac{ \sqrt{T\ln N}}{\left(1 - \sigma_2(W)\right)^{1/4}} \Big)\Big).
    \end{equation*}
\label{th:fupperbound}
\end{restatable}

Two remarks are as follows. First, note that \citet{nguyen2024handling} considered the exact same case where $\dtot$ is known and obtain a regret bound no better than $\mathcal{O}\left( \frac{N^2}{(1 - \sigma_2)^2}\sqrt{ \dtot  }+  \frac{N \sqrt N}{1 - \sigma_2} \sqrt T\right)$. Comparing to their results, our result not only achieves a better dependency on the spectral gap $1 - \sigma_2(W)$ and the number of agents $N$, but also shows that the effects of the delay and those of the network topology can be \emph{decoupled}. Specifically, the portion of the regret that does not depend on the delay scales with \( N / {(1 - \sigma_2(W))}^{1/4} \sqrt T \) in \tref{th:fupperbound} instead of \( {N\sqrt N}/{(1-\sigma_2(W))}\sqrt T \) in their bound. For the delay related term, our bound \emph{does not} depend on the spectral gap \( 1-\sigma_2(W) \) while theirs suffer from a suboptimal $1/(1-\sigma_2(W))^2$ dependency. 
Specifically, our result also improves the dependency on $N$ upon the $\order(N\sqrt{\dtot}+\frac{N^{1.5}\sqrt{T}}{1-\sigma_2(W)})$ achieved by~\citet{CaoB22}, where 
delays are time-invariant and agent-specific, i.e, $d_t(u)=d(u)$ for all $t\in[T]$. Moreover, our upper bound matches the lower bound up to logarithmic factors, as will be shown later. In addition, our bound also recovers the regret bound proven in \citet{wan2024nearly} when $d(u)=0$ for all $u\in V$.

% Second, we save an additional $N^{1/4}$ factor when delays are time-invariant and agent-specific. i.e when $d_t(u)=d(u)$ for all $t\in[T]$. 

\subsubsection{Proof Sketch} The full proof of \tref{th:fupperbound} is deferred to the Appendix~\ref{app: Non-Adaptive Algorithm with Known Total Delay} and we introduce the proof sketch in this section. With some calculation we decompose the regret for agent $u$ as follows:
\begin{align*}
    &\Reg_T(u) \leq  \underbrace{\sum_{s=1}^{T/B} \sum_{t\in \mathcal{T}_s}  \sum_{v\in [N]} \inner{g_t(v), \bar{x}_s - x^*}}_{\spadesuit} \\
    &~~~+ BL\underbrace{\sum_{s=1}^{ T/B }\sum_{v\in [N]} \order\left(\|x_s(v) -\bar{x}_s\|_2+\|x_s(u) -\bar{x}_s\|\right)}_{\clubsuit},
\end{align*}
where $\bar{x}_s=\argmin_{x\in\calX}\{\langle\sum_{v\in V}\sum_{\tau\in o_{(s-1)B+1}(v)}g_\tau(v),x\rangle+\frac{N}{\eta}\|x\|_2^2\}$
denotes the FTRL decision assuming that agent $u$ receives all agents' gradients that have been observed up to time $t$ and we use $\eta$ to represent $\eta_s(u)$ since $\eta_s(u)$ is fixed over all agents and blocks. Intuitively, $\spadesuit$ accounts for the regret incurred by the agent if she only suffers from the delayed feedback, while \( \clubsuit\) accounts for the regret incurred due to the communication among the network.

To bound $\spadesuit$, following a classic analysis in online learning with delayed feedback, we further split $\spadesuit$ into the regret of the decision assuming no feedback delay and the distance between the decisions with and without feedback delay. With some rather standard calculations, the first part can be bounded by $\mathcal{O}(ND^2/\eta+\eta BNL^2T)$ while the second term can be bounded by $\order(\eta NL^2(\dtot+BT))$.%

To bound $\clubsuit$, we analyze the effect of gossip-based averaging. While agents can not locally receive the true global gradient, using accelerated gossip, the disagreement between local and average quantities decays exponentially in $B$ as shown by \propref{lemma: acc_gossip}. Specifically, we show that for any agent $v\in[N]$, $\sum_{s=1}^{T/B}\|x_s(v) - \bar{x}_s\|_2$ is bounded by $\order(\eta TL)$, which is the main technical part of the proof and require an involved analysis. Finally, picking $\eta$ optimally leads to our final bound.

\subsection{Adaptive Algorithm with Unknown Total Delay}
\label{sec: Adaptive Algorithm with Unknown Total Delay}
The main issue with the algorithm described above is that the learning rate choice \(\eta_s(u)\) relies on the unknown total delay \(\dtot\). To illustrate the difficulty of adaptively tuning the learning rate with respect to the total delay in D-OCO, consider the single-agent setting, where it is indeed possible to adjust the learning rate dynamically by tracking the cumulative number of the agent’s own missing observations~\citep{mcmahan2014delay, gyorgy2021adapting}. In contrast, in the decentralized setting, each agent cannot directly observe the number of gradients missed by other agents, and thus cannot directly compute the global cumulative delay. However, note that \(\dtot=\frac{1}{N}\sum_{u\in[N]}\sum_{t=1}^T |m_{t}(u)|\). Therefore, if each agent additionally communicates their own number of missing observations to others through a gossiping protocol, every agent can well estimate the total number of averaged missing observations, leading to an estimation of $\dtot$.

Specifically, each agent still runs an instance of \alref{alg:dist_da_cvx_acc} to perform the decision update and track the average gradients under delay. In addition, each agent also runs an instance of \alref{alg:adaptive_learning_rate_acc} in parallel to compute the learning rate by gossiping the number of their own missing observations with their neighbors. The algorithm is formally shown in \alref{alg:adaptive_learning_rate_acc}. From a high level, \alref{alg:adaptive_learning_rate_acc} closely mirrors the accelerated gossip routine of \alref{alg:dist_da_cvx_acc}, but instead focuses on gossiping the cumulative number of missing observations. Concretely, \alref{alg:adaptive_learning_rate_acc} still goes in blocks and updates the auxiliary variable $\zeta_s^k$ using the accelerated gossiping, which can be viewed as an approximation of the cumulative missing observations averaged till block $s-1$. The learning rate $\eta_{s+1}(u)$ is then computed by replacing the exact total delay $\dtot$ used in~\eref{eq:fixedeta} by this local estimate till block $s-1$ as shown in~\eref{eq:def_adlr}. At the end of the epoch $s$, similar to \alref{alg:dist_da_cvx_acc}, we update the first two iterates $\zeta_{s+1}^{-1}$ and $\zeta_{s+1}^{-1}$ of the subsequent block by adding the number of missing observations at the end of block $s$ to $\zeta_s^{B-1}$ and $\zeta_s^{B}$. This finishes our algorithm for adaptive learning rate tuning. Each agent $u$ is then supposed to run \alref{alg:dist_da_cvx_acc}~alongside \alref{alg:adaptive_learning_rate_acc} (with the same $\theta$ and $B$ described in \eref{eq:choice_B_theta}) to use $\eta_s(u)$ computed in \eref{eq:def_adlr} to update $x_{s+1}(u)$. The following theorem shows that with this adaptive learning rate tuning, we achieve $\otil(N\sqrt{\dtot}+N\sqrt{T}/(1-\sigma_2(W))^{1/4})$ without knowing $\dtot$.

\begin{algorithm}[ht]

\textbf{Initialize:} $\eta_{1}(u) = \frac{D}{L\sqrt{BT+ 3B^2}}, ~\zeta_1^{-1}(u) = \zeta_1^0(u)=0$.

\begin{algorithmic}
        \caption{Accelerated Gossip Routine for the Adaptive Learning Rate for Agent $u$}
    \label{alg:adaptive_learning_rate_acc}
    \For{$s = 1, 2, \dots, T/B$} 
    \For{$t = (s-1)B+1, \dots, sB$}
        \State  $k\leftarrow t-(s-1)B-1$.
        \State Update  $\zeta_{s}^{k+1}(u) $ using accelerated gossiping:
            \begin{align*} 
                \zeta_{s}^{k+1}(u) &= (1 + \theta) \sum_{v \in V} W(u,v)\zeta_{s}^{k}(v)- \theta z_{s}^{k-1}(u). 
                \end{align*}
        \State Send $\zeta_{s}^{k+1}(u) $ to every neighbor $v\in \mathcal{N}_u$.
        \EndFor
        \State Count missing observations at the end of the block 
        \[q_s(u) = |m_{sB+1}(u)|.\]
        \vspace{-1em}
        \State Update 
		\begin{align}
                \eta_{s+1}(u) = 
        \frac{D}{L \sqrt{BT+B\cdot\zeta_s^B(u)+3sB^2}}. \label{eq:def_adlr}
        \end{align}        
        \State Compute first iterates for next block:
     \begin{align*}
        \zeta_{s+1}^{-1}(u) &= \zeta_{s}^{B-1}(u) + q_{s}(u),\\
     \zeta_{s+1}^0(u) &= \zeta_{s}^B(u) + q_{s}(u).
     \end{align*}
    \EndFor
 \end{algorithmic}
\end{algorithm}

\begin{restatable}{theorem}{cvxadaptthm}\label{th:cvxadaptthm}
Assuming each agent \( u \in [N] \) runs an instance of \alref{alg:adaptive_learning_rate_acc} with a valid communication matrix \( W \) and parameters $\theta$ and $B$ defined in \eref{eq:choice_B_theta} together with an instance of \alref{alg:dist_da_cvx_acc} parametrized by the same $W$, $\theta$ and $B$ and using $\eta_{s}(u)$ computed by \alref{alg:adaptive_learning_rate_acc}. Then, under \asref{asm:bounded} and \ref{asm:Lipschitz}, the regret is bounded as
    \begin{align*}
        \Reg_T =   \widetilde{\mathcal{O}}\left( DLN \left( \sqrt{\dtot} + \frac{\sqrt{T}}{(1-\sigma_2(W))^{1/4}} \right) \right).
\end{align*}

\end{restatable}
The proof of \tref{th:cvxadaptthm} is provided in the Appendix~\ref{app: Adaptive Algorithm with Unknown Total Delay}. We emphasize that our analysis is non-trivial, which includes (i) a careful bounding on the gossip-based estimation error of the adaptive learning rate compared to the optimal rate defined with respect to $\dtot$, and (ii) a more involved analysis of the FTRL updates, particularly due to possibly non-decreasing learning rates $\eta_s(u)$.

\subsection{Lower bound}
Finally, we complement our obtained upper bounds with the following $\Omega(N\sqrt{T}/(1-\sigma_2(W))^{1/4} + N\sqrt{\dtot})$ lower bound.
\begin{restatable}{theorem}{lowerbound}
\label{th:lowerbound} 
Let \( d \) be the constant feedback delay suffered by all agents $u\in[N]$ in the network. Then, there exists a graph \( G = ([N], E) \), with $N=2(M+1)$ where $M$ is an even integer, and a sequence of $L$-Lipschitz loss functions $\left\{ f_1(1, \cdot), \ldots, f_1(N, \cdot) \right\}, \ldots, \left\{ f_T(1, \cdot), \ldots, f_T(N, \cdot) \right\}$
such that any algorithm has to suffer regret at least:
\[
\Reg_T \  = \Omega\left( D L N \left(\sqrt{T}/(1-\sigma_2(W))^{1/4} +   \sqrt{dT}\right)\right),
\]
where \( W   = I - \frac{1}{\sigma_1(\text{Lap}(G))} \cdot \text{Lap}(G) \).
\end{restatable}
Compared to this lower bound, our obtained upper bounds are optimal in the dependence on $T$, $1 - \sigma_2(W)$, and $\dtot$, though there is still a gap of polynomial factors in the number of agents $N$.
We provide a proof sketch here and the full proof is deferred to Appendix~\ref{app: Lower Bound for the general convex case}. Our proof is adapted from the construction in \citet{wan2024nearly} which considered a carefully designed problem instance where the global loss is supported on one half of the graph, while the remaining half consists of agents with identically zero local loss functions. Focusing on an agent \( u \) in the latter group, we observe that its optimization problem effectively reduces to an instance of online linear optimization (OLO) with feedback delay. The total delay experienced by agent \( u \) in this setting consists of the constant delay \( d \), combined with a graph-dependent communication delay due to the network structure. The remaining proof builds on standard lower bound analysis for centralized OLO with delayed feedback.

\section{D-OCO with Strongly-Convex Loss Functions}
\label{sec: D-OCO with Strongly-Convex Loss Functions}
In this section, we consider the case where all loss functions satisfy $\alpha$-strongly convexity defined as follows. 
\begin{assumption}[strong convexity]\label{asm:stronglyconvex}
    For every \( t \le T \) and $v \in V$, we assume that \( f_t(v, \cdot) \) is $\alpha$-strongly convex: $\forall x, y \in \mathcal{X} $
    \begin{align*}
        f_{t}(v, y) \geq f_{t}(v, x)+\left\langle\nabla f_{t}(v, x), y-x\right\rangle+\frac{\alpha}{2}\|y-x\|_2^2.
    \end{align*}
\end{assumption}
In order to show an improved regret bound when losses are strongly convex in D-OCO with feedback delay, following the algorithm proposed in \citet{wan2024nearly} for strongly convex functions, we propose our algorithm \algnamesc~outlined in \alref{alg:dist_da_cvx_acc_sc}. Compared to \algname~shown in \alref{alg:dist_da_cvx_acc}, there are two key differences.  
First, the cumulative gradient \( y_s(u) \) are replaced by \( y_s^+(u) \), which includes an additional \( -\alpha B x_s(u) \) term (\eref{eq:yplus}); second, we do not need to apply a gossip-based communication among agents to tune the learning rate adaptively but only need \( \eta_{s+1}(u) = \frac{2}{\alpha s B} \) for all $u\in[N]$. The following theorem shows that \alref{alg:dist_da_cvx_acc_sc} achieves $\order((N\delta_{\max}+{N\ln N}/{\sqrt{1-\sigma_2(W)}})
(\ln T/ \alpha))$ regret.

\begin{algorithm}[ht]
\caption{Accelerated Decentralized Follow the Regularized Leader with Delayed Feedback  under Strong Convexity (\algnamesc) for Agent $u$.} 
\label{alg:dist_da_cvx_acc_sc}
\textbf{Initialize:} $x_1(u) =z_1^{-1}(u) = z_1^0(u)=\mathbf{0}$

\begin{algorithmic}
\For{$s = 1, 2, \dots,  T/B $} 
\State \(\eta_{s+1} = \frac{2}{\alpha s B}\)
    \For{$t = (s-1)B+1, \dots, sB$}
        \State Play $x_s(u)$ and set  $k \leftarrow t-(s-1)B-1$.
        \State Update $z_{s}^{k+1}(u) $ using accelerated gossiping:
            \begin{align*} 
                z_{s}^{k+1}(u) &= (1 + \theta) \sum_{v \in V} W(u,v)z_{s}^{k}(v)- \theta z_{s}^{k-1}(u).
            \end{align*}
         \State Send $z_{s}^{k+1}(u) $ and $x_s(u)$ to every  $v\in \mathcal N_u$. 
    \EndFor
    \State Compute action $x_{s+1}(u)$:
    \begin{align*}  x_{s+1}(u) = \argmin_{x \in \mathcal{X}} \langle z_{s}^B(u), x \rangle + \frac{1}{\eta_{s+1}} \| x \|_2^2.\end{align*} 
    \State Compute augmented aggregated gradients $y_s^+(u)$:
    \begin{align}
        y_s^+(u) = \sum_{\tau \in o_{sB+1}(u) \backslash o_{(s-1)B+1}(u)} g_{\tau}(u) - \alpha B x_{s}(u).
        \label{eq:yplus}
        \end{align}
    \State Compute first iterates for next block:
     \begin{align*}
        z_{s+1}^{-1}(u) &= z_{s}^{B-1}(u) + y_s^+(u),\\
     z_{s+1}^0(u) &= z_{s}^B(u) + y_s^+(u).\end{align*}
    \EndFor
\end{algorithmic}
\end{algorithm}

\begin{restatable}{theorem}{fstupperboundsc} \label{th:fstupperboundsc}
Assume each agent \( u \in V \) runs an instance of \algnamesc~  with a valid communication matrix \( W \) and  parameters $\theta$ and $B$ defined in \eqref{eq:choice_B_theta}.
Then, under \asref{asm:bounded}, \ref{asm:Lipschitz} and \ref{asm:stronglyconvex}, the global regret is bounded as
\begin{align*}
{\mathcal{O}}\left(  \frac{N(\alpha DL + L^2)}{\alpha} \left(  \delta_{\max} + \frac{\ln(N)}{\sqrt{1-\sigma_2(W)}}\right)  \ln{\left(T\right)}\right),
    \end{align*} 
    where $\delta_{\max}=\max_{\substack{t \in [T]}}\frac{1}{N}\sum_{u\in[N]}|m_t(u)|$. Moreover, when $d_t(u)=d(u)$ for all $t\in[T]$, define $\Bar{d} \triangleq \frac{1}{N} \sum_{v \in V} d(v)$ and the global regret is bounded as
    \begin{align*}
{\mathcal{O}}\left(  \frac{N(\alpha DL + L^2)}{\alpha} \left( \Bar{d} + \frac{\ln(N)}{\sqrt{1-\sigma_2(W)}}\right)  \ln{\left(T\right)}\right).
    \end{align*}
    
\end{restatable}

The full proof is deferred to Appendix~\ref{app: D-OCO with Strongly-Convex Loss
Functions}. To our knowledge, there are no previous results for D-OCO under strongly convex losses with time- and agent-varying delays. Several remarks are as follows. First, to interpret the delay-dependent term $\delta_{\max}$, it is not hard to see that $\delta_{\max} \leq \frac{1}{N}\sum_{n\in[N]}\max_{t\in[T]}d_{t}(u)$, which is the maximum delay averaged over all agents. Following \citet{qiu2025exploiting}, we can also show that $\delta_{\max}\leq \sqrt{N\dtot}$. Second, reducing to the case where the delay is time-invariant, we achieve an improved bound compared to \citet{CaoB22}, which obtained a regret bound of $\order(\frac{N\Bar{d}}{\alpha} \ln T + \frac{N \sqrt{N}}{1-\sigma_2} \frac{\ln T}{\alpha})$. We also recover the bound proven in \citet{wan2024nearly} when $d(u)=0$ for all $u\in[N]$. Finally, in Appendix~\ref{app:strong_lower_bound}, we also provide a lower bound of $\Omega\left( ( d + 1/(1-\sigma_2(W))^{1/2})\cdot N\alpha \ln(T/d) \right)$ when $d_t(u)=d$ for all $t\in[T]$ and $u\in[N]$, and all loss functions are $\alpha D$-Lipschitz and $\alpha$-strongly convex, showing that our upper bound is tight with respect to $T$, $\delta_{\max}$ (since $\delta_{\max}=d$ in this case), $N$ and $1-\sigma_2(W)$ up to logarithmic factors.

\section{Numerical Experiments}
\begin{figure*}[t]
\centering
\includegraphics[width=0.3\textwidth]{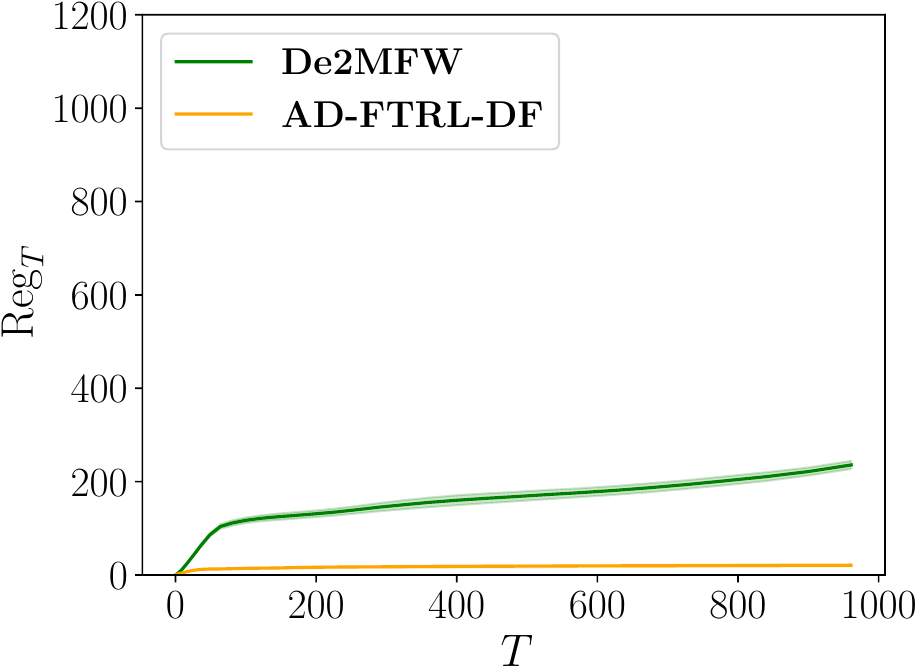}
\includegraphics[width=0.3\textwidth]{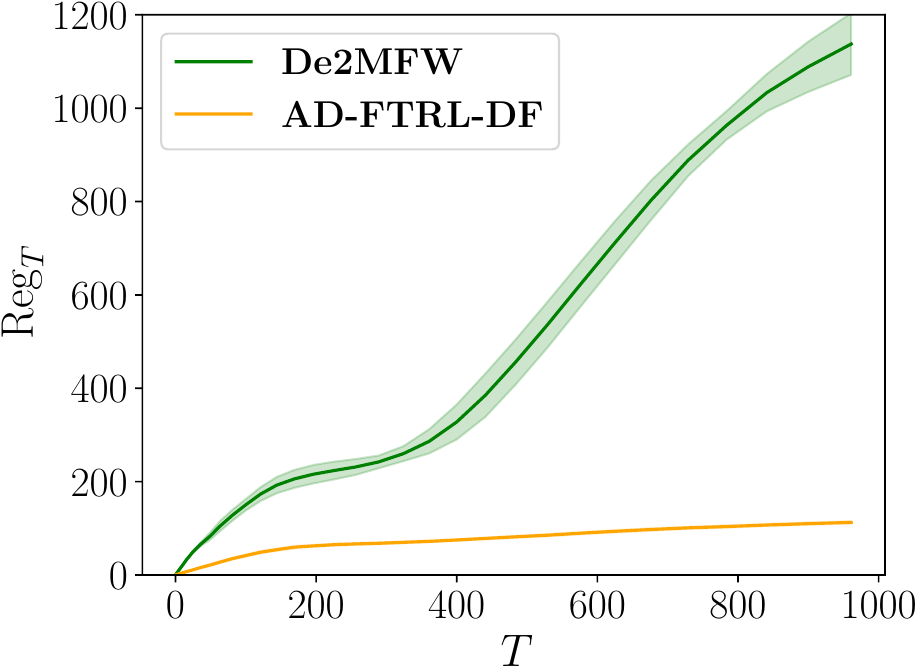}
\includegraphics[width=0.3\textwidth]{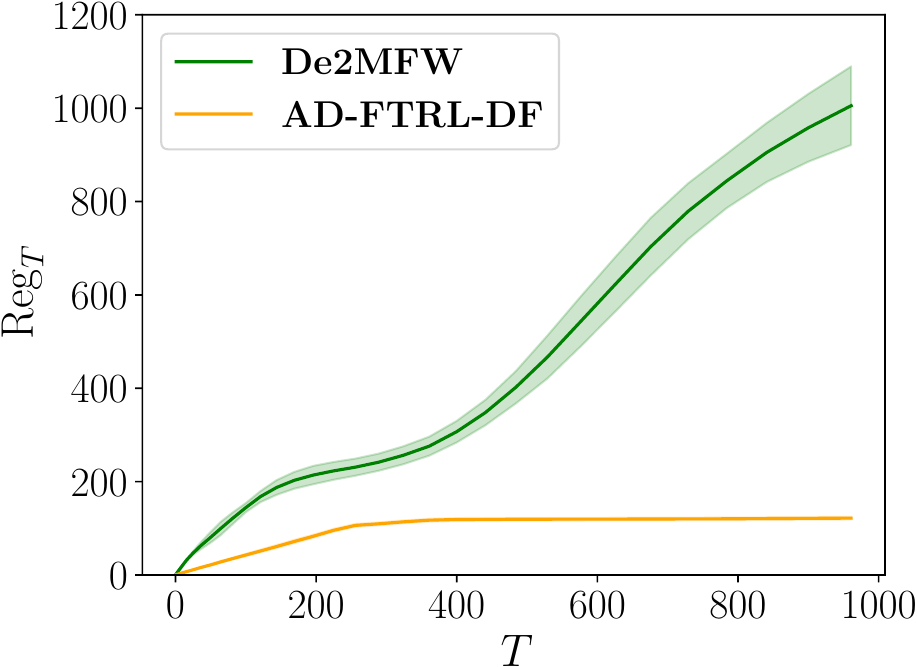}

\includegraphics[width=0.3\textwidth]{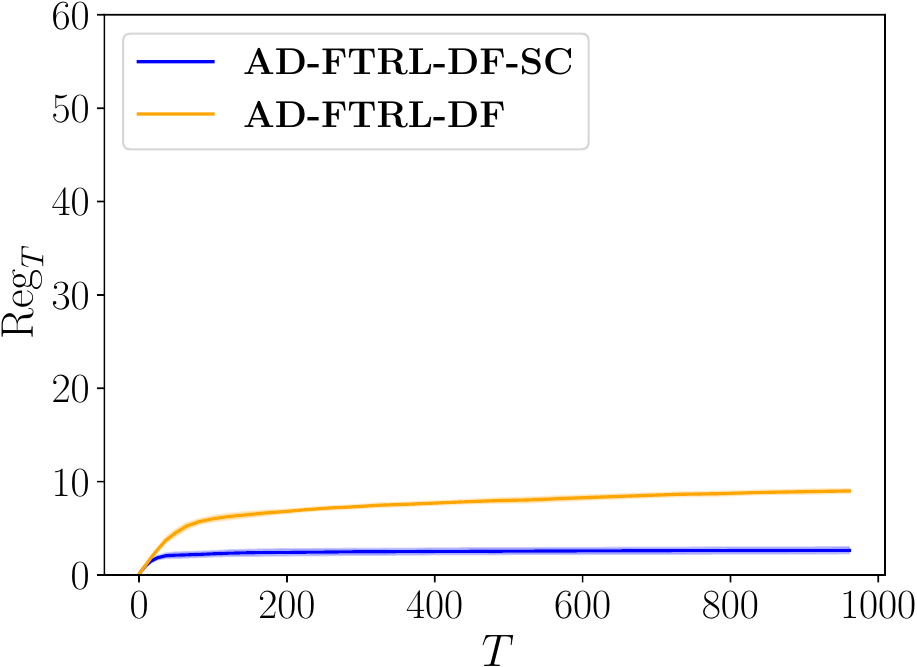}
\includegraphics[width=0.3\textwidth]{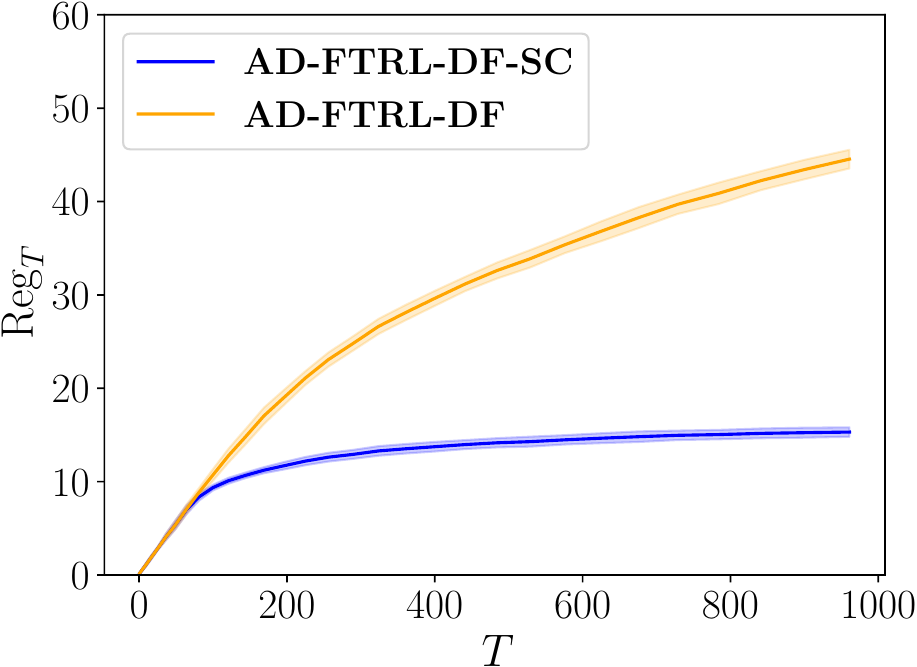}
\includegraphics[width=0.3\textwidth]{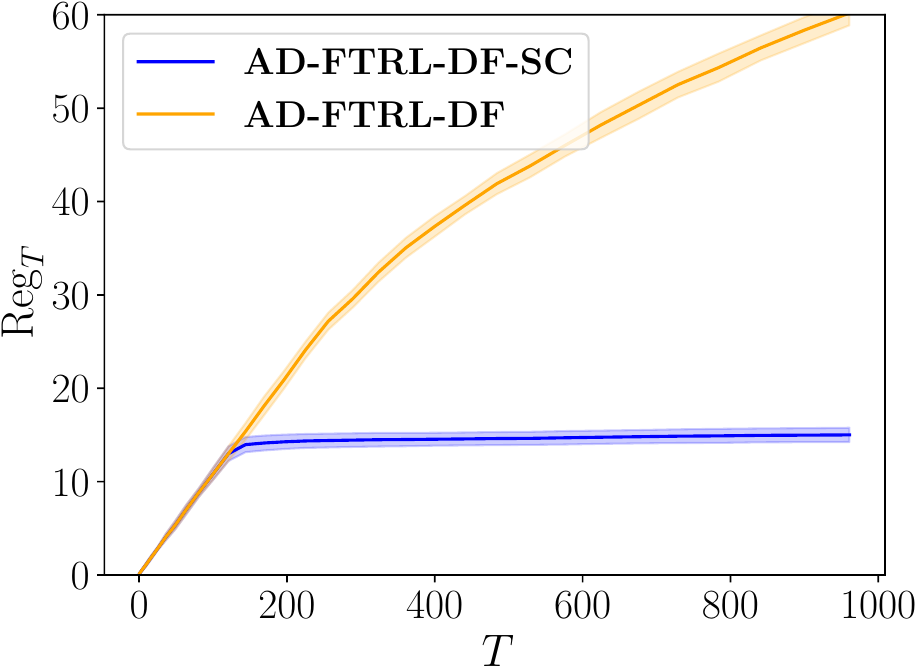}
\caption{Uniform delays. Comparison with relevant baselines across three network topologies—complete (left), grid (middle), and cycle (right)—under convex losses (top row) and strongly convex losses (bottom row).}
\label{fig:exps1}
\end{figure*}

\begin{figure*}[t]
\centering
\includegraphics[width=0.3\textwidth]{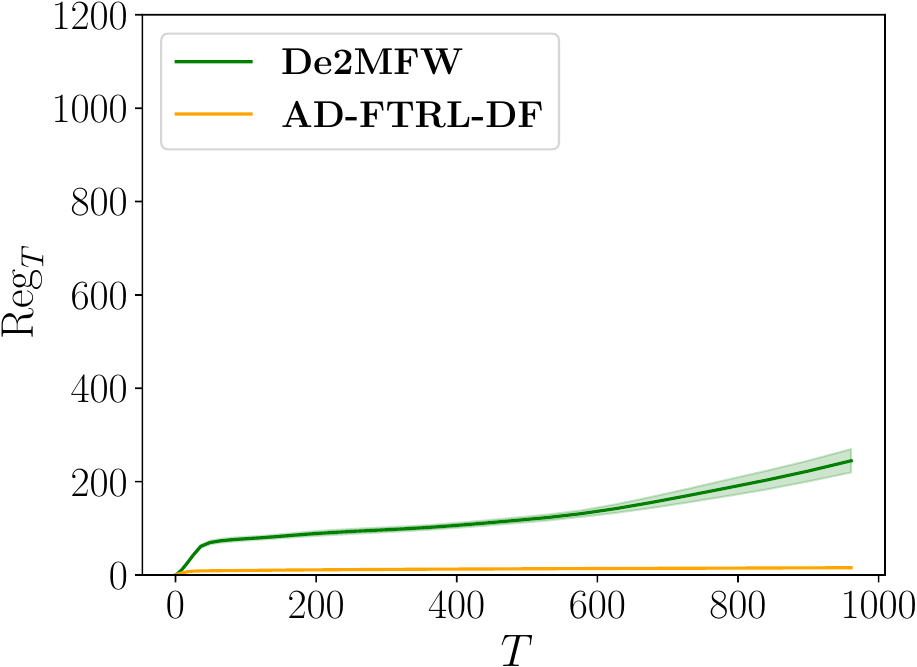}
\includegraphics[width=0.3\textwidth]{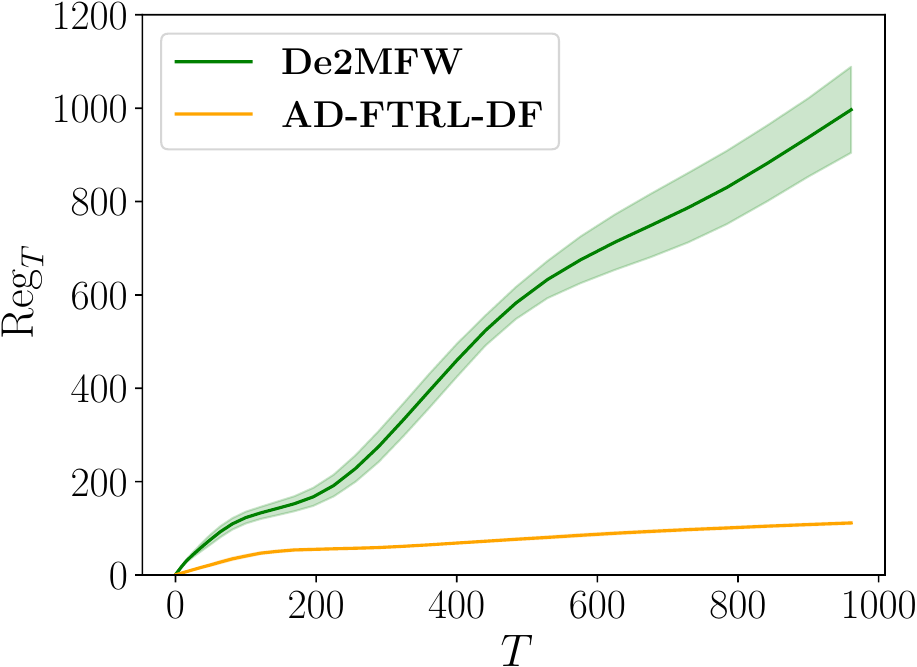}
\includegraphics[width=0.3\textwidth]{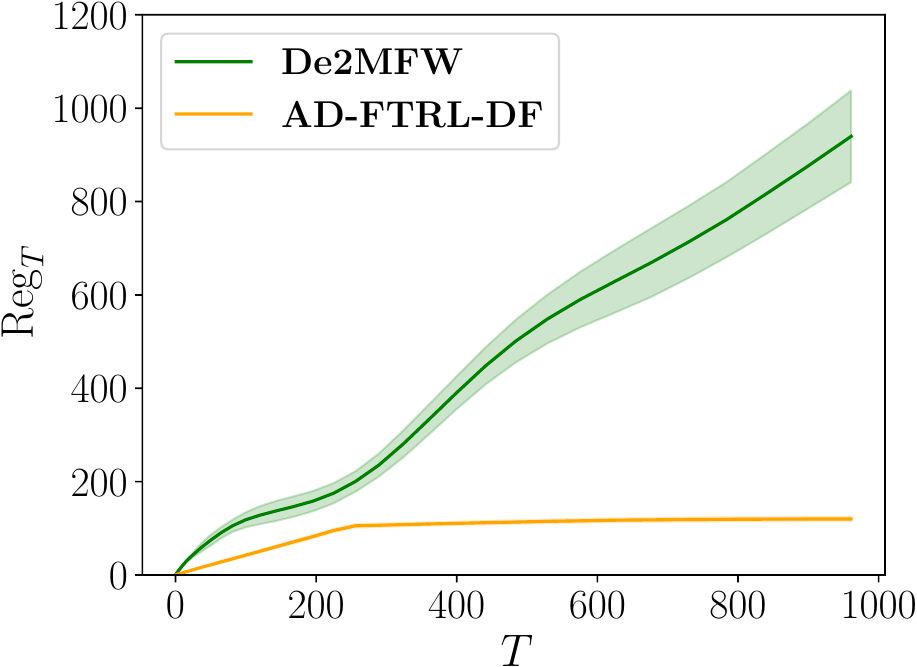}

\includegraphics[width=0.3\textwidth]{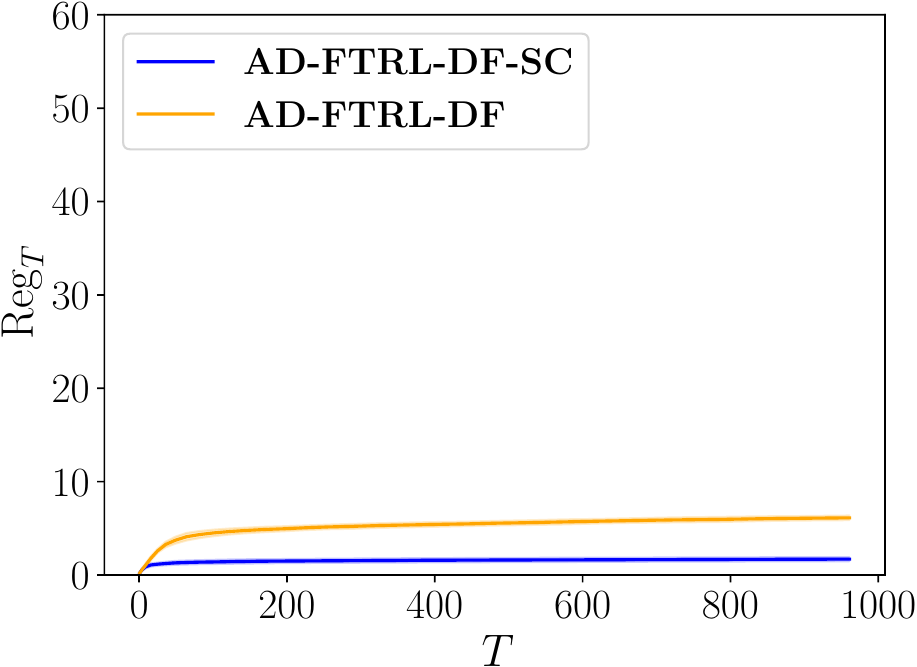}
\includegraphics[width=0.3\textwidth]{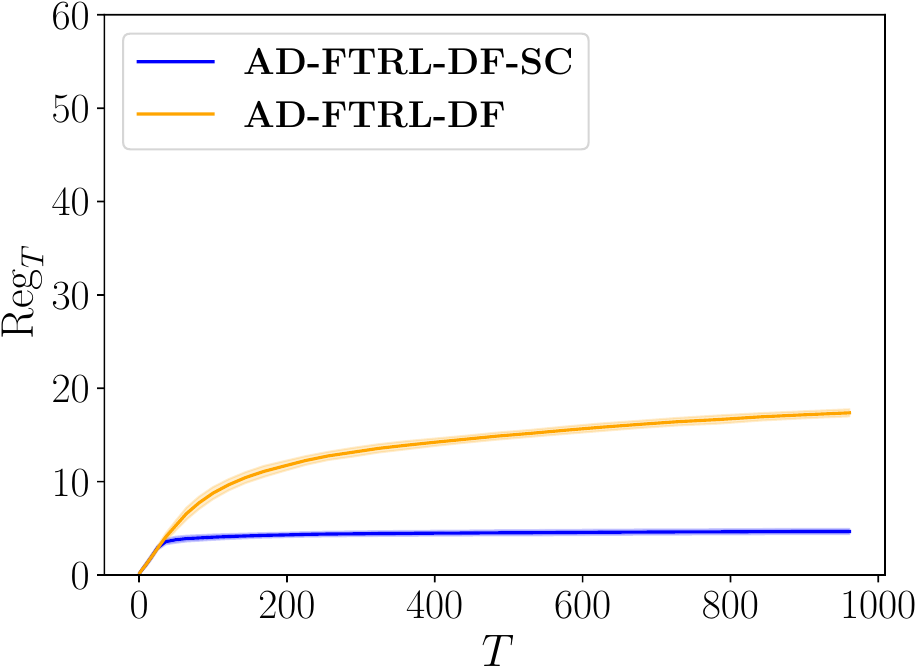}
\includegraphics[width=0.3\textwidth]{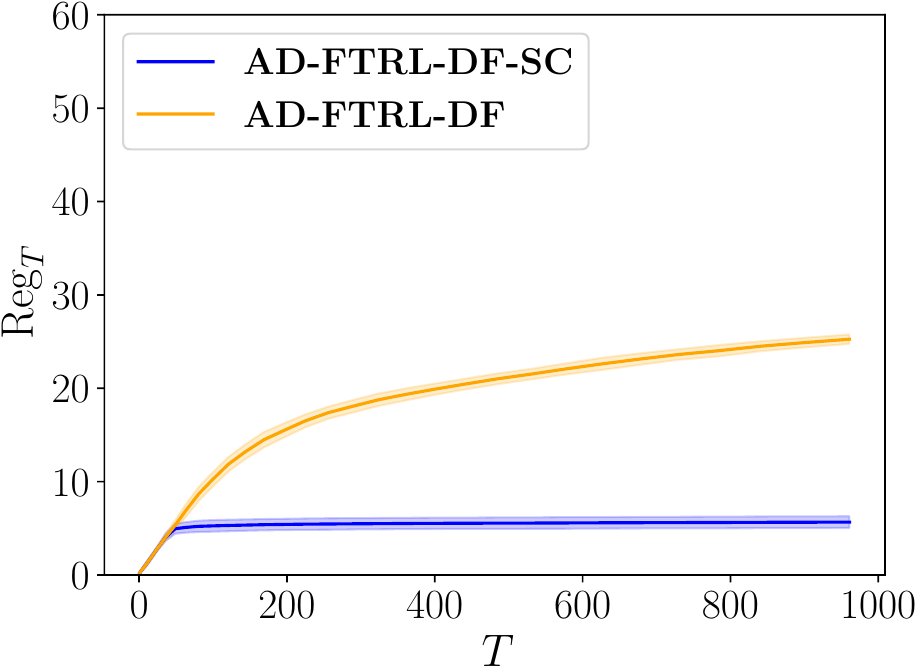}
\caption{Geometric delays. Comparison with relevant baselines across three network topologies—complete (left), grid (middle), and cycle (right)—under convex losses (top row) and strongly convex losses (bottom row).}
\label{fig:exps2}
\end{figure*}

In this section, we evaluate the performance of our proposed algorithms in the delayed D-OCO setting, using two representative sets of loss functions that capture the convex and strongly convex regimes, respectively.
 
 \textbf{Setting.}  To show the algorithms' performances under the general convex loss case, following the experiment setup used in~\citet{yuan2020distributed}, we define the local losses for all agents \(v\in V\) as
\begin{equation}
 f_t(v,x) = \tfrac{1}{2}\bigl(\langle w_t(v), x\rangle - y_t(v)\bigr)^2, \label{eq:exp_cvx}   
\end{equation}
where each feature vector \(w_t(v)\in \R^{10}\) has independent coordinates drawn uniformly from \([-1,1]\). We set the agents’ decision space to be \(\mathcal{X}=\{x\in\R^{10}, \|x\|_2\leq 2\}\). Labels are generated as follows: for \(1 \le v \leq  N/2\), \(y_t(v) = \varepsilon_t(v)\), and for the remaining agents, we have $y_t(v) = \langle w_t(v), \mathbf{1}\rangle + \varepsilon_t(v)$
with \(\varepsilon_t(v)\) being zero‑mean, unit‑variance Gaussian noise clipped to \([-1,1]\).
 For strongly convex losses, we augment each local loss with an \(\ell_2\)-regularizer:
\begin{equation}
f_t(v,x) = \tfrac{1}{2}\bigl(\langle w_t(v), x\rangle - y_t(v)\bigr)^2 + \tfrac{1}{2}\|x\|_2^2. \label{eq:exp_strcvx} 
\end{equation}
We evaluate  the performance of our algorithms and baselines on three network topologies with $N=36$ nodes ---  the \emph{complete graph}, in which all agents are connected to one another; the \emph{grid}, in which agents are organized in a two-dimensional lattice and communicate with their immediate horizontal and vertical neighbors; and the \emph{cycle}, where each agent  $v$ is connected to $v-1$ and $v+1$. We use \eref{eqn:def_w} with $c = 1/N$ to set the communication matrix $W$. Therefore, direct calculation shows that $1/(1-\sigma_2(W))^{1/4}$ associated to each of the above topologies is respectively $1$, $3.40$ and $5.87$. We consider two delayed environments. In the first, each local delay \(d_{t}(v)\) is independently and uniformly drawn from \(\{0,1,\dots,50\}\). In the second, each local delay \(d_{t}(v)\) is drawn independently from a geometric distribution with success probability $0.1$.

All experiments are conducted over \(T = 1000\) rounds. Reported results are computed by averaging the relevant performance metrics over $20$ independent runs (trials) with different random seeds.

\textbf{Baselines.} For the general convex loss setting, we compare our algorithm \algname~(\alref{alg:dist_da_cvx_acc}) with adaptive learning rate tuning (\alref{alg:adaptive_learning_rate_acc}) against De2MFW \citep{nguyen2024handling}. In the strongly convex loss setting, we compare our algorithm \algnamesc~(\alref{alg:dist_da_cvx_acc_sc}) against \algname~(\alref{alg:dist_da_cvx_acc}) with adaptive learning rate tuning.

\textbf{Results.}~\fref{fig:exps1} and~\fref{fig:exps2} present the regret curves of our algorithms and the aforementioned baselines, where the shaded regions denote the standard deviation over 20 trials. The losses are defined in \eref{eq:exp_cvx} and \eref{eq:exp_strcvx}, and results are reported for all three topologies under two delayed environments. From the plots, we observe that for the losses defined in \eref{eq:exp_cvx}, \algname~ with an adaptive learning rate substantially outperforms De2MFW across all network topologies.
In the strongly convex loss case, \algnamesc~ achieves consistently lower regret than the baseline \algname, which matches our theoretical guarantees. Comparing among different network topologies, for both convexity regimes, the regret is significantly higher with the grid and cycle graph compared to the one with the complete graph. This is consistent with the regret dependence on the reciprocal of a power of the spectral gap since the associated spectral gap for complete graph is smaller than that for grid and cycle graph.

\bibliographystyle{icml2025}

\bibliography{ref}

%%%%%%%%%%%%%%%%%%%%%%%%%%%%%%%%%%%%%%%%%%%%%%%%%%%%%%%%%%%%

\appendix

\newpage
\section{Preliminary Results}
In this section, we show several auxiliary lemmas that will be helpful throughout the paper.
\subsection{General properties of FTRL}
The following FTRL stability lemma bounds the distance between two FTRL iterates with different linear losses and possibly different regularizers.
It also shows a simplified upper bound in the case when two decisions are made by using FTRL with the same regularizer.

\begin{lemma}[Lemma A.2 of \citet{qiu2025exploiting}]
\label{lem:ftrl_sc}
    Let $\cX \subseteq \R^n$ be closed and non-empty.
    Let $A_1,A_2\succeq 0$ be two positive semidefinite matrices, $b_1,b_2\in \R^n$, and $c_1,c_2\in \R$.
    Define $\psi_1(x) = x^\top A_1x + b_1^\top x + c_1$ and $\psi_2(x) = x^\top A_2x + b_2^\top x + c_2$.
    Suppose that $z_1 \in \argmin_{x\in\calX} \left\{\inner{w_1,x}+\psi_1(x)\right\}$ and $z_2 \in \argmin_{x\in\calX}\left\{\inner{w_2,x}+\psi_2(x)\right\}$.
    Then, we have
    \begin{align*}
        \|z_1 - z_2\|_{A_1}^2 + \|z_1 - z_2\|_{A_2}^2 \leq \inner{w_1 - w_2, z_2 - z_1} + \left(\psi_1(z_2)-\psi_2(z_2)\right) - \left(\psi_1(z_1)-\psi_2(z_1)\right) \;.
    \end{align*}
    Furthermore, if $\psi_1(x)=\psi_2(x)=x^\top Ax+b^\top x+c$ with positive definite $A \succ 0$, we have
    \begin{align*}
        \|z_1-z_2\|_{A}\leq \frac{1}{2}\|w_1-w_2\|_{A^{-1}},
    \end{align*}
    where $\|x\|_A=\sqrt{x^\top Ax}$ denotes the Mahalanobis norm induced by a positive semi-definite matrix $A$.
\end{lemma}
\subsection{Basic analysis facts}
\begin{lemma}[Lemma~4.13 in \cite{orabona2019modern}]
\label{lemma: ada-grad}
    Let $a_0 \geq 0$ and let $f:[0,+\infty) \rightarrow[0,+\infty)$ be a non-increasing function. Then

\begin{equation*}
    \sum_{t=1}^T a_t f\left(a_0+\sum_{i=1}^t a_i\right) \leq \int_{a_0}^{\sum_{t=0}^T a_t} f(x) \mathrm{d} x.
\end{equation*}
\end{lemma}

\subsection{Facts on the delay}
The following lemma illustrates the relationship between the cumulative number of missing observations at the end of each block and total delay, which will be useful in later analysis.
\begin{lemma}
\label{lemma: block_total_delay}
For any $u\in V$ and any fixed integer $B>0$ with $T/B$ an integer, 
    \[B\sum_{s=1}^{T/B}  \left|m_{s B+1}(u)\right|\leq \sum_{s=1}^T d_s(u) + BT.\]
Consequently, we also have for all $s\in[T/B]$.
 \[BM_s \leq \frac{1}{N} \sum_{s=1}^T  \sum_{v \in V}d_s(u) + BT = \dtot + BT,\]
 where $M_s \triangleq \frac{1}{N} \sum_{u \in V} |m_{sB+1}(u)|$ for all $s\in[T/B]$.
\end{lemma}
\begin{proof}
Note that each gradient $g_t(u)$ that is delayed by $d_t(u)$ remains unobserved for $d_t(u)$ rounds, and therefore contributes to $|m_{kB+1}(u)|$ for exactly $\lceil d_t(u)/B \rceil$ consecutive blocks. Summing over all $t \in [T]$, we obtain that
$$
B \sum_{k=1}^{T/B} |m_{kB+1}(u)| = B \sum_{t=1}^T \left\lceil \frac{d_t(u)}{B} \right\rceil \leq \sum_{t=1}^T \left( d_t(u) + B \right) = \sum_{t=1}^T d_t(u) + BT.
$$
This proves the first inequality. To obtain the bound on $M_s$, since $
M_s = \frac{1}{N} \sum_{u \in V} |m_{sB+1}(u)|$, 
summing both sides over $s = 1, \dots, T/B$ and applying the bound above lead to
$$
B\sum_{s=1}^{T/B} M_s = \frac{B}{N} \sum_{s=1}^{T/B} \sum_{u \in V} |m_{sB+1}(u)| \leq \frac{1}{N} \sum_{u \in V} \left( \sum_{t=1}^T d_t(u) + BT \right) = \dtot + BT.
$$  
\end{proof}

\newpage

\section{Omitted Details in Section~\ref{sec: D-OCO with General Convex Loss Functions}}

\subsection{Non-Adaptive Algorithm with Known Total Delay}
\label{app: Non-Adaptive Algorithm with Known Total Delay}
In this section, we show the omitted details in Section~\ref{sec: Non-Adaptive Algorithm with Known Total Delay}. For completeness, we first restate the theorem and then present its proof. After establishing the main result, we proceed to prove several auxiliary lemmas that will be used in the algorithm’s regret analysis.
\fstupperbound*

\begin{proof}
We start the proof with some notations. We define
\begin{align}
    \bar{z}_{s-1} &\triangleq\frac{1}{N} \sum_{l=1}^{s-1} \sum_{v \in V} y_{l}(v).\label{eq:barz}
\end{align}
Direct calculation shows that $\bar{z}_{s-1}$ equals to the cumulative received gradients till block $s-1$ averaged over all agents:
\begin{align*}\bar{z}_{s-1}
    &= \frac{1}{N} \sum_{l=1}^{s-1}\sum_{v \in V} \sum_{\tau \in o_{lB+1}(v) \backslash o_{(l-1)B+1}(v)}   g_{\tau}(v) \tag{Definition of $y_{l}(v)$}
    \\&= \frac{1}{N}\sum_{v \in V}  \sum_{\tau \in o_{(s-1)B+1}(v)}  g_{\tau}(v),
\end{align*}
where the last inequality is due to $o_{1}(v)= \emptyset$ for any $v \in V$.
Then for all $ v \in V$, define
\begin{equation}
    \bar{x}_{s}(v)\triangleq \argmin_{x\in\calX} \left\langle  
    \bar{z}_{s-1}, x\right\rangle+\frac{1}{ \eta_s(v)} \| x\|_2^2 .\label{eq:defbarx}
\end{equation}
In this case, since $\eta_s(v)=\eta$ for all $s\in[T/B]$ and $v\in V$, we have $\bar{x}_s(u)=\bar{x}_s(v)$ for all $u,v\in V$ and we let $\bar{x}_s$ denote this value.
We also define
\begin{equation}
    \tilde{z}_{s-1} = \frac{1}{N} \sum_{l=1}^{s-1}\sum_{\tau \in \mathcal{T}_l}  \sum_{v \in V}  g_{\tau}(v) \label{eq:deftildaz}
\end{equation}
to be the cumulative gradients till block $s-1$ averaged over all agents assuming no delay,
where $\mathcal{T}_l=\{(l-1) B+1, \ldots, l B\}$. We also define 
\begin{align*}
F_s(x) &\triangleq \left\langle \tilde{z}_{s-1}, x \right\rangle + \frac{1}{\eta} \|x\|_2^2,
\end{align*}
and let $\Tilde{x}_s\triangleq \argmin_{x\in\calX} F_s(x)$ be the minimizer of $F_s(x)$.

With all the above notations, we apply the regret decomposition proven in \lref{lem:reg_decomp} and obtain that:
\begin{align*}\Reg_T(u) &\le\underbrace{\sum_{s=1}^{T/B} \sum_{t\in \mathcal{T}_s}  \sum_{v\in V} \inner{g_t(v), \bar{x}_s(u) - x^*}}_{\spadesuit} \nonumber
    \\& \quad+ \underbrace{2BL\sum_{s=1}^{T/B}   \sum_{v\in V} \left(\|\bar{x}_s(u) -\bar{x}_s(v)\|_2+\|x_s(v) -\bar{x}_s(v)\|_2\right)+ NBL \sum_{s=1}^{T/B}\|x_s(u) -\bar{x}_s(u)\|_2}_{\clubsuit}\\
    &=\underbrace{\sum_{s=1}^{T/B} \sum_{t\in \mathcal{T}_s}  \sum_{v\in V} \inner{g_t(v), \bar{x}_s - x^*}}_{\spadesuit} + \underbrace{2BL\sum_{s=1}^{T/B}   \sum_{v\in V} \|x_s(v) -\bar{x}_s\|_2+ NBL \sum_{s=1}^{T/B}\|x_s(u) -\bar{x}_s\|_2}_{\clubsuit},\end{align*}
where the last equality uses the fact that \(\bar{x}_s(u) = \bar{x}_s(v) = \bar{x}_{s}\) for all $u,v\in V$.

We start analyzing Term $\spadesuit$ by decomposing it as follows:
 \begin{align}
    \frac{1}{N}\spadesuit  &= \frac{1}{N}\sum_{s=1}^{T/B} \sum_{t\in \mathcal{T}_s}  \sum_{v\in V} \left\langle g_t(v), \bar{x}_{s}(u) - \Tilde{x}_{s} + \Tilde{x}_{s} -x^*\right\rangle  \nonumber
    \\&= \underbrace{\frac{1}{N}\sum_{s=1}^{T/B} \sum_{t\in \mathcal{T}_s}  \sum_{v\in V} \left\langle g_t(v), \Tilde{x}_{s} -x^*\right\rangle}_{{\operatorname{full-info}_T}} + \underbrace{\frac{1}{N}\sum_{s=1}^{T/B} \sum_{t\in \mathcal{T}_s}  \sum_{v\in V} \left\langle g_t(v), \bar{x}_{s}(u)- \Tilde{x}_{s}\right\rangle}_{{\operatorname{drift}_T}}, \label{ineq: spade_cheat_drift}
\end{align}
where $\operatorname{full-info}_T$ corresponds to the regret assuming there is no delay and $\operatorname{drift}_T$ corresponds to the regret induced by delayed feedback.

To analyze $\operatorname{full-info}_T$, since $$\Tilde{x}_s=\argmin\{\langle \frac{1}{N}\sum_{v\in V} \sum_{\tau \in \mathcal{T}_s}g_{\tau}(v), \cdot\rangle + \frac{\|x\|_2^2}{\eta}\},$$ invoking \asref{asm:bounded}, \asref{asm:Lipschitz}, and applying Corollary 7.7 in 
\citet{orabona2019modern} yields the following bound 
\begin{align}
{\operatorname{full-info}_T} &\le \frac{D^2}{\eta} + \frac{\eta BL^2 T}{2} . \label{eq:ftrlreg}
\end{align}

Now we turn to the analysis of  $\operatorname{drift}_T$ in Term $\spadesuit$. Specifically,
\begin{align}
    \operatorname{drift}_T & = \frac{1}{N}\sum_{s=1}^{T/B} \sum_{t\in \mathcal{T}_s}  \sum_{v\in V} \left\langle g_t(v), \bar{x}_{s}- \Tilde{x}_{s}\right\rangle \nonumber
    \\& \leq BL \sum_{s=1}^{T/B} \left\|\bar{x}_{s}-\Tilde{x}_{s}\right\|_2 \tag{Cauchy–Schwarz inequality and \asref{asm:Lipschitz}}
    \\& = BL \sum_{s=2}^{T/B} \left\|\bar{x}_{s}-\Tilde{x}_{s}\right\|_2 \tag{$\bar{x}_{1}=\Tilde{x}_{1}=\mathbf{0}$}
    \\& \leq \frac{\eta BL}{2} \sum_{s=2}^{T/B} \left\|\bar{z}_{s-1}-\Tilde{z}_{s-1}\right\|_2 \tag{\lref{lem:ftrl_sc}}
    \\& \leq \frac{\eta BL}{2} \sum_{s=2}^{T/B} \left\|\frac{1}{N} \sum_{v \in V} \sum_{\tau \in o_{(s-1)B+1}(v)}  g_{\tau}(v)-\frac{1}{N} \sum_{l=1}^{s-1}\sum_{\tau \in \mathcal{T}_l}  \sum_{v \in V}  g_{\tau}(v)\right\|_2
    \tag{Definition of $\bar{z}_{t-1}$ and $\tilde{z}_{t-1}$}
    \\& = \frac{\eta BL}{2} \sum_{s=2}^{T/B} \left\|-\frac{1}{N}  \sum_{v \in V} \sum_{\tau \in m_{(s-1)B+1}(v)}  g_{\tau}(v)\right\|_2
    \tag{$\mathcal{T}_s=\{(s-1) B+1, \ldots, s B\} $ and $m_{t}(v)=[t-1] \backslash o_{t}(v)$ }
    \\& \leq \frac{\eta B L^2}{2} \sum_{s=1}^{T/B} \left(\frac{1}{N}\sum_{v \in V} |m_{(s-1)B+1}(v)|\right) \label{ineq:drift_T_acc_before_eta}
\end{align}
where the last inequality is by the \asref{asm:Lipschitz}. 
Combining \eref{ineq: spade_cheat_drift}, \eref{eq:ftrlreg}, and \eref{ineq:drift_T_acc_before_eta}, we obtain
\begin{align}
    \frac{1}{N}\spadesuit \leq \frac{D^2}{\eta}+\frac{\eta BL^2}{2} \sum_{s=1}^{T/B}\left(\frac{1}{N}\sum_{v \in V} |m_{(s-1)B+1}(v)| + B\right).
    \label{ineq: spade_before_learning_rate_1}
\end{align}

Now we start analyzing Term $\clubsuit$. For notational convenience, we use $z_s(u)$ to denote $z_s^B(u)$ for all $u\in V$. From \lref{lemma: bar_zt_acc}, we know that $\forall w \in V$ and $\forall s \in [1,T/B]$,
\begin{align}
    \|z_s(w) - \bar{z}_{s}\|_2 & \leq \frac{2}{N\sqrt{N}}\sum_{l=1}^{s-1} b^{(s-l-1) B}\left(\sqrt{\sum_{v\in V}\left\|y_l(v)\right\|_2^2}\right). \label{ineq:gossi_z_t_acc}
\end{align}

Note that $x_1(u) = \bar{x}_1 = \mathbf{0}$. Combining \lref{lem:ftrl_sc} with \eref{ineq:gossi_z_t_acc}, we derive the following bound on the cumulative deviation between $x_s(w)$ and $\Bar{x}_s$ for any $w\in V$:
\begin{align}
    \sum_{s=1}^{T/B} \|x_{s}(w) - \bar{x}_{s} \|_2 &= \sum_{s=2}^{T/B} \|x_{s}(w) - \bar{x}_{s} \|_2
    \\&\leq \sum_{s=1}^{T/B-1} \eta \|z_{s}(w) - \bar{z}_{s}\|_2 \tag{according to \lref{lem:ftrl_sc}} \nonumber
    \\& =\frac{2\eta}{N\sqrt{N}}\sum_{s=1}^{T/B-1} \sum_{l=1}^{s-1} b^{(s-l-1) B}\left(\sqrt{\sum_{v\in V}\left\|y_l(v)\right\|_2^2}\right) \tag{\eref{ineq:gossi_z_t_acc}}
    \\& =\frac{2\eta}{N\sqrt{N}} \sum_{l=1}^{T/B - 1} \left( \sqrt{ \sum_{v \in V} \left\| y_l(v) \right\|_2^2 } \cdot\sum_{s=l+1}^{T/B} b^{(s - l - 1)B} \right)\tag{swapping the order of summation}
    \\& \leq \frac{2\eta}{N\sqrt{N}} \frac{1}{1-b^B} \sum_{l=1}^{T/B - 1} \left( \sqrt{ \sum_{v \in V}  \left\| y_l(v) \right\|_2^2 }  \right)
    \\& \leq \frac{2\eta}{N\sqrt{N}}\frac{1}{1-\frac{1}{\sqrt{14 N}}}  \sum_{l=1}^{T/B - 1} \left(\sqrt{ \sum_{v \in V} \left\| y_l(v) \right\|_2^2 } \right) \tag{since $b^B\leq \frac{1}{\sqrt{14N}}$ shown in \eref{eqn:bB_bound}}
    \\& \leq \frac{3\eta}{N\sqrt{N}}\sum_{l=1}^{T/B - 1} \left( \sqrt{ \sum_{v \in V} \left\| y_l(v) \right\|_2^2 } \right),
\label{ineq:gossip_acc}
\end{align} 
where the last inequality follows from $N\ge 1$.
Furthermore, according to \lref{lem:ineqNTL}, we have 
\begin{align}
    \sum_{l=1}^{T/B - 1} \left( \sqrt{ \sum_{v \in V} \left\| y_l(v) \right\|_2^2 } \right) \leq NTL. \label{ineq:network_error_acc}
\end{align}
Combining \eref{ineq:gossip_acc} and \eref{ineq:network_error_acc},
\[
\sum_{s=2}^{T/B} \|x_{s}(w) - \bar{x}_{s} \|_2 \le \frac{3 \eta TL}{\sqrt{N}}
\]
for all $w\in V$. 
Hence we obtain 
\begin{equation}
\label{eq:spade_1} 
\clubsuit\le 18 B  \eta N T L^2
\end{equation}
according to the definition of $\clubsuit$. 

Finally, combining \eref{eq:spade_1} with \eref{ineq: spade_before_learning_rate_1}, \eref{eq:spade_1} and \lref{lem:reg_decomp}, we can bound the overall regret as follows:

\begin{align}\Reg_T(u) &\le  \frac{D^2N}{\eta}+\frac{\eta BL^2N}{2} \sum_{s=1}^{T/B}  \left(\frac{1}{N}\sum_{v \in V} |m_{(s-1)B+1}(v)| + B\right) +  18 B \eta N \sqrt NT L^2 \nonumber\\
&\le  \frac{D^2N}{\eta}+ \frac{L^2 N}{2} \eta \dtot +\frac{ L^2 N B}{2}   \eta T  +  18 B \eta N T L^2 \tag{\lref{lemma: block_total_delay}}\\
&\le  \frac{D^2N}{\eta}+ \frac{L^2 N}{2} \eta \dtot  +  19 B \eta N T L^2.
\end{align}
Picking $ \eta$ to be $\frac{D}{L\sqrt{\dtot + B T}}$ leads to

\begin{align}\Reg_T(u) = \mathcal{O}\left(DLN  \sqrt{\dtot + \frac{ \ln(N)}{\sqrt{1 - \sigma_2(W)}}  T}\right) = \widetilde{\mathcal{O}}\left(DLN \left( \sqrt{\dtot} + \frac{ \sqrt{T}}{(1 - \sigma_2(W))^{1/4}}  \right)\right).
\end{align}

% \textcolor{blue}{
% When $d_t(u)=d(u)$ for all $t\in[T]$ for all $u\in V$, we combine \eref{eq:spade_2} with \eref{ineq: spade_before_learning_rate_1}, \eref{eq:spade_2} and \lref{lem:reg_decomp} we obtain
% \begin{align}\Reg_T(u) &\le \frac{D^2N}{\eta}+ \frac{L^2 N}{2} \eta \dtot + 19BL^2 N\eta T. \nonumber 
% \end{align} 
% Picking $\eta$ to be $\frac{D}{L\sqrt{\dtot + B T}}$ gives us
% \begin{align}\Reg_T(u) = \mathcal{O}\left(DLN  \sqrt{\dtot + B T}\right) = \mathcal{O}\left(DLN  \sqrt{\dtot + \frac{ \ln(N)}{\sqrt{1 - \sigma_2(W)}}T}\right).
% \end{align}}

\end{proof}

We now turn to proving the auxiliary lemmas invoked in the proof of the main theorem. The following lemma introduces the decomposition of the regret for \algname.
\begin{lemma}
    \label{lem:reg_decomp}
    For any sequences $\{\bar{x}_s(v)\}_{s\in[T/B], v\in V}$, $\bar{x}_s(v)\in \calX$, the regret of \alref{alg:dist_da_cvx_acc} can be bounded as 
    \begin{align*}
        \Reg_T(u) &\le{\sum_{s=1}^{T/B} \sum_{t\in \mathcal{T}_s}  \sum_{v\in V} \inner{g_t(v), \bar{x}_s(u) - x^*}} \nonumber
    \\& \quad+ {2BL\sum_{s=1}^{T/B}   \sum_{v\in V} \left(\|\bar{x}_s(u) -\bar{x}_s(v)\|_2+\|x_s(v) -\bar{x}_s(v)\|_2\right)+ NBL \sum_{s=1}^{T/B}\|x_s(u) -\bar{x}_s(u)\|_2},
    \end{align*}
        where $\mathcal{T}_s\triangleq \{(s-1)B+1, \ldots sB\}$ and $x^*=\argmin_{x\in\calX}\sum_{t=1}^T\sum_{v\in V}f_t(v,x)$.
\end{lemma}
\begin{proof}
    By definition of $\Reg_T(u)$, we know that
    \begin{align} 
        \Reg_T(u)  &= \sum_{t=1}^T  \sum_{v\in V} \left ( f_{t}(v, x_t(u)) -  f_{t}(v,x^*) \right)  \nonumber \\
    &= \sum_{t=1}^T  \sum_{v\in V} \left ( f_{t}(v, x_t(v)) -  f_{t}(v,x^*) \right) + \sum_{t=1}^T  \sum_{v\in V} \left ( f_{t}(v, x_t(u)) -  f_{t}(v,x_t(v)) \right) \nonumber 
    \\&\leq \sum_{t=1}^T  \sum_{v\in V} \left ( \inner{g_t(v), x_t(v) - x^*} \right) + L\sum_{t=1}^T  \sum_{v\in V} \|x_t(u) - x_t(v)\|_2 \nonumber \tag{\asref{asm:Lipschitz} and the convexity of $f_t$}
    \\&= \sum_{t=1}^T  \sum_{v\in V} \left ( \inner{g_t(v), x_t(v) +\bar{x}_t(v) - \bar{x}_t(v) +\bar{x}_t(u) - \bar{x}_t(u)   - x^*} \right) + L\sum_{t=1}^T  \sum_{v\in V} \|x_t(u) - x_t(v)\|_2 \nonumber 
    \\&= \sum_{t=1}^T  \sum_{v\in V} \left ( \inner{g_t(v), \bar{x}_t(u) - x^*} \right)  + L\sum_{t=1}^T\sum_{v\in V}  \left(  \|\bar{x}_t(v) -\bar{x}_t(u)\|_2 + \|x_t(v) -\bar{x}_t(v) \|_2\right) \nonumber
    \\&\quad + L\sum_{t=1}^T  \sum_{v\in V} \|x_t(u) - x_t(v)\|_2
    \nonumber \tag{\asref{asm:Lipschitz}}
    \\&\leq \sum_{t=1}^T  \sum_{v\in V} \left ( \inner{g_t(v), \bar{x}_t(u) - x^*} \right) + L\sum_{t=1}^T\sum_{v\in V}  \left(  \|\bar{x}_t(v) -\bar{x}_t(u)\|_2 + \|x_t(v) -\bar{x}_t(v) \|_2\right) \nonumber
    \\&\qquad + L\sum_{t=1}^T  \sum_{v\in V} \left(\|x_t(u) -\bar{x}_t(u)\|_2+\|\bar{x}_t(u) -\bar{x}_t(v)\|_2+\|x_t(v) -\bar{x}_t(v)\|_2\right) \nonumber \tag{triangle inequality}
    \\ &={\sum_{s=1}^{T/B} \sum_{t\in \mathcal{T}_s}  \sum_{v\in V} \inner{g_t(v), \bar{x}_s(u) - x^*}} \nonumber
    \\& \quad+ {2BL\sum_{s=1}^{T/B}   \sum_{v\in V} \left(\|\bar{x}_s(u) -\bar{x}_s(v)\|_2+\|x_s(v) -\bar{x}_s(v)\|_2\right)+ NBL \sum_{s=1}^{T/B}\|x_s(u) -\bar{x}_s(u)\|_2},
    \label{ineq:regret_decomposition_acc}
\end{align}
where the last equality is due to the fact that the algorithm uses the same decision over all time steps in the same block.
\end{proof}

\subsection{Properties induced by the gossiping mechanism}
The following two lemmas characterize the properties induced by the accelerated gossiping mechanism used in \alref{alg:dist_da_cvx_acc}. 
\begin{lemma}
\label{lem: z_k}
For any $n \geq 0$ , any $u \in V$ and any $s \in [T/B -1]$, we define
\begin{equation}
    y_s^{n}(u)  = y_s(u)
\end{equation} if $n=0$ or $n=-1$ and 
\begin{equation}
\label{eq: y_k}
    y_{s}^{n+1}(u) = (1+\theta)\sum_{v \in V} W(u,v)y_{s}^{n}(u)  - \theta y_{s}^{n-1}(u)
\end{equation}
otherwise.
    For any $k \geq 0$ , any $u \in V$ and any $s \in [T/B -1]$, \alref{alg:dist_da_cvx_acc} ensures

\begin{equation}
\label{eq:z_k}
    z_{s}^k(u)=\sum_{l=1}^{s-1} y_{l}^{(s-l-1) B+k}(v), \forall k=1, \ldots, B.
\end{equation}
\end{lemma}

\begin{proof}
The proof is taken from Lemma~2 in \cite{wan2024nearly}. We provide it here for completeness.
We introduce a new notation $z_{s}(u)$ to denote $z_{s}^B(u)$.
We use a double induction method. 
Recall that 
\begin{equation}
    y_s^{0}(u) = y_s^{-1}(u) = y_s(u).
\end{equation}
It is easy to verify by induction on $k$ that \eref{eq:z_k} holds for $s=2$ due to $z_2^0(u)=z_2^{-1}(u)=y_1(u)$ (initialization) and by using \eref{eq: y_k} for the induction . Then, we assume that \eref{eq:z_k} holds for some $s>2$, and prove it also holds for $s+1$. From the update of \alref{alg:dist_da_cvx_acc}, we have
\begin{align}
    z_{s+1}^0(u) &= z_{s}(u) + y_{s}(u) \nonumber
    \\&= z_{s}^B(u) + y_{s}^0(u)\nonumber
    \\&= \sum_{l=1}^s y_{l}^{(s-l)B}(u)\nonumber
\end{align}
and 
\begin{align}
    z_{s+1}^{-1}(u) &= z_{s}^{B-1}(u) + y_{s}(u) \nonumber
    \\&= z_{s}^{B-1}(u) + y_{s}^{-1}(u)\nonumber
    \\&= \sum_{l=1}^s y_{l}^{(s-l)B-1}(u).\nonumber
\end{align}

By induction, suppose that $z_{s+1}^k(u)$ and $z_{s+1}^{k-1}(u)$ satisfy  \eref{eq:z_k}.
By the update of \alref{alg:dist_da_cvx_acc}, we have

\begin{align}
z^k_{s+1}(u) & =(1+\theta) \sum_{v \in V} W(u,v) z^{k-1}_{s+1}(v)-\theta z^{k-2}_{s+1}(u) \nonumber
\\& =(1+\theta) \sum_{v \in V} W(u,v)\sum_{l=1}^s y_{l}^{(s-l)B+k-1}(u)-\theta \sum_{l=1}^s y_{l}^{(s-l)B+k-2}(u) \nonumber
\\& =\sum_{l=1}^s\left((1+\theta) \sum_{v \in V} W(u,v)\sum_{l=1}^s y_{l}^{(s-l)B+k-1}(u)- \theta \sum_{l=1}^s y_{l}^{(s-l)B+k-2}(u)\right) \nonumber\\& = \sum_{l=1}^s y_s^{(s-l) B+k}(u),\nonumber
\end{align}
which suffices to complete the induction for block $s+1$.
\end{proof}

The following lemma bounds the deviations between $z_s(u)$ and $\bar{z}_s$, for all agent $u \in V$.
\begin{lemma}
\label{lemma: bar_zt_acc}     
\alref{alg:dist_da_cvx_acc} guarantees that for any $u \in V$, for any $s \in [1,T/B]$,

\begin{align}
    \left \|z_s(u) - \bar{z}_s \right \|_2 \leq  \frac{2}{N\sqrt{N}}\sum_{l=1}^{s-1} b^{(s-l-1) B}\left(\sqrt{\sum_{v\in V}\|y_{l}(v)\|_2^2}\right), 
\end{align}
where $b = \left(1-(1-1/\sqrt{2})\sqrt{1-\sigma_2(W)}\right)$ and $B = \left\lceil\frac{\sqrt{2} \ln (N\sqrt{14 N})}{(\sqrt{2}-1) \sqrt{1-\sigma_2(W)}}\right\rceil$.
\end{lemma}

\begin{proof}
According to Equation 22 in \citet{wan2024nearly}, we know 
\begin{align}
    b^B \leq \frac{1}{N\sqrt{14N}} .\label{eqn:bB_bound}
\end{align} Then, with the same notation as in \lref{lem: z_k},
\begin{align}
    \left\|z_s(u) - \bar z_s\right\|_2 &= \left\| \sum_{l=1}^{s-1} y_{l}^{(s-l-1)B}(u)-  \frac{1}{N} \sum_{l=1}^{s-1} \sum_{v \in V} y_{l}(v)\right\|_2  \tag{from \lref{lem: z_k}}
    \\ & \leq \sum_{l=1}^{s-1}\left\|  y_{l}^{(s-l-1)B}(u)-  \frac{1}{N} \sum_{v \in V} y_{l}^0(v)\right\|_2 \tag{from the triangle inequality}
    \\ & \leq \sum_{l=1}^{s-1}\left\| Y^{(s-l-1)B}_l -\bar{Y}_l \right\|_F \nonumber
    \\ & \leq \sum_{l=1}^{s-1} \sqrt{14}b^{(s-l) B} \left\| Y^0_l -\bar{Y}_l \right\|_F \tag{from \propref{lemma: acc_gossip}}
    \\ & \leq \sum_{l=1}^{s-1} \sqrt{14}b^{(s-l) B}\left(\sqrt{\sum_{v\in V}\left\|  y_l(v) - \frac{1}{N}  \sum_{v\in V}y_l(v)\right\|_2^2}\right) \nonumber
    \\ & \leq \sum_{l=1}^{s-1} \sqrt{14}b^{(s-l) B}\left(\sqrt{\sum_{v\in V} \left\|y_l(v)\right\|^2}+\sqrt{N\left\| \frac{1}{N}  \sum_{v \in V} y_l(v)\right\|_2^2}\right) \tag{triangle inequality}
    \\&\leq \sum_{l=1}^{s-1} 2\sqrt{14}b^{(s-l) B}\left(\sqrt{\sum_{v\in V}\left\|y_l(v)\right\|_2^2}\right) \nonumber
    \\ &\leq \frac{2}{N\sqrt{N}}\sum_{l=1}^{s-1}b^{(s-l-1)B}\left(\sqrt{\sum_{v\in V}\left\|y_l(v)\right\|_2^2}\right),
\end{align}
where $Y^n_s$ are defined as
\[Y^n_s =[y_s^{(n)}(0), y_s^{(n)}(1) \ldots y_s^{(n)}(N)] \in \mathbb{R}^{N\times 1}\]
and in the third inequality, we apply \propref{lemma: acc_gossip} with $X^k = Y^k_s$ and the last inequality is because of   \eref{eqn:bB_bound}.
\end{proof}

Similarly, we can show the following two lemmas for the accelerated gossiping mechanism in \alref{alg:adaptive_learning_rate_acc} by replacing \( y_s^n(u) \) with \( q_s^n(u) \), \( y_s(u) \) with \( q_s(u) \), and \( z_s^k(u) \) with \( \zeta_s^k(u) \), noting that the gossip mechanisms for \( z \) in Algorithm 1 and for \( \zeta \) in Algorithm 2 are identical. The proof for \lref{lem:zeta_k} is omitted as they follow exactly the same steps as the one in \lref{lem: z_k}.

\begin{lemma}
\label{lem:zeta_k}
For any $n \geq 0$ , any $u \in V$ and any $s \in [T/B -1]$, we define
\begin{equation}
    q_s^{n}(u)  = q_s(u).
\end{equation} if $n=0$ or $n=-1$ and 
\begin{equation}
\label{eq: q_k}
    q_{s}^{n+1}(u) = (1+\theta)\sum_{v \in V} W(u,v)q_{s}^{n }(u)  - \theta q_{s}^{n-1}(u).
\end{equation}
otherwise.
    For any $k \geq 0$ , any $u \in V$ and any $s \in [T/B -1]$, \alref{alg:adaptive_learning_rate_acc} ensures

\begin{equation}
\label{eq:zeta_k}
    \zeta_{s}^k(u)=\sum_{l=1}^{s-1} q_{l}^{(s-l-1) B+k}(v), \forall k=1, \ldots, B.
\end{equation}
\end{lemma}

We introduce new notations $\hat{M}_{s}(u)$ to denote $\zeta_{s}^B(u)$ and $M_s\triangleq \frac{1}{N}\sum_{k=1}^{s}\sum_{v \in [N]} |m_{kB+1,v}|$ to be the cumulative missing observations averaged over all agents till block $s$. 
Then, we can bound the deviations between $\hat{M}_s(u)$ and $M_s(u)$ for all agents $u\in V$ as follows. The proof follows a similar analysis to \lref{lemma: bar_zt_acc}.

\begin{lemma}
\label{lemma: bar_mt_acc}     
\alref{alg:adaptive_learning_rate_acc} guarantees that for any $u \in V$, for any $s \in [1,T/B]$,
\begin{align}
    \left |\hat{M}_s(u) - M_s\right |  \leq  \frac{2}{N\sqrt{N}}\sum_{l=1}^{s-1} b^{(s-l-1) B}\left(\sqrt{\sum_{v\in V}|m_{lB+1}(v)|^2}\right), \label{ineq:gossi_zeta_t_acc}
\end{align}
and  consequently
\begin{align}
    \left |\hat{M}_s(u) - M_s \right | &\leq 3sB, \label{eq:first_ineq}
\end{align}
where $b = \left(1-(1-1/\sqrt{2})\sqrt{1-\sigma_2(W)}\right)$ and $B = \left\lceil\frac{\sqrt{2} \ln (N\sqrt{14 N})}{(\sqrt{2}-1) \sqrt{1-\sigma_2(W)}}\right\rceil$.
\end{lemma}

\begin{proof}
From Equation~22 from \cite{wan2024nearly}, we obtain 
\begin{equation}
\label{ineq: bb}
    b^B \leq \frac{1}{N\sqrt{14N}} .
\end{equation}

With the same notation as in \lref{lem:zeta_k},
\begin{align}
    \left|\zeta_s(u) - M_s\right| &= \left| \sum_{l=1}^{s-1} q_{l}^{(s-l-1)B}(u)-  \frac{1}{N} \sum_{l=1}^{s-1} \sum_{v \in V} q_{l}(v)\right|  \tag{from \lref{lem:zeta_k}}
    \\ & \leq \sum_{l=1}^{s-1}\left|  q_{l}^{(s-l-1)B}(u)-  \frac{1}{N} \sum_{v \in V} q_{l}^0(v)\right| \tag{ from the triangle inequality}
    \\ & \leq \sum_{l=1}^{s-1}\left\| Q^{(s-l-1)B}_l -\bar{Q}_l \right\|_F \nonumber
    \\ & \leq \sum_{l=1}^{s-1} \sqrt{14}b^{(s-l) B} \left\| Q^0_l -\bar{Q}_l \right\|_F \tag{ from \propref{lemma: acc_gossip}}
    \\ & \leq \sum_{l=1}^{s-1} \sqrt{14}b^{(s-l) B}\left(\sqrt{\left| \sum_{v\in V} q_l(v) - \frac{1}{N}  \sum_{v \in V} q_l(v)\right|^2}\right) \nonumber
    \\ & \leq \sum_{l=1}^{s-1} \sqrt{14}b^{(s-l) B}\left(\sqrt{\sum_{v\in V} \left|q_l(v)\right|^2}+\sqrt{N\left| \frac{1}{N}  \sum_{v \in V} q_l(v)\right|^2}\right) \tag{triangle inequality}
    \\&\leq \sum_{l=1}^{s-1} 2\sqrt{14}b^{(s-l) B}\left(\sqrt{\sum_{v\in V}\left|q_l(v)\right|^2}\right) \nonumber
    \\& \leq \frac{2}{N\sqrt{N}}\sum_{l=1}^{s-1} b^{(s-l-1) B}\left(\sqrt{\sum_{v\in V}\left|q_l(v)\right|^2}\right) \tag{from \eref{ineq: bb}}\\ 
    &\leq  \frac{2}{N\sqrt{N}}\sum_{l=1}^{s-1} b^{(s-l-1) B}\left(\sqrt{\sum_{v\in V}\left|m_{lB+1}(v)\right|}\right) , \label{eq:last_zeta}
\end{align}
where $Q^n_s$ are defined as
\[Q^n_s =[q_s^{(n)}(0), q_s^{(n)}(1) \ldots q_s^{(n)}(N)] \in \mathbb{R}^{N\times 1}\]
and \propref{lemma: acc_gossip} is used with $X^k = Q^k_s$.
Observing that $\zeta_s(u) = \hat{M}_s(u)$ directly yields \eref{ineq:gossi_zeta_t_acc}.

It also holds that 
\begin{align*}
    \left|\zeta_s(u) - M_s\right| &\leq \frac{2}{\sqrt{N}}\sum_{l=1}^{s-1} b^{(s-l-1) B}\left(\sqrt{\sum_{v\in V}\left|B s\right|^2}\right)
\leq {2 B s}\sum_{l=1}^{s-1} b^{(s-l-1) B} 
\\&\le  {2 B s}\frac{1}{1-b^B}\le \frac{2}{1-\frac{1}{\sqrt{14 N}}}Bs \le 3Bs
\end{align*}
thanks to \eref{ineq: bb}, which along with $\zeta_s(u) = \hat{M}_s(u)$ directly yields the first inequality of \lref{lemma: bar_mt_acc}.

\end{proof}

Similarly, we can establish the following lemma characterising the properties induced by the accelerated gossiping mechanism in \alref{alg:dist_da_cvx_acc_sc}, by replacing $y_s(u)$ by $y^+_s(u)$ and by observing that the accelerated gossip mechanisms for $z$ in \alref{alg:dist_da_cvx_acc} and in \alref{alg:dist_da_cvx_acc_sc} are identical. The proof for \lref{lemma: bar_zt_acc_sc} is omitted for conciseness since it directly follows the proof of \lref{lemma: bar_zt_acc_sc}.

\begin{lemma}
\label{lemma: bar_zt_acc_sc}     
\alref{alg:dist_da_cvx_acc_sc} guarantees that for any $u \in V$, for any $s \in [1,T/B]$,

\begin{align}
    \left \|z_s(u) - \bar{z}_s \right \|_2 \leq  \frac{2}{N\sqrt{N}}\sum_{l=1}^{s-1} b^{(s-l-1) B}\left(\sqrt{\sum_{v\in V}\|y_{l}^+(v)\|_2^2}\right),
\end{align}
where $b = \left(1-(1-1/\sqrt{2})\sqrt{1-\sigma_2(W)}\right)$ and $B = \left\lceil\frac{\sqrt{2} \ln (N\sqrt{14 N})}{(\sqrt{2}-1) \sqrt{1-\sigma_2(W)}}\right\rceil$.
\end{lemma}

The following lemma, used in the proof of \alref{alg:dist_da_cvx_acc}, provides a uniform upper bound on the square root of the cumulative squared norms of received gradient sums across all agents and blocks.
\begin{lemma} \label{lem:ineqNTL}
It holds that 
\begin{equation}
    \sum_{l=1}^{T/B - 1} \left( \sqrt{ \sum_{v \in V} \left\| y_l(v) \right\|_2^2 } \right) \le NTL. 
\end{equation}
\end{lemma}

\begin{proof} We have
    \begin{align}
    \sum_{s=1}^{T/B - 1} \left( \sqrt{ \sum_{v \in V} \left\| y_l(v) \right\|_2^2 } \right) &= 
    \sum_{s=1}^{T/B - 1} \left( \sqrt{ \sum_{v \in V} \left\| \sum_{\tau \in  o_{sB+1}(v) \backslash o_{(s-1)B+1}(v)} g_{\tau}(v) \right\|_2^2 } \right)\nonumber
    \\  &\le L\sum_{s=1}^{T/B} \left( \sqrt{ \sum_{v \in V}  \left( |o_{sB+1}(v)| - |o_{(s-1)B+1(v)}|\right)^2} \right) \tag{\asref{asm:Lipschitz}}
    \\ & \le L\sum_{s=1}^{T/B} \left(  \sum_{v \in V}  \left( |o_{sB+1}(v)| - |o_{(s-1)B+1}(v)|\right) \right) \tag{$\|\cdot\|_2 \leq \|\cdot\|_1$ }
    \\ & \leq NTL, \nonumber
\end{align}
where the last inequality holds because 
\begin{equation*}
    \sum_{s=1}^{T/B}\sum_{v \in V}  \sum_{\tau \in  o_{sB+1}(v) \backslash o_{(s-1)B+1}(v)} 1= \sum_{v \in V} \sum_{s=1}^{T/B}  \sum_{\tau \in  o_{sB+1}(v) \backslash o_{(s-1)B+1}(v)} 1 =\sum_{v\in V}\sum_{\tau \in  o_{TB+1}(v)}1 \leq NT.
\end{equation*}
\end{proof}

\subsection{Adaptive Algorithm with Unknown Total Delay}
\label{app: Adaptive Algorithm with Unknown Total Delay}
In this section, we show omitted details in Section~\ref{sec: Adaptive Algorithm with Unknown Total Delay}. For completeness, we first restate the theorem and then present its proof. 
\cvxadaptthm*
\begin{proof}

We define \( \bar{z}_{s-1} \) and \( \tilde{z}_{s-1} \) as in \eref{eq:barz} and \eref{eq:deftildaz}, respectively:
\begin{align}
    \bar{z}_{s-1} &\triangleq \frac{1}{N} \sum_{l=1}^{s-1} \sum_{v \in V} y_l(v), \label{eq:barz_2} \\
    \tilde{z}_{s-1} &\triangleq \frac{1}{N} \sum_{l=1}^{s-1} \sum_{\tau \in \mathcal{T}_l} \sum_{v \in V} g_\tau(v). \label{eq:deftildaz_2}
\end{align}

We also define the following:
\begin{align}
    \bar{x}_s(u) &\triangleq \argmin_{x \in \mathcal{X}} \left\langle \bar{z}_{s-1}, x \right\rangle + \frac{1}{\eta_{s}(u)} \|x\|_2^2, \label{eq:defbarx2} \\
    F_s(u, x) &\triangleq \left\langle \tilde{z}_{s-1}, x \right\rangle + \frac{1}{\eta_{s}(u)} \|x\|_2^2, \nonumber \\
    \tilde{x}_s(u) &\triangleq \argmin_{x \in \mathcal{X}} F_s(u, x). \nonumber
\end{align}
Recall that in \alref{alg:dist_da_cvx_acc} using $\eta_s(u)$ computed by \alref{alg:adaptive_learning_rate_acc}, we have 
\begin{equation*}
    x_1 = \mathbf{0} = \argmin_{x \in \mathcal{X}} \frac{1}{\eta_{1}(u)}\|x\|_2^2,
\end{equation*}
where $\eta_{1}(u) = \frac{D}{L\sqrt{BT+ 3B^2}}.$

Applying the regret decomposition proven in \lref{lem:reg_decomp} with the decision sequence $\{\bar{x}_{s}(u)\}_{s\in [T/B],u\in V}$ defined in \eref{eq:defbarx2}, we know that
\begin{align}
\Reg_T(u) &\le\underbrace{\sum_{s=1}^{T/B} \sum_{t\in \mathcal{T}_s}  \sum_{v\in V} \inner{g_t(v), \bar{x}_s(u) - x^*}}_{\spadesuit} \nonumber
    \\& \quad+ \underbrace{2BL\sum_{s=1}^{T/B}   \sum_{v\in V} \left(\|\bar{x}_s(u) -\bar{x}_s(v)\|_2+\|x_s(v) -\bar{x}_s(v)\|_2\right)+ NBL \sum_{s=1}^{T/B}\|x_s(u) -\bar{x}_s(u)\|_2}_{\clubsuit}. \label{eq: regret_decom_adap}
\end{align}

We start by analyzing Term $\spadesuit$. Similar to the non-adaptive learning rate analysis, we further decompose $\spadesuit$ as follows:
    \begin{align}
\frac{1}{N}\spadesuit &= \frac{1}{N}\sum_{s=1}^{T/B} \sum_{t\in \mathcal{T}_s}  \sum_{v\in V} \left\langle g_t(v), \bar{x}_{s}(u) - \Tilde{x}_{s}(u) + \Tilde{x}_{s}(u) -x^*\right\rangle  \nonumber
    \\&= \underbrace{\frac{1}{N}\sum_{s=1}^{T/B} \sum_{t\in \mathcal{T}_s}  \sum_{v\in V} \left\langle g_t(v), \Tilde{x}_{s}(u) -x^*\right\rangle}_{{\operatorname{full-info}_T}} + \underbrace{\frac{1}{N}\sum_{s=1}^{T/B} \sum_{t\in \mathcal{T}_s}  \sum_{v\in V} \left\langle g_t(v), \bar{x}_{s}(u)- \Tilde{x}_{s}(u)\right\rangle}_{{\operatorname{drift}_T}}. \label{ineq: spade_cheat_drift_2}\end{align}
    
For notational convenience, we define
\begin{equation*}
    \ell_s(x)\triangleq\left\langle\frac{1}{N}\sum_{\tau \in \mathcal{T}_s} \sum_{v \in V} g_t(v), x\right\rangle,
\end{equation*}
for all $s\in[T/B]$.
Regarding $\operatorname{full-info}_T$ , by using Lemma~7.1 in \cite{orabona2019modern}, we obtain  
\begin{align}
\operatorname{full-info}_T &= \frac{1}{N}\sum_{s=1}^{T/B} \sum_{t\in \mathcal{T}_s}  \sum_{v\in V} \left\langle g_t(v), \Tilde{x}_{s}(u) -x^*\right\rangle \nonumber
\\&= \sum_{s=1}^{T/B} \left\langle \frac{1}{N} \sum_{t\in \mathcal{T}_s}  \sum_{v\in V} g_t(v), \Tilde{x}_{s}(u) -x^*\right\rangle \nonumber
\\& \leq \frac{1}{\eta_{T/B+1}(u)}\|x^*\|_2^2 - \min_{x \in \calX}\frac{1}{\eta_{1}(u)}\|x\|_2^2   +F_{T/B+1}\left(u, \tilde x_{T/B+1}\right)-F_{T/B+1}(u,x^*) \nonumber
\\& \quad +\sum_{s=1}^{T/B}  \left[F_{s}(u, \Tilde{x}_{s}(u)) - F_{s+1}(u, \tilde{x}_{s+1}(u))+ \ell_s(\tilde{x}_s(u))\right] \nonumber
\\& \leq \frac{D^2}{\eta_{T/B+1}(u)}  +\sum_{s=1}^{T/B}  \left[F_{s}(u, \Tilde{x}_{s}(u)) - F_{s+1}(u, \tilde{x}_{s+1}(u))+ \ell_s(\tilde{x}_s(u))\right] 
\label{ineq:cheat_1}
\end{align}
where the last inequality holds because $F_{T/B+1}\left(u, \tilde x_{T/B+1}\right)-F_{T/B+1}(u,x^*) $ is a negative term by definition of $\tilde x_{T/B+1}(u)$, \asref{asm:bounded} and together with
non-negativity of $\min_{x \in \calX}\frac{1}{\eta_{1}(u)}\|x\|_2^2$. To analyze the second term in \eref{ineq:cheat_1}, we proceed as follows: 
\begin{align}
    &\sum_{s=1}^{T/B}  \left[F_{s}(u, \Tilde{x}_{s}(u)) - F_{s+1}(u, \tilde{x}_{s+1}(u))+ \ell_s(\tilde{x}_s(u))\right] \nonumber\\&\leq\sum_{s=1}^{T / B}\left[\left\langle\nabla \ell_s\left(\tilde{x}_s(u)\right), \tilde{x}_s(u)-\tilde{x}_{s+1}(u)\right\rangle-\frac{\lambda_{s}}{2}\left\|\tilde{x}_s(u)-\tilde{x}_{s+1}(u)\right\|^2_2 \right. \nonumber
    \\& \left.\quad+ \frac{1}{\eta_{s}(u)}\left\|\tilde{x}_{s}(u)\right\|^2-\frac{1}{\eta_{s+1}(u)}\left\|\tilde{x}_{s+1}(u)\right\|^2_2  \right] \nonumber
    \\&\leq \sum_{s=1}^{T/B}\left[\left\|\nabla \ell_s\left(\tilde{x}_s(u)\right)\right\|_2\left\|\tilde{x}_s(u)-\tilde{x}_{s+1}(u)\right\|_2-\frac{\lambda_{s}}{2}\left\|\tilde{x}_s(u)-\tilde{x}_{s+1}(u)\right\|^2 \right. \nonumber
    \\& \left. \quad\quad+\frac{1}{\eta_{s}(u)}\left\|\tilde{x}_{s+1}(u)\right\|^2-\frac{1}{\eta_{s+1}(u)}\left\|\tilde{x}_{s}(u)\right\|^2\right] \tag{Cauchy–Schwarz inequality}
    \\&\leq \sum_{s=1}^{T / B}\left[\frac{1}{\lambda_{s}}\left\|\nabla \ell_s\left(\tilde{x}_s(u)\right)\right\|_2^2-\frac{\lambda_{s}}{4}\left\|\tilde{x}_{s-1}(u)-\tilde{x}_{s}(u)\right\|^2\right. \nonumber
    \\& \left.\quad+\frac{1}{\eta_{s}(u)}\left\|\tilde{x}_{s+1}(u)\right\|^2-\frac{1}{\eta_{s+1}(u)}\left\|\tilde{x}_{s+1}(u)\right\|^2\right] \tag{$a b \leq \frac{a^2}{\lambda_{s}}+\frac{\lambda_{s}}{4} b^2$}
    \\&\leq \sum_{s=1}^{T / B}\left[\frac{1}{\lambda_{s}}\left\|\nabla \ell_s\left(\tilde{x}_s(u)\right)\right\|_2^2+\frac{1}{\eta_{s}(u)}\left\|\tilde{x}_{s+1}(u)\right\|^2-\frac{1}{\eta_{s+1}(u)}\left\|\tilde{x}_{s+1}(u)\right\|^2\right]  \nonumber
    \\&\leq \sum_{s=1}^{T / B}\left[\frac{B^2L^2}{2} \eta_{s}(u)+\frac{1}{\eta_{s}(u)}\left\|\tilde{x}_{s+1}(u)\right\|_2^2-\frac{1}{\eta_{s+1}(u)}\left\|\tilde{x}_{s+1}(u)\right\|_2^2\right], \label{ineq: stab_1}  
\end{align}
where the first inequality is because $\frac{1}{\eta_s(u)}\|x\|_2^2$ is $\lambda_s$-strongly convex convexity and $\lambda_s = 2 / \eta_s(u)$,  and the last inequality is because of \asref{asm:Lipschitz}. Combining \eref{ineq:cheat_1} and \eref{ineq: stab_1}, we obtain
\begin{align}
    \frac{1}{N}\operatorname{full-info}_T & \leq \frac{D^2}{\eta_{T/B+1}(u)}+ \frac{B^2L^2}{2}\sum_{s=1}^{T/B}\eta_{s}(u) + \sum_{s=1}^{T/B}\left ( \frac{1}{ \eta_{s}(u)}-\frac{1}{ \eta_{s+1}(u)}\right)\| \tilde{x}_{s+1}(u)\|^2_2. \label{ineq:cheating_T_acc_before_eta}
\end{align}

We now analyze the drift term \( \operatorname{drift}_T \). By definition, we have:
\begin{align}
    \operatorname{drift}_T 
    &= \frac{1}{N} \sum_{s=1}^{T/B} \sum_{t \in \mathcal{T}_s} \sum_{v \in V} \left\langle g_t(v), \bar{x}_s(u) - \tilde{x}_s(u) \right\rangle \nonumber \\
    &\leq BL \sum_{s=1}^{T/B} \left\| \bar{x}_s(u) - \tilde{x}_s(u) \right\|_2 \tag{Cauchy–Schwarz inequality and \asref{asm:Lipschitz}} \\
    &= BL \sum_{s=2}^{T/B} \left\| \bar{x}_s(u) - \tilde{x}_s(u) \right\|_2 \tag{\( \bar{x}_1(u) = \tilde{x}_1(u) = \mathbf{0} \)} \\
    &\leq \frac{BL}{2} \sum_{s=2}^{T/B} \eta_{s}(u) \left\| \bar{z}_{s-1} - \tilde{z}_{s-1} \right\|_2 \tag{\lref{lem:ftrl_sc}} \\
    &\leq \frac{BL}{2} \sum_{s=2}^{T/B} \eta_{s}(u) \left\| \frac{1}{N} \sum_{v \in V} \sum_{\tau \in o_{(s-1)B+1}(v)} g_\tau(v) - \frac{1}{N} \sum_{l=1}^{s-1} \sum_{\tau \in \mathcal{T}_l} \sum_{v \in V} g_\tau(v) \right\|_2
    \tag{Definition of \( \bar{z}_{s} \) and \( \tilde{z}_{s} \)} \\
    &= \frac{BL}{2} \sum_{s=2}^{T/B} \eta_{s}(u) \left\| -\frac{1}{N} \sum_{v \in V} \sum_{\tau \in m_{(s-1)B+1}(v)} g_\tau(v) \right\|_2
    \tag{\( \mathcal{T}_s = \{(s-1)B+1, \ldots, sB\} \), \( m_t(v) = [t-1] \setminus o_t(v) \)} \\
    &\leq \frac{BL^2}{2} \sum_{s=2}^{T/B} \eta_{s}(u) \left( \frac{1}{N} \sum_{v \in V} |m_{(s-1)B+1}(v)| \right), \label{ineq:drift_T_acc_before_eta_2}
\end{align}
where the last inequality is because of \asref{asm:Lipschitz}. Combining \eref{ineq: spade_cheat_drift_2}, \eref{ineq:cheating_T_acc_before_eta}, and \eref{ineq:drift_T_acc_before_eta_2}, we obtain:
\begin{align}
    \frac{1}{N} \spadesuit 
    &\leq \frac{D^2}{\eta_{T/B+1}(u)} + \frac{BL^2}{2} \sum_{s=2}^{T/B} \eta_{s}(u) \left( \frac{1}{N} \sum_{v \in V} |m_{(s-1)B+1}(v)| + B \right) \nonumber  
    \\& \quad+ \sum_{s=1}^{T/B} \left( \frac{1}{\eta_{s}(u)} - \frac{1}{\eta_{s+1}(u)} \right) \left\| \tilde{x}_{s+1}(u) \right\|_2^2.
    \label{ineq: spade_before_learning_rate_2}
\end{align}

Recall that $M_s$ is defined as $M_s\triangleq\frac{1}{N} \sum_{l=1}^{s-1} \sum_{v \in V}\left|m_{l B+1}(v)\right|,$ and $\hat{M}_s(u)\triangleq\zeta_s^B(u)$. 
Using the bound 
\(\left |\hat{M}_s(u) - M_s \right | \leq 3sB\) 
from \lref{lemma: bar_mt_acc}, for all $ u \in V$ we obtain
\begin{align}
    \eta_{s+1}(u) &= \frac{D}{L \sqrt{B  T+B \widehat{M}_s(u)+3 s B^2}} \le \frac{D}{L\sqrt{ BT + B M_s}} \nonumber
    \\&\le \frac{D}{L\sqrt{ BT + B M_s}} \le \frac{D}{L\sqrt{ BT}}\;, \label{eq:bound_eta_s}
\end{align}
We start with the first term in \eref{ineq: spade_before_learning_rate_2}. We have 
\begin{align}
    \frac{D^2}{\eta_{T/B+1}(u)} &= DL \sqrt{BT + B \widehat{M}_{T/B}(u)+ 3BT}  \tag{Definition of $\eta_s(u)$}
    \\& \leq DL \sqrt{BT + B M_{T/B}+ 6BT} \tag{Using the bound $\left |\hat{M}_s(u) - M_s \right | \leq 3sB$ from \lref{lemma: bar_mt_acc}}
    \\& \leq DL \sqrt{BT + \dtot+ 7BT} \tag{\lref{lemma: block_total_delay}}
    \\& \leq DL \sqrt{8BT + \dtot}.  
    \label{ineq:bias}
\end{align}

Focus on the second term in \eref{ineq: spade_before_learning_rate_2}, we thus have
\begin{align}
\label{ineq: spade_after_learning_rate}
     & \frac{BL^2}{2}\sum_{s=2}^{T/B}  \eta_{s}(u) \left(\frac{1}{N}\sum_{v \in V} |m_{(s-1)B+1}(v)| + B\right) \nonumber
     \\&\leq \frac{BL^2}{2}\sum_{s=2}^{T/B}  \eta_{s}(u) \left(\frac{1}{N}\sum_{v \in V} |m_{(s-1)B+1}(v)|\right) +\frac{1}{2} DL\sqrt{BT} \tag{\eref{eq:bound_eta_s}}
     \\&\leq \frac{DL}{2} \left(B  \sum_{s=2}^{T/B}  \frac{\frac{1}{N}\sum_{v \in V} |m_{(s-1)B+1}(v)|}{\sqrt{BT+B {M}_{s-1}}}\right) +\frac{1}{2} DL\sqrt{BT}  \tag{\eref{eq:bound_eta_s}}
     \\&= \frac{DL}{2} \left(B  \sum_{s=2}^{T/B}  \frac{\frac{1}{N}\sum_{v \in V} |m_{(s-1)B+1}(v)|}{\sqrt{B {M}_{s}}}\right) +\frac{1}{2} DL\sqrt{BT} \tag{$|M_s| \leq |M_{s-1}| +T$}
    \\&= \frac{DL}{2} \left(  \sum_{s=2}^{T/B}  \frac{\frac{B}{N}\sum_{v \in V} |m_{(s-1)B+1}(v)|}{\sqrt{\frac{B}{N}\sum_{l=1}^{s-1}\sum_{v \in V} |m_{lB+1}|}}\right) +\frac{1}{2} DL\sqrt{BT}  \tag{Definition of $|M_s|$}
     \\&\leq DL \sqrt{BM_{T/B}} +\frac{1}{2} DL\sqrt{BT} \tag{\lref{lemma: ada-grad}}
     \\&\leq DL \sqrt{BT + \dtot} +\frac{1}{2} DL\sqrt{BT} \tag{\lref{lemma: block_total_delay}}
     \\&\leq 2DL \sqrt{BT + \dtot}.
     \label{eq:fst_term_bef_lr}
\end{align}

Let us now analyze the third term of \eref{ineq: spade_before_learning_rate_2}. We have  
\begin{align}
    \sum_{s=1}^{T/B}&\left ( \frac{1}{ \eta_{s}(u)}-\frac{1}{ \eta_{s+1}(u)}\right)\| \tilde{x}_{s+1}(u)\|^2_2\nonumber
    \\&\leq D^2\sum_{s=1}^{T/B}\left | \frac{1}{\eta_{s}(u)}-\frac{1}{\eta_{s+1}(u)} \right|
    \tag{\asref{asm:bounded}}
    \\&= D^2\sum_{s=1}^{T/B}\left | \frac{ 1/\eta_{s+1}(u)^2-1/\eta_{s}(u)^2}{1/\eta_{s}(u)+1/\eta_{s+1}(u)} \right|\nonumber
    \\& \leq 
    DL\sum_{s=1}^{T/B}  \frac{ \left(B\left|\hat{M}_{s}(u)-\hat{M}_{s-1}(u)\right| + 3B^2\right)}{\sqrt{B T+B \hat{M}_s(u)+3s B^2}+\sqrt{B T+B\hat{M}_{s-1}(u)+3(s-1)B^2}} \tag{Plugging the definition of the learning rate}
    \\& \leq 
    DL\sum_{s=1}^{T/B}  \frac{\left(B \left|\hat{M}_{s}(u)-\hat{M}_{s-1}(u)\right| + 3B^2\right)}{\sqrt{BT+BM_s}},
    \end{align}
    where the last inequality is due to \eref{eq:first_ineq} in \lref{lemma: bar_mt_acc}. Now decomposing the numerator and using the triangle inequality,
    \begin{align}   
&\sum_{s=1}^{T/B}\left ( \frac{1}{ \eta_{s}(u)}-\frac{1}{ \eta_{s+1}(u)}\right)\| \tilde{x}_{s+1}(u)\|^2_2 \nonumber\\
&\leq 
    DL\sum_{s=1}^{T/B}  
    \frac{B\left(|\hat{M}_{s}(u) - M_{s}| + |M_{s} - M_{s-1}| + |\hat{M}_{s-1}(u) - M_{s-1}|\right) +3B^2}
    {\sqrt{BT+BM_s}}  \nonumber
\\& \leq 
    DL\sum_{s=1}^{T/B}  
    \frac{B\left(|\hat{M}_{s}(u) - M_{s}| + |M_{s} - M_{s-1}| + |\hat{M}_{s-1}(u) - M_{s-1}|\right)} 
    {\sqrt{BT+BM_s}} + 3DL\sqrt{BT} \nonumber  
\\& \leq DL \sum_{s=1}^{T/B}
    \frac{
    B|M_{s} - M_{s-1}|+ \frac{2B}{N\sqrt{N}}\sum_{l=1}^{s-1} b^{(s-l-1) B}\left(\sqrt{\sum_{v\in V}|m_{lB+1}(v)|^2}\right)}
    {\sqrt{BT+BM_s}} \nonumber\\
& \qquad + \frac{2B}{N\sqrt{N}}\sum_{l=1}^{s-2} b^{(s-l-2) B}\left(\sqrt{\sum_{v\in V}|m_{lB+1}(v)|^2}\right)
    {\Big/ \sqrt{BT+BM_s}} + 3DL\sqrt{BT} 
    \tag{from \eref{ineq:gossi_zeta_t_acc} in \lref{lemma: bar_mt_acc}}
\\& \leq  DL    \sum_{s=1}^{T/B}
    \frac{
    \frac{B}{N} \sum_{v \in V} |m_{(s-1)B+1}(v)|+ \frac{2B}{N\sqrt{N}}\sum_{l=1}^{s-1} b^{(s-l-1) B}\left(\sqrt{\sum_{v\in V}|m_{lB+1}(v)|^2}\right)}
    {\sqrt{BT+BM_s}} \nonumber \\
& \qquad + \frac{2B}{N\sqrt{N}}\sum_{l=1}^{s-2} b^{(s-l-2) B}\left(\sqrt{\sum_{v\in V}|m_{lB+1}(v)|^2}\right)
    {\Big/ \sqrt{BT+BM_s}} + 3DL\sqrt{BT}  \tag{Definition of $M_s$}
\\& = DL   \sum_{s=1}^{T/B}
    \frac{ \frac{2B}{N\sqrt{N}}\sum_{l=1}^{s-1} b^{(s-l-1) B}\left(\sqrt{\sum_{v\in V}|m_{lB+1}(v)|^2}\right)}
    {\sqrt{BT+BM_s}} \nonumber \\
& \qquad + DL   \sum_{s=1}^{T/B}
    \frac{ \frac{2B}{N\sqrt{N}}\sum_{l=1}^{s-2} b^{(s-l-2) B}\left(\sqrt{\sum_{v\in V}|m_{lB+1}(v)|^2}\right)}
    {\sqrt{BT+BM_s}} \tag{Rearranging of terms}\\
& \qquad + \sum_{s=1}^{T/B}
    \frac{\frac{B}{N} \sum_{v \in V} |m_{(s-1)B+1}(v)|}
    {\sqrt{BT+BM_s}}  
+ 3DL\sqrt{BT} 
\label{eq:last_of_ineq}
\end{align}

Let us consider the first two summation terms in \eref{eq:last_of_ineq}. We have
    \begin{align}
        &\frac{\sum_{s=1}^{T/B} \frac{2B}{N\sqrt{N}}\sum_{l=1}^{s-1} b^{(s-l-1) B}\left(\sqrt{\sum_{v\in V}|m_{lB+1}(v)|^2}\right)}{\sqrt{BT+BM_s}}  \nonumber
    \\ & = \frac{2B}{N\sqrt{N}}\sum_{l=1}^{T/B-1} \sum_{s=1+l}^{T/B} \frac{b^{(s-l-1) B}\left(\sum_{v\in V}|m_{lB+1}(v)|\right)}{\sqrt{BT+BM_{s}}}\tag{swapping the order of summation}
    \\ & \le  \frac{1}{1-b^B}
         \frac{2B}{N\sqrt{N}}\sum_{l=1}^{T/B-1}   \left(\sum_{v\in V}\frac{|m_{lB+1}(v)|}{\sqrt{BT+BM_{l+1}}}\right) \tag{$M_l$ is non-decreasing}
        \\&\le
          \frac{8B}{N\sqrt{N}}\sum_{l=1}^{T/B-1}   \left(\sum_{v\in V}\frac{|m_{lB+1}(v)|}{\sqrt{BT+BM_{l+1}}}\right) \tag{ from $\frac{1}{1-b^B}\le \frac{1}{1-1/(14\sqrt{N})} \le 4$}
    \end{align}
Similarly, we have 
\begin{align*}
    &\frac{\sum_{s=1}^{T/B} \frac{2B}{\sqrt{N}}\sum_{l=1}^{s-2} b^{(s-l-2) B}\left(\sqrt{\sum_{v\in V}|m_{lB+1}(v)|^2}\right)}{\sqrt{BT+BM_s}} \\
    &\leq \frac{8B}{N\sqrt{N}}\sum_{l=1}^{T/B-2}   \left(\sum_{v\in V}\frac{|m_{lB+1}(v)|}{\sqrt{BT+BM_{l+2}}}\right) \leq \frac{8B}{N}\sum_{l=1}^{T/B-1}   \left(\sum_{v\in V}\frac{|m_{lB+1}(v)|}{\sqrt{BT+BM_{l+1}}}\right).
\end{align*}
Plugging above two inequalities back into \eref{eq:last_of_ineq}, 
    \begin{align}   
    &\sum_{s=1}^{T/B}\left ( \frac{1}{ \eta_{s}(u)}-\frac{1}{ \eta_{s+1}(u)}\right)\| \tilde{x}_{s+1}(u)\|^2_2 \nonumber
    \\
    & \leq  DL  \left( 16  \sum_{s=1}^{T/B -1} \frac{\frac{B}{N}    \sum_{v\in V}|m_{sB+1}(v)|}{\sqrt{BT + M_{s+1}}}
      + \sum_{s=1}^{T/B}\frac{\frac{B}{N} \sum_{v \in V} |m_{(s-1)B+1}(v)|}{\sqrt{BT+B M_{s}}} + 3\sqrt{BT} \right) \nonumber\\
      &  \leq  DL  \left( 17    \sum_{s=1}^{T/B}\frac{\frac{B}{N}   \sum_{v\in V}|m_{sB+1}(v)|}{\sqrt{B/N\sum_{l=1}^{s}\sum_{v \in V} |m_{lb+1,v}|}}
       + 3\sqrt{BT} \right) \tag{Definition of $M_{s}$}
       \\
       & \leq  DL  \left( 17    \sqrt{\frac{B}{N}\sum_{l=1}^{T/B}   \left(\sum_{v\in V}|m_{lB+1}(v)|\right)}
       + 3\sqrt{BT} \right) \tag{\lref{lemma: ada-grad} }\\
       & \leq  DL  \left( 17   \sqrt{\dtot + BT}
       + 3\sqrt{BT} \right). \tag{\lref{lemma: block_total_delay}}
    \end{align}  

The above inequality, together with \eref{ineq:bias}, \eref{eq:fst_term_bef_lr} and \eref{ineq: spade_before_learning_rate_2}, yields
\begin{align}
     \spadesuit &\le  N \left(DL \sqrt{8BT + \dtot} + 2DL \sqrt{BT + \dtot} + DL  \left( 17  \sqrt{\dtot + BT} + 3\sqrt{BT} \right)\right) \nonumber
     \\&\le  25 NDL \sqrt{BT} +  20 NDL\sqrt{\dtot}\nonumber \\
     &=\order(NDL\sqrt{BT+\dtot})\label{eq:spades2}
 \end{align}
 
Let us now turn to the analysis of Term $\clubsuit$. 
From \lref{lemma: bar_zt_acc}, we have for all $w \in [N]$,
\begin{align}
    \|z_s(w) - \bar{z}_{s}\|_2 & \leq \frac{2}{N\sqrt{N}}\sum_{l=1}^{s-1} b^{(s-l-1) B}\left(\sqrt{\sum_{v\in V}\left\|y_l(v)\right\|_2^2}\right). \label{ineq:gossi_z_t_acc_2}
\end{align}
Define $\bar{\eta}$ such that  $\eta_s(v)\leq \bar{\eta} \triangleq \frac{D}{L\sqrt{BT}}$ for all $v \in V$.
Note that $x_1(v) = \bar{x}_1 = \mathbf{0}$. Using \lref{lem:ftrl_sc}, we know that for any $w\in V$,
\begin{align}
\sum_{s=1}^{T / B}\left\|x_s(w)-\bar{x}_s(w)\right\|_2&=
    \sum_{s=2}^{T/B } \|x_{s}(w) - \bar{x}_{s}(w) \|_2 \nonumber
    \\&\leq \sum_{s=1}^{T/B-1} \eta_{s+1}(w) \|z_{s}(w) - \bar{z}_{s}\|_2  \tag{\lref{lem:ftrl_sc}}
    \\& =\frac{2}{N\sqrt{N}}\sum_{s=1}^{T/B-1}\eta_{s+1}(w) \sum_{l=1}^{s-1} b^{(s-l-1) B}\left(\sqrt{\sum_{v\in V}\left\|y_l(v)\right\|_2^2}\right) 
    \\& \le\frac{2\bar{\eta}}{N\sqrt{N}}\sum_{s=1}^{T/B-1}\sum_{l=1}^{s-1} b^{(s-l-1) B}\left(\sqrt{\sum_{v\in V}\left\|y_l(v)\right\|_2^2}\right) \nonumber
    \\& =\frac{2\bar{\eta}}{N\sqrt{N}} \sum_{l=1}^{T/B - 1} \left( \sqrt{ \sum_{v \in V} \left\| y_l(v) \right\|_2^2 } \sum_{s=l+1}^{T/B-1}  \cdot b^{(s - l - 1)B} \right)\tag{Swapping the order of summation}
    \\& \leq \frac{2\bar{\eta}}{N\sqrt{N}}\frac{1}{1-\frac{1}{\sqrt{14 N}}}  \sum_{l=1}^{T/B - 1} \left( \sqrt{ \sum_{v \in V} \left\| y_l(v) \right\|_2^2 } \right) \tag{\eref{ineq: bb}}
    \\& \leq \frac{3\bar{\eta}}{N\sqrt{N}}  \sum_{l=1}^{T/B - 1} \left( \sqrt{ \sum_{v \in V} \left\| y_l(v) \right\|_2^2 } \right).
    \label{ineq:gossip_acc_2}
\end{align} 
Moreover, according to \lref{lem:ineqNTL}, we have 
\begin{align*}
    \sum_{s=1}^{T/B - 1} \left( \sqrt{ \sum_{v \in V} \left\| y_l(v) \right\|_2^2 } \right) \leq NTL.
\end{align*}
Therefore, combining the above two inequalities, we know that
\[\sum_{s=2}^{T/B} \|x_{s}(w) - \bar{x}_{s}(w) \|_2 \le 3 \bar{\eta} \frac{1}{\sqrt{N}}TL,\]
leading to a bound on term $\clubsuit$:
\begin{equation}
    \clubsuit \le 18 B N\bar{\eta} TL^2 + 2BL \sum_{v \in V}\sum_{s=1}^{T/B}\|\bar{x}_s(u) -\bar{x}_s(v)||_2. \label{eq_last_bound_club}
\end{equation}

Now we bound the second term in \eref{eq_last_bound_club}.
By using \lref{lem:ftrl_sc}, we obtain
\begin{align*}
    &\frac{1}{\eta_{s}(u)}\|\bar{x}_{s}(u) -\bar{x}_{s}(v) \|_2^2 + \frac{1}{ \eta_{s}(v)}\| \bar{x}_{s}(u) -\bar{x}_{s}(v) \|_2^2 \\&\leq \frac{1}{ \eta_{s}(u)} \|\bar{x}_{s}(v)\|_2^2 - \frac{1}{ \eta_{s}(v)} \|\bar{x}_{s}(v)\|_2^2 + \frac{1}{ \eta_{s}(v)} \|\bar{x}_{s}(u)\|_2^2 - \frac{1}{ \eta_{s}(u)} \|\bar{x}_{s}(u)\|_2^2
\end{align*}
Rearranging the last inequality, we have 
\begin{align*}
    \frac{\eta_{s}(u)+ \eta_{s}(v)}{\eta_{s}(v) \eta_{s}(u)} \|\bar{x}_{s}(u) -\bar{x}_{s}(v) \|_2^2 &\leq \frac{\eta_{s}(u)- \eta_{s}(v)}{\eta_{s}(v) \eta_{s}(u)} (\|\bar{x}_{s}(u)\|_2^2 - \|\bar{x}_{s}(v)\|_2^2)\\
    &\leq \left|\frac{\eta_{s}(u)- \eta_{s}(v)}{\eta_{s}(v) \eta_{s}(u)} \right|\|\bar{x}_{s}(u) -\bar{x}_{s}(v) \|_2 \|\bar{x}_{s}(u) +\bar{x}_{s}(v) \|_2.
\end{align*}
Since $\|\bar{x}_{s}(u) -\bar{x}_{s}(v) \|_2 \geq 0$ and \asref{asm:bounded},  we have 
\begin{equation}
\label{ineq：bar_xt}
    \|\bar{x}_{s}(u) -\bar{x}_{s}(v) \|_2 \le 2D \left|\frac{\eta_{s}(u)- \eta_{s}(v)}{\eta_{s}(u)+ \eta_{s}(v)}\right|. 
\end{equation}
By pure algebraic computations, we have 
\begin{align}
    \left|\frac{\eta_{s}(u)- \eta_{s}(v)}{\eta_{s}(u)+ \eta_{s}(v)}\right| 
    &= \left|\frac{\frac{\eta_{s}(v)- \eta_{s}(u)}{\eta_{s}(v)\eta_{s}(u)}}{\frac{\eta_{s}(u)+ \eta_{s}(v)}{\eta_{s}(v)\eta_{s}(u)}}\right| 
    = \left|\frac{\eta_{s}(u)^{-1}- \eta_{s}(v)^{-1}}{\eta_{s}(u)^{-1}+ \eta_{s}(v)^{-1}}\right| \nonumber 
    \\&= \frac{\left|(\eta_{s}(u)^{-1})^2- (\eta_{s}(v)^{-1})^2\right|}{(\eta_{s}(u)^{-1}+ \eta_{s}(v)^{-1})^2} \nonumber \leq \frac{\left|(\eta_{s}(u)^{-1})^2- (\eta_{s}(v)^{-1})^2\right|}{(\eta_{s}(u)^{-2}+ \eta_{s}(v)^{-2})} .\nonumber
\end{align} 
Combining above inequality with \eref{ineq：bar_xt}, we have
\begin{align}
    \sum_{s=1}^{T/B} \|\bar{x}_{s}(u) -\bar{x}_{s}(v)\|_2  &\leq 2D \sum_{s=1}^{T/B} \frac{\left|\eta_{s}(u)^{-2}- \eta_{s}(v)^{-2}\right|}{(\eta_{s}(u)^{-2}+ \eta_{s}(v)^{-2})} \nonumber
    \\& \leq 2D \sum_{s=2}^{T/B} \frac{\left|\eta_{s}(u)^{-2}- \eta_{s}(v)^{-2}\right|}{(\eta_{s}(u)^{-2}+ \eta_{s}(v)^{-2})} \tag{$\eta_1(u)= \eta_1(v)=\frac{D}{L \sqrt{ B T+3 B^2}}$}
    \\ &\leq 2D \sum_{s=1}^{T/B} \frac{\left|\paren{\frac{D}{L \sqrt{T B+B\hat{M}_{s}(u)+3sB^2}}}^{-2}- \paren{\frac{D}{L \sqrt{T B+B\hat{M}_{s}(v)+3sB^2}}}^{-2}\right|}{\paren{\frac{D}{L \sqrt{T B+B\hat{M}_{s}(u)+3sB^2}}}^{-2}+ \paren{\frac{D}{L \sqrt{T B+B\hat{M}_{s}(v)+3sB^2}}}^{-2}} \nonumber
    \\ & = 2DB \sum_{s=1}^{T/B} \frac{\left|\hat{M}_s(u) - \hat{M}_s(v)\right|}{2  B T+6sB^2 +B\hat{M}_s(u) +B\hat{M}_s(v)} \tag{re-arranging}
       \\ & \leq DB   \sum_{s=1}^{T/B} \frac{\left|\hat{M}_s(u) - \hat{M}_s(v)\right|}{ B T+B {M}_s} \tag{Using \eref{eq:first_ineq} in \lref{lemma: bar_mt_acc}}
    \\ & \leq 
    DB   \sum_{s=1}^{T/B} \frac{\frac{2}{N\sqrt{N}}\sum_{l=1}^{s-1} b^{(s-l-1) B}\left(\sqrt{\sum_{v\in V}|m_{lB+1}(v)|^2}\right)}{BT+B{M}_{s}}. 
    \tag{Using \eref{ineq:gossi_zeta_t_acc} in \lref{lemma: bar_mt_acc}}
    \end{align}
Now notice that we have
    \begin{align}
        &\frac{\sum_{s=1}^{T/B} \frac{1}{N\sqrt{N}}\sum_{l=1}^{s-1} b^{(s-l-1) B}\left(\sqrt{\sum_{v\in V}|m_{lB+1}(v)|^2}\right)}{B M_{s} + BT}  \nonumber\\
        & \le \sum_{s=1}^{T/B}
     \frac{1}{N\sqrt{N}}\sum_{l=1}^{s-1} \frac{b^{(s-l-1) B}\left(\sum_{v\in V}|m_{lB+1}(v)|\right)}{B M_{s} + BT}
        \\ & \le 
         \frac{1}{N\sqrt{N}}\sum_{l=1}^{T/B-1} \sum_{s=1+l}^{T/B} \frac{b^{(s-l-1) B}\left(\sum_{v\in V}|m_{lB+1}(v)|\right)}{B M_{s} + BT}\tag{swapping sums}
        \\ & \le  \frac{1}{1-b^B}
         \frac{1}{N\sqrt{N}}\sum_{l=1}^{T/B-1}   \left(\sum_{v\in V}\frac{|m_{lB+1}(v)|}{B M_{l+1} + BT}\right) \tag{$M_l$ is non-decreasing}
         \\ & \le  \frac{1}{1-\frac{1}{\sqrt{14N}}}
         \frac{1}{N\sqrt{N}}\sum_{l=1}^{T/B-1}   \left(\sum_{v\in V}\frac{|m_{lB+1}(v)|}{B M_{l+1} + BT}\right) \tag{since $b^B\leq \frac{1}{\sqrt{14N}}$ shown in \eref{eqn:bB_bound}} 
         \\&\le
          \frac{4}{B\sqrt {N}}  \sum_{l=1}^{T/B-1}   \left(\frac{\frac{B}{N}\sum_{v\in V}|m_{lB+1}(v)|}{{B M_{l+1} + BT}}\right). \tag{ from $\frac{1}{1-b^B}\le \frac{1}{1-1/(14\sqrt{N})} \le 4$}
    \end{align}

    To analyze the term $\sum_{l=1}^{T/B-1}   \left(\frac{\frac{B}{N}\sum_{v\in V}|m_{lB+1}(v)|}{{B M_{l+1} + \sqrt{N}BT}}\right)$, we use \lref{lemma: ada-grad} and \lref{lemma: block_total_delay}: 
    \begin{align*}
        \sum_{l=1}^{T/B-1}   \left(\frac{\frac{B}{N}\sum_{v\in V}|m_{lB+1}(v)|}{{B M_{l+1} + BT}}\right) &= \sum_{l=1}^{T/B-1}   \left(\frac{\frac{B}{N}\sum_{v\in V}|m_{lB+1}(v)|}{{\frac{B}{N}\sum_{\tau=1}^{l}\sum_{v\in V}|m_{\tau B+1}(v)| + BT}}\right) \tag{by definition of $M_{l+1}$}
        \\&\le \ln\left(BT + \sum_{l=1}^{T/B-1} \frac{B}{N}\sum_{v\in V} |m_{lB+1}(v)|\right) -\ln(BT) \tag{using \lref{lemma: ada-grad}}\\
        &\le \ln\left(BT + \sum_{l=1}^{T/B-1} \frac{B}{N}\sum_{v\in V} |m_{lB+1}(v)|\right) \\
        &\leq \ln\left(2BT + \dtot \right). \tag{using \lref{lemma: block_total_delay}}
    \end{align*}

    Combining the above two bounds, we can obtain the bound for $\sum_{s=1}^{T/B} \|\bar{x}_{s}(u) -\bar{x}_{s}(v)\|_2$:
    \begin{align}
    \sum_{s=1}^{T/B} \|\bar{x}_{s}(u) -\bar{x}_{s}(v)\|_2 &\le 8 DB   \frac{1}{\sqrt{N} B} \ln(2BT + \dtot) \nonumber
     \\&= 8D \ln(2BT + \dtot)   \nonumber\\
  & \le 8D\ln(2BT + T^2),  \nonumber
    \end{align}
    where the last inequality uses $\dtot \leq T^2$. Plugging the above inequality into \eref{eq_last_bound_club},
    we have
    \begin{align}
        \clubsuit &\le 18 BL^2 N\bar{\eta} T + 16BDL N \ln(2BT + T^2).
    \end{align}
Recall that $\bar{\eta} = \frac{D}{L\sqrt{BT}}$. The above upper bound further implies the following bound on $\clubsuit$.
\begin{align}
        \clubsuit 
        &\le 18 DL N  \sqrt{BT} + 16 BDL N \ln(2BT + T^2) \nonumber
        \\&= \order(DLN\sqrt{BT}+BDLN\ln(BT+T^2)).
    \end{align}
Finally, combining the above inequality, \eref{eq: regret_decom_adap}, and \eref{eq:spades2}, we obtain
\begin{align*}
        \Reg_T(u) &\le\clubsuit + \spadesuit \\ 
        &\leq \order(NDL\sqrt{BT+\dtot}) + \order(DLN\sqrt{BT}+BDLN\ln(BT+T^2)) \\ 
        &\leq \order(NDL\sqrt{BT+\dtot}+ BDLN\ln(BT+T^2)).
        \end{align*}
Plugging in the form of $B=\Theta\left(\frac{\ln N}{\sqrt{1-\sigma_2(W)}}\right)$, we obtain our final bound:
\begin{align*}
        \Reg_T(u) =   \widetilde{\mathcal{O}}\left( NDL \left(\frac{\sqrt{T}}{(1-\sigma_2(W))^{1/4}} +\sqrt{\dtot} \right) \right) .
\end{align*}
\end{proof}

\subsection{Lower Bound for the general convex case}
\label{app: Lower Bound for the general convex case}
In this section, we present the omitted details for the lower bound in the general convex case. For completeness, we first restate the theorem and then present its proof.
\lowerbound*

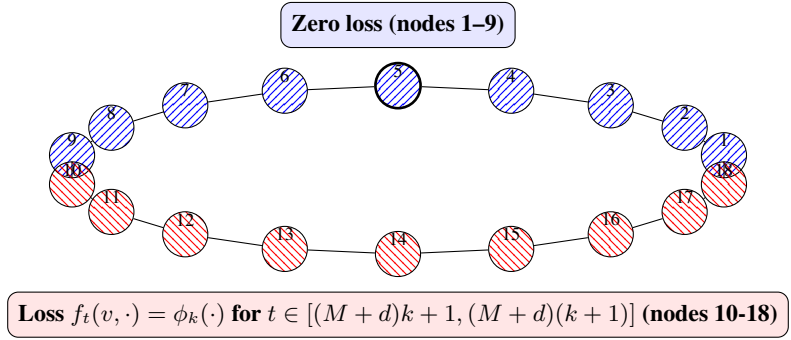
\begin{figure}
\begin{center}
\begin{tikzpicture}[scale=0.8, 
  every node/.style={circle, draw, minimum size=6mm, inner sep=0pt}]

  \def\n{18}
  \def\rx{5.5}
  \def\ry{1.4}

  \foreach \i in {1,...,18} {
    \pgfmathsetmacro{\angle}{360/\n * (\i-1) +10}
    \pgfmathsetmacro{\x}{\rx * cos(\angle)}
    \pgfmathsetmacro{\y}{\ry * sin(\angle)}
    \ifnum\i=5
      \node[pattern=north east lines, pattern color=blue, thick, line width=1pt] (N\i) at (\x,\y) {};
    \else
      \ifnum\i<10
        \node[pattern=north east lines, pattern color=blue] (N\i) at (\x,\y) {};
      \else
          \node[pattern=north west lines, pattern color=red] (N\i) at (\x,\y) {};
      \fi
    \fi
  }

  \foreach \i in {1,...,17} {
    \pgfmathtruncatemacro{\j}{\i + 1}
    \draw (N\i) -- (N\j);
  }
  \draw (N18) -- (N1);

  \foreach \i in {1,...,18} {
    \node[draw=none, fill=none] at ($(N\i)+(0,0.25)$) {\scriptsize \i};
  }

  \node[rectangle, fill=blue!10, rounded corners, inner sep=4pt, above=0.5cm of $(N5)$] 
    (L) {\footnotesize \textbf{Zero loss (nodes 1–9)}};

  \node[rectangle, fill=red!10, rounded corners, inner sep=4pt, below=0.5cm of $(N14)$] 
    (ZL) {\footnotesize \textbf{Loss $f_t(v,\cdot) = \phi_k(\cdot)$ for $t \in [(M + d)k + 1, (M + d)(k + 1)]$ (nodes 10-18)}};
\end{tikzpicture}
\end{center}
\caption{Example of configuration used in the lower bound}
\label{fig:ill_LB}
\end{figure}

\begin{proof}
We consider the setting where all delays are fixed and equal to \( d \), and  we let \( G \) denote a cycle graph with \( N = 2(M+1) \) nodes where $M$ is even, to simplify. In the example shown in \fref{fig:ill_LB}, we have such a graph with $M=4$.

For the \( N \)-cycle graph, the smallest nonzero and largest eigenvalues of the Laplacian are given by \( \sigma_{N-1}(\text{Lap}(G)) = 2 - 2\cos(2\pi/N) \) and \( \sigma_1(\text{Lap}(G)) = 4 \), respectively~\citep[Chapter~5.5]{spielman2019spectral}. Applying the inequality \( 1 - \cos(x) \ge x^2/5 \) for all \( x \in [0, \pi] \) (which holds since \( N \ge 4 \Rightarrow 2\pi/N \le \pi \)), we obtain
\[
\sigma_{N-1}(\text{Lap}(G)) \ge \frac{8\pi^2}{5N^2}.
\]
Next, we bound the inverse of the  spectral gap $\frac{1}{1 - \sigma_2(W)}$.
First, observe that  \(\sigma_2(W)= \sigma_2(I-  \frac{1}{\sigma_1(\text{Lap}(G))}\text{Lap}(G)) = \sigma_2(I- \frac{1}{4}\text{Lap}(G)) \)  is the second highest eigenvalue of $I -\frac{1}{4}\text{Lap}(G)$.
The eigenvalues of $I -\frac{1}{4}\text{Lap}(G)$ can be expressed as $1- \lambda/4$, where $\lambda$ is an eigenvalue of $\text{Lap}(G)$, so $\sigma_2(W)= 1- \frac{1}{4}\sigma_{N-1}(\text{Lap}(G))$. Hence $
\frac{1}{1-\sigma_2(W)}\le \frac{5N^2\cdot 4}{8 \pi^2}$, which directly yields :
\begin{equation}\label{eq:rho}
\frac{1}{1-\sigma_2(W)}\le \frac{N^2}{2}.\end{equation}
Now suppose that for a subset of \( M + 1\) nodes, the local loss functions are identically zero at all times: 
\[
f_t(1, \cdot) = \cdots = f(M/2+1, \cdot) = \cdots = f_t(M+1, \cdot) = 0 \quad \forall t \in [T].
\]
Further suppose that the remaining nodes update their loss functions every \( M + d \) rounds. Specifically, for each \( k = 0, \ldots, \lceil T / (M + d) \rceil - 1 \),
\[
f_t(M+2, \cdot) = \cdots = f_t(2M+2, \cdot) = \phi_k(\cdot) \quad \text{for } t \in [(M + d)k + 1, (M + d)(k + 1)],
\]
where \( \phi_k(x) = \varepsilon_k L \langle w, x \rangle \),
with \( \varepsilon_k \) being i.i.d. Rademacher random variables ($\pm 1$ with probability 1/2), and the vector \( w \) 
is defined as \( w = (x_1 - x_2) / \|x_1 - x_2\|_2 \), for some \( x_1, x_2 \in \X \) such that \( \|x_1 - x_2\|_2 = D \).
 See \fref{fig:ill_LB} for an illustration.

The resulting global loss at time \( t \), observed by agent $M/2+1$ when it plays \( x \), is:
\[
\ell_t(x) \triangleq \sum_{u\in V}f_t(u,x) = M \phi_{\lceil t / (M + d) \rceil}(x).
\]

Due to the structure of the cycle, agent $M/2+1$ cannot receive information about any node in \( \{M+2, \ldots, N\} \) until at least \( M/2 + d \) time steps have passed. Thus, predictions \( x_{kt+1}(M/2+1), \ldots, x_{kt+M+d}(M/2+1) \) are made without access to \( \phi_k \).

Applying the standard lower bound from online learning~\citep[Theorem~5.1]{orabona2019modern}, we obtain: 
\begin{align*}
\E\left[\sum_{t=1}^T \ell_t\big(x_t(M/2+1)\big) - \min_{x \in \X} \sum_{t=1}^T \ell_t(x)\right]
&= (M+1) \E\left[ \sum_{k=0}^{\lceil T / (M + d) \rceil - 1} \sum_{t=k(M+d)+1}^{(k+1)(M+d)} \phi_k(x_t(M/2+1))\right.\\
&\qquad\left.- \min_{x \in \X}(M + d) \sum_{k=0}^{\lceil T / (M + d) \rceil - 1} \phi_k(x)\right] \\
&= (M+1)(M + d) \E\left[-\min_{x \in \X} \sum_{k=0}^{\lceil T / (M + d) \rceil - 1} \phi_k(x)\right] \\
&= (M+1)(M + d) L \E\left[\max_{x \in \X} \sum_{k=0}^{\lceil T / (M + d) \rceil - 1} \varepsilon_k \langle w, x \rangle\right] \\
&\ge M(M + d) L \E\left[\max_{x \in \{x_1, x_2\}} \sum_{k=0}^{\lceil T / (M + d) \rceil - 1} \varepsilon_k \langle w, x \rangle\right] \\
&= M(M + d) L D \E\left[\left| \sum_{k=0}^{\lceil T / (M + d) \rceil - 1} \varepsilon_k \right|\right] \\
&\ge M(M + d) L D \sqrt{\frac{T}{M + d}} \quad \text{(Khintchine inequality)} \\
&= M L D \sqrt{(M + d)T}.
\end{align*}

Thus, there exists a realization of \( \varepsilon_0, \ldots, \varepsilon_{\lceil T / (M + d) \rceil - 1} \) for which:
\[
\Reg_T \ge M L D \sqrt{(M + d)T}.
\]

Now, from~\eref{eq:rho}, we know:
\[
\frac{1}{1 - \sigma_2(W)} \le \frac{N^2}{2},\text{ so } M \ge \frac 1 4  N \ge \frac{\frac{1}{4}\sqrt{2}}{\sqrt{1 -\sigma_2(W)}}\,.
\]
which implies the lower bound:
\[
\Reg_T \ge \frac{N}{4} L D \sqrt{T}\sqrt{\frac{1}{2} \sqrt{\frac{1}{1 - \sigma_2(W)}} + d}.
\]
\end{proof}

\newpage

\section{Omitted Details in Section~\ref{sec: D-OCO with Strongly-Convex Loss Functions}}

\label{app: D-OCO with Strongly-Convex Loss
Functions}
In this section, we include the omitted details in Section~\ref{sec: D-OCO with Strongly-Convex Loss Functions}. For completeness, we first restate the theorem and then present its proof.
\fstupperboundsc*

\begin{proof}
We start the proof with some notations. With a slight abuse of notation, we also define $\bar{z}_{s-1}$ as the the cumulative received augmented gradients till block $s-1$ averaged over all agents in the strongly convex case:
\begin{align}
    \bar{z}_{s-1} &=\frac{1}{N} \sum_{l=1}^{s-1} \sum_{v \in V} y_{l}^+(v).\label{eq:barz_3}
    \end{align}
Direct calculation shows that
\begin{align*}\bar{z}_{s-1}
    &= \frac{1}{N} \sum_{l=1}^{s-1}\sum_{v \in V} \left (\sum_{\tau \in o_{lB+1}(v) \backslash o_{(l-1)B+1}(v)}   g_{\tau}(v) - \alpha B x_l(v) \right)  \tag{Definition of $y^{+}_{l}(v)$}
    \\&= \frac{1}{N}\sum_{v \in V}  \sum_{\tau \in o_{(s-1)B+1}(v)}  g_{\tau}(v)- \frac{1}{N}\sum_{l=1}^{s-1}\sum_{v \in V} \alpha B x_l(v),
\end{align*}
where the last inequality is due to $o_{1}(v)= \emptyset$ for any $v \in V$. Again with an abuse of notation, similar to the case where the loss functions are convex in general, we define $\bar{x}_s$, $\wt{z}_s$, and  $\wt{x}_s$ in the following. Specifically, we define $\bar{x}_s$, which is the FTRL strategy at block $s$ assuming the agent has the received gradient information among all agent: 
\begin{align}
    \bar{x}_{s}=\argmin_{x\in\calX} \left\{\left\langle  
    \bar{z}_{s-1}, x\right\rangle+\frac{\alpha(s-1)B}{ 2} \| x\|_2^2\right\} .\label{eq:defbarx_3}
\end{align}
We also define $\wt{z}_{s-1}$ as follows
\begin{align*}
    \tilde{z}_{s-1} &= \frac{1}{N} \sum_{l=1}^{s-1} \sum_{v \in V}\left(\sum_{\tau \in \mathcal{T}_l}    g_{\tau}(v) - \alpha B x_{l}(v) \right),
\end{align*}
where $\mathcal{T}_l=\{(l-1) B+1, \ldots, l B\}$, and define $\wt{x}_s$ to be the FTRL strategy with respect to $\wt{z}_{s-1}$:
\begin{align*}
\Tilde{x}_{s}&=\argmin_{x\in\calX} \left\{\left\langle 
    \tilde{z}_{s-1}, x\right\rangle+\frac{\alpha(s-1)B}{2} \| x\|_2^2\right\} \nonumber
\\&= \argmin_{x\in\calX} \left\{\left\langle 
    \frac{1}{N} \sum_{l=1}^{s-1} \sum_{\tau \in \mathcal{T}_l}  \sum_{v \in V}  g_{\tau}(v), x\right\rangle+ \frac{\alpha B}{2N} \sum_{l=1}^{s-1}\sum_{v \in V} \|x-x_l(v)\|_2^2\right\}. \nonumber
\end{align*}
Finally, we define 
\begin{align}
    G_{s}(x) &\triangleq \left\langle 
    \frac{1}{N} \sum_{l=1}^{s-1} \sum_{\tau \in \mathcal{T}_l}  \sum_{v \in V}  g_{\tau}(v), x\right\rangle+ \psi_s(x)
\end{align}
where $\psi_s(x)$ is defined as
\begin{equation*}
    \psi_s(x)\triangleq \frac{\alpha B}{2N} \sum_{l=1}^{s-1}\sum_{v \in V} \|x-x_l(v)\|_2^2.
\end{equation*}
Next, we apply a regret decomposition that almost mirrors the one in \lref{lem:reg_decomp} except that we use the property that all loss functions are now $\alpha$-strongly convex.
\begin{align}
    \Reg_T(u)  &= \sum_{t=1}^T  \sum_{v\in V} \left ( f_{t}(v, x_t(u)) -  f_{t}(v,x^*) \right)  \nonumber \\
    &= \sum_{t=1}^T  \sum_{v\in V} \left ( f_{t}(v, x_t(v)) -  f_{t}(v,x^*) \right) + \sum_{t=1}^T  \sum_{v\in V} \left ( f_{t}(v, x_t(u)) -  f_{t}(v,x_t(v)) \right) \nonumber \\
    &\leq \sum_{t=1}^T  \sum_{v\in V} \left ( \inner{g_t(v), x_t(v) - x^*} -\frac{\alpha}{2}\|x_t(v) - x^* \|_2^2\right) + L\sum_{t=1}^T  \sum_{v\in V} \|x_t(u) - x_t(v)\|_2 \nonumber \tag{\asref{asm:Lipschitz} and \asref{asm:stronglyconvex}} 
    \\&\leq \sum_{t=1}^T  \sum_{v\in V} \left ( \inner{g_t(v), x_{t}(v) + \bar{x}_t - \bar{x}_t  - x^*} -\frac{\alpha}{2}\|x_t(v) - x^* \|_2^2\right)  \nonumber
    \\&\qquad + L\sum_{t=1}^T  \sum_{v\in V} \left(\|x_t(u) -\bar{x}_{t}\|_2+\|x_t(v) -\bar{x}_{t}\|_2\right) \nonumber \tag{Triangular inequality}
    \\&\leq \sum_{t=1}^T  \sum_{v\in V} \left ( \inner{g_t(v), \bar{x}_t - x^*} -\frac{\alpha}{2}\|x_t(v) - x^* \|_2^2\right) \nonumber
    \\&\qquad + 2L\sum_{t=1}^T  \sum_{v\in V} \|x_t(v) -\bar{x}_{t}\|_2  + NL\sum_{t=1}^T\| x_t(u) - \bar{x}_t\|_2\tag{\asref{asm:Lipschitz}}
    \\&= \underbrace{\sum_{s=1}^{T/B} \sum_{t\in \mathcal{T}_s}  \sum_{v\in V} \left ( \inner{g_t(v), \bar{x}_s  - x^*} -\frac{\alpha}{2}\|x_s(v) - x^* \|_2^2\right)}_{\spadesuit} \nonumber
    \\&\qquad + \underbrace{2BL\sum_{s=1}^{T/B}   \sum_{v\in V} \|x_s(v) -\bar{x}_{s}\|_2  + NBL\sum_{s=1}^{T/B} \| x_s(u) - \bar{x}_s\|_2}_{\clubsuit} \label{ineq: regret_decompose_scx}
\end{align}
where the last equality holds because the algorithm uses the same decision over all time steps in the same block,  and the block length is $B$.

We first analyze the term $\spadesuit$ by decomposing it as follows:
    \begin{align}
    \frac{1}{N}\spadesuit  &= \frac{1}{N}\sum_{s=1}^{T/B} \sum_{t\in \mathcal{T}_s}  \sum_{v\in V} \left ( \inner{g_t(v), \bar{x}_s  - x^*} -\frac{\alpha}{2}\|x_s(v) - x^* \|_2^2\right)  \nonumber
    \\&= \underbrace{\frac{1}{N}\sum_{s=1}^{T/B} \sum_{t\in \mathcal{T}_s}  \sum_{v\in V} \left\langle g_t(v), \Tilde{x}_{s} -x^*\right\rangle}_{{\operatorname{full-info}_T}} + \underbrace{\frac{1}{N}\sum_{s=1}^{T/B} \sum_{t\in \mathcal{T}_s}  \sum_{v\in V} \left\langle g_t(v), \bar{x}_{s}- \Tilde{x}_{s}\right\rangle}_{{\operatorname{drift}_T}} \nonumber 
    \\&\quad-  \frac{1}{N}\sum_{s=1}^{T/B} \sum_{t\in \mathcal{T}_s}  \sum_{v\in V} \frac{\alpha}{2}\|x_s(v) - x^* \|_2^2 \nonumber
    \\&= \underbrace{\frac{1}{N}\sum_{s=1}^{T/B} \sum_{t\in \mathcal{T}_s}  \sum_{v\in V} \left\langle g_t(v), \Tilde{x}_{s} -x^*\right\rangle}_{{\operatorname{full-info}_T}} + \underbrace{\frac{1}{N}\sum_{s=1}^{T/B} \sum_{t\in \mathcal{T}_s}  \sum_{v\in V} \left\langle g_t(v), \bar{x}_{s}- \Tilde{x}_{s}\right\rangle}_{{\operatorname{drift}_T}}  \nonumber
    \\&\quad-  \frac{\alpha B}{2N}\sum_{s=1}^{T/B}   \sum_{v\in V} \|x_s(v) - x^* \|_2^2, \label{eq:spade_neg}
\end{align}
where the last equality is because $|\mathcal{T}_s|=B$. 
First, we analyze $\operatorname{full-info}_T$ by using Lemma~7.1 in \cite{orabona2019modern}:
\begin{align}
\operatorname{full-info}_T &= \frac{1}{N}\sum_{s=1}^{T/B} \sum_{t\in \mathcal{T}_s}  \sum_{v\in V} \left\langle g_t(v), \Tilde{x}_{s} -x^*\right\rangle \nonumber
\\&= \sum_{s=1}^{T/B} \left\langle \frac{1}{N}\sum_{t\in \mathcal{T}_s}  \sum_{v\in V}g_t(v), \Tilde{x}_{s} -x^*\right\rangle \nonumber
\\& = \psi_{T/B+1}(x^*) - \min_{x \in \calX}\psi_1(x)  + G_{T/B+1}(\tilde{x}_{T/B+1}) - G_{T/B+1}(x^*) \nonumber \\
&\quad + \sum_{s=1}^{T/B} \left[G_{s}(\tilde{x}_{s}) - G_{s+1}(\tilde{x}_{s+1})+ \inner{\frac{1}{N}\sum_{t\in \mathcal{T}_s}  \sum_{v\in V}g_t(v),\tilde{x}_s}\right]  \tag{Lemma 7.1 in \citet{orabona2019modern}}\nonumber
\\& \leq \psi_{T/B+1}(x^*)  +\sum_{s=1}^{T/B} \left[ \left(G_{s}(\tilde{x}_{s})+ \inner{\frac{1}{N}\sum_{t\in \mathcal{T}_s}  \sum_{v\in V}g_t(v),\tilde{x}_s} \right)
\right. \nonumber
\\&\left. \quad - \left( G_{s}(\tilde{x}_{s+1})  + \inner{\frac{1}{N}\sum_{t\in \mathcal{T}_s}  \sum_{v\in V}g_t(v),\tilde{x}_{s+1}}\right) \right] + \sum_{s=1}^{T/B} \left( \psi_{s}(\tilde{x}_{s+1}) - \psi_{s+1}(\tilde{x}_{s+1})\right)\nonumber
\\& \leq \psi_{T/B+1}(x^*) +\sum_{s=1}^{T/B} \left[ \left(G_{s}(\tilde{x}_{s})+ \inner{\frac{1}{N}\sum_{t\in \mathcal{T}_s}  \sum_{v\in V}g_t(v),\tilde{x}_s} \right)\right. \nonumber
\\& \left.\quad- \left( G_{s}(\tilde{x}_{s+1})  + \inner{\frac{1}{N}\sum_{t\in \mathcal{T}_s}  \sum_{v\in V}g_t(v),\tilde{x}_{s+1}}\right) \right],
\label{eq:full_info_sc}
\end{align}
where the first inequality holds because $G_{T/B+1}(\tilde{x}_{T/B+1}) \le G_{T/B+1}(x^*)$ by optimality of $\tilde{x}_{T/B+1}$ and together with non-negativity of $\psi_{1}$, and the second inequality holds because 
\begin{equation*}
    \sum_{s=1}^{T/B} \left( \psi_{s}(\tilde{x}_{s+1}) - \psi_{s+1}(\tilde{x}_{s+1})\right) = -\frac{\alpha B}{2 N} \sum_{s=1}^{T / B} \sum_{v \in V}\left\|\tilde{x}_{s+1}-x_s(v)\right\|_2^2 \leq 0.
\end{equation*}

As for the difference between  $\left(G_{s}(\tilde{x}_{s})+ \inner{\frac{1}{N}\sum_{t\in \mathcal{T}_s}  \sum_{v\in V}g_t(v),\tilde{x}_s} \right) $ and $ \left( G_{s}(\tilde{x}_{s+1})  + \inner{\frac{1}{N}\sum_{t\in \mathcal{T}_s}  \sum_{v\in V}g_t(v),\tilde{x}_{s+1}}\right)$, direct calculation shows that 
\begin{align*}
    &\left(G_{s}(\tilde{x}_{s})+ \inner{\frac{1}{N}\sum_{t\in \mathcal{T}_s}  \sum_{v\in V}g_t(v),\tilde{x}_s} \right) - \left( G_{s}(\tilde{x}_{s+1})  + \inner{\frac{1}{N}\sum_{t\in \mathcal{T}_s}  \sum_{v\in V}g_t(v),\tilde{x}_{s+1}}\right) \nonumber
    \\&\leq \inner{\frac{1}{N}\sum_{t\in \mathcal{T}_s}  \sum_{v\in V}g_t(v),\tilde{x}_{s}-\tilde{x}_{s+1}} \tag{since $G_{s}(\tilde{x}_{s})=\min_{x\in\calX}G_s(x) \leq G_{s}(\tilde{x}_{s+1})$ }
    \\&\leq BL\|\tilde{x}_{s}-\tilde{x}_{s+1}\|_2. 
\end{align*}
To bound $\|\wt{x}_s-\wt{x}_{s+1}\|_2$, applying \lref{lem:ftrl_sc} with $w_1=\frac{1}{N} \sum_{l=1}^{s-1} \sum_{\tau \in \mathcal{T}_l}  \sum_{v \in V}  g_{\tau}(v)$, $w_2=\frac{1}{N} \sum_{l=1}^{s} \sum_{\tau \in \mathcal{T}_l}  \sum_{v \in V}  g_{\tau}(v)$, $\psi_1=\psi_s$, $\psi_2=\psi_{s+1}$ shows that
\begin{align*}
    \frac{\alpha B(2s-1)}{2}\|\wt{x}_{s+1}-\wt{x}_s\|_2^2 &\leq \inner{\frac{1}{N}\sum_{\tau \in \calT_s}\sum_{v\in V}g_\tau(v),\wt{x}_s-\wt{x}_{s+1}} - \frac{\alpha B}{2N}\sum_{v\in V}\|\wt{x}_{s+1}-x_s(v)\|_2^2 
    \\& \quad+ \frac{\alpha B}{2N}\sum_{v\in V}\|\wt{x}_{s}-x_s(v)\|_2^2 \\
    &\leq BL\|\wt{x}_s-\wt{x}_{s+1}\|_2 + \frac{\alpha B}{2N}\|\wt{x}_s-\wt{x}_{s+1}\|_2\cdot \|\wt{x}_s+\wt{x}_{s+1}-x_s(v)\|_2 \\
    &\leq BL\|\wt{x}_s-\wt{x}_{s+1}\|_2 + \alpha BD\|\wt{x}_s-\wt{x}_{s+1}\|_2.
\end{align*}
Rearranging the terms leads to
\begin{equation*}
    \|\tilde{x}_{s}-\tilde{x}_{s+1}\|_2 \leq 
    \frac{2}{\alpha (2 s-1)}L +\frac{2D}{2s-1}.
\end{equation*}
Plugging the above into  \eref{eq:full_info_sc}, we obtain 
\begin{align*}
    \operatorname{full-info}_T &\leq \psi_{T/B+1}(x^*) +  \frac{2BL(L +\alpha D)}{\alpha}\sum_{s=1}^{T/B} \frac{1}{2s-1} 
    \\&\leq \psi_{T/B+1}(x^*) + \frac{2BL(L +\alpha D)}{\alpha}\ln(2T/B)
\end{align*}
Combining with the negative term in \eref{eq:full_info_sc}, we know that
\begin{align}
    \operatorname{full-info}_T - \frac{\alpha B}{2N}\sum_{s=1}^{T/B}\sum_{v\in V}\|x_s(v) - x^*\|_2^2 \leq \frac{2BL(L +\alpha D)}{\alpha}\ln(2T/B).\label{ineq:full_info_sc}
\end{align}

Now we turn to analyze $\operatorname{drift}_T$ in the term $\spadesuit$. 
\begin{align}
    \frac{1}{N}\operatorname{drift}_T & = \frac{1}{N}\sum_{s=1}^{T/B} \sum_{t\in \mathcal{T}_s}  \sum_{v\in V} \left\langle g_t(v), \bar{x}_{s}- \Tilde{x}_{s}\right\rangle \nonumber
    \\& = \sum_{s=1}^{T/B} \left\langle  \frac{1}{N}\sum_{t\in \mathcal{T}_s}  \sum_{v\in V} g_t(v), \bar{x}_{s}- \Tilde{x}_{s}\right\rangle \nonumber
    \\& \leq BL \sum_{s=1}^{T / B}  \left\|\bar{x}_{s}-\Tilde{x}_{s}\right\|_2 \tag{Cauchy–Schwarz inequality, \asref{asm:Lipschitz} and $|\mathcal{T}_s|=B$}
    \\& = BL \sum_{s=2}^{T / B} \left\|\bar{x}_{s}-\Tilde{x}_{s}\right\|_2 \tag{$\bar{x}_{1}=\Tilde{x}_{1}=\mathbf{0}$}
    \\& \leq BL \sum_{s=2}^{T/B}\frac{1}{(s-1)B \alpha} \left\|\bar{z}_{s-1}-\Tilde{z}_{t-1}\right\|_2 \tag{\lref{lem:ftrl_sc}}
    \\& \leq \frac{L}{\alpha} \sum_{s=2}^{T/B}  \frac{1}{s-1} \left\|\frac{1}{N} \sum_{v \in V} \sum_{\tau \in o_{(s-1)B+1}(v)}  g_{\tau}(v)-\frac{1}{N} \sum_{l=1}^{s-1} \sum_{\tau \in \mathcal{T}_l} \sum_{v \in V}  g_{\tau}(v)\right\|_2
    \tag{definition of $\bar{z}_{t-1}$ and $\tilde{z}_{t-1}$}
    \\& = \frac{L}{\alpha} \sum_{s=2}^{T/B}  \frac{1}{s-1 }\left\|-\frac{1}{N}  \sum_{v \in V} \sum_{\tau \in m_{(s-1)B+1}(v)}  g_{\tau}(v)\right\|_2
    \tag{$\mathcal{T}_s=\{(s-1) B+1, \ldots, s B\} $ and $m_{t}(v)=[t-1] \backslash o_{t}(v)$ }
    \\& \leq \frac{L^2}{\alpha} \sum_{s=2}^{T/B}  \frac{1}{s-1} \left(\frac{1}{N}\sum_{v \in V} |m_{(s-1)B+1}(v)|\right) \tag{\asref{asm:Lipschitz}}
    \\& \leq \frac{\delta_{\max}L^2}{\alpha} \sum_{s=2}^{T/B} \frac{1}{s-1} 
\tag{$\delta_{\max}=\max_{\substack{t \in [T]}}\frac{1}{N}\sum_{u \in V}|m_t(u)|$}
    \\& \leq \frac{\delta_{\max}L^2}{\alpha} \left(\ln (T / B)+1\right),\label{ineq:drift_T_acc_before_eta_sc}
\end{align}
where the second inequality applies \lref{lem:ftrl_sc} using the definition of $\bar{x}_s$ and $\wt{x}_s$, and the last inequality is due to $\sum_{s=2}^{T / B} \frac{1}{s-1} \leq \ln (T / B)+1$. When $d_t(u)=d(u)$ for all $u\in [N]$, we can further upper bound $\frac{1}{N}\sum_{v \in V} |m_{(s-1)B+1}(v)|$ by $\frac{1}{N}\sum_{u\in V}d(u)\triangleq \bar{d}$, leading to
\begin{align}\label{eqn:ineq:drift_T_acc_before_eta_sc_du}
    \frac{1}{N}\operatorname{drift}_T \leq \frac{\bar{d}L^2}{\alpha}(\ln(T/B)+1).
\end{align}

Combining \eref{eq:spade_neg}, \eref{ineq:full_info_sc} and \eref{ineq:drift_T_acc_before_eta_sc}, we have 
\begin{equation}
    \frac{1}{N} \spadesuit \leq  \frac{\delta_{max}L^2}{\alpha} \left( \ln{\left(T/B\right)}+1 \right) + \frac{2BL(L +\alpha D)}{\alpha}\ln(2T/B). \label{eq:spade_sc}
\end{equation}
We now turn to the analysis of the term \( \clubsuit \). By definition, \( x_1(v) = \bar{x}_1 = 0 \), which implies \( \| x_1(v) - \bar{x}_1 \|_2 = 0 \). To bound \( \| x_{s+1}(u) - \bar{x}_{s+1} \|_2 \) for any \( s \geq 1 \) and \( u \in V \), we proceed as follows.
\begin{align}
    &\sum_{s=1}^{T/B-1} \|x_{s+1}(u) - \bar{x}_{s+1} \|_2\\
    &\leq \sum_{s=1}^{T/B-1} \frac{1}{\alpha sB} \|z_{s}(u) - \bar{z}_{s}\|_2 \tag{\lref{lem:ftrl_sc}}
    \\& =\frac{2}{N\sqrt{N}}\sum_{s=1}^{T/B-1}\frac{1}{\alpha sB} \sum_{l=1}^{s-1} b^{(s-l-1) B}\left(\sqrt{\sum_{v\in V}\left\|y_l^{+}(v)\right\|_2^2}\right) \tag{\lref{lemma: bar_zt_acc_sc}}
    \\& \leq \frac{2}{N\sqrt{N}} \sum_{l=1}^{T/B - 1} \left( \sqrt{ \sum_{v \in V} \left\| y_l^{+}(v) \right\|_2^2 } \sum_{s=l+1}^{T/B} \frac{1}{\alpha sB}  b^{(s - l - 1)B} \tag{swap the summation order}\right)
    \\& \leq\frac{2}{N\sqrt{N}} \sum_{l=1}^{T/B - 1} \left( \frac{1}{\alpha (l+1) B}\sqrt{ \sum_{v \in V} \left\| y_l^{+}(v) \right\|_2^2 } \sum_{s=l+1}^{T/B} b^{(s - l - 1)B} \right) \nonumber
    \\& \leq \frac{2}{N\sqrt{N}}\frac{1}{1-\frac{1}{\sqrt{14 N}}}  \sum_{l=1}^{T/B - 1} \left( \frac{1}{\alpha (l+1)B}\sqrt{ \sum_{v \in V} \left\| y_l^{+}(v) \right\|_2^2 } \right) \tag{\eref{ineq: bb} and Geometric sum}
    \\& \leq \frac{3}{N\sqrt{N}}\sum_{l=1}^{T/B - 1} \left( \frac{1}{\alpha (l+1)B}\sqrt{ \sum_{v \in V} \left\| y_l^{+}(v) \right\|_2^2 } \right),
    \label{ineq:gossip_acc_2_sc}
\end{align} 
where the last inequality follows from $N\ge 1$. Plugging in the definition of $y_l^{+}(v)$ in \eref{ineq:gossip_acc_2_sc}, we obtain that
\begin{align}
    &\sum_{s=1}^{T/B-1} \|x_{s+1}(u) - \bar{x}_{s+1} \|_2 \nonumber\\
    &\leq \frac{3}{N\sqrt{N}}\sum_{s=1}^{T/B - 1} \left( \frac{1}{\alpha (s+1)B}\sqrt{ \sum_{v \in V} \left\| y_s^{+}(v) \right\|_2^2 } \right) \nonumber
    \\&= \frac{3}{N\sqrt{N}}\sum_{s=1}^{T/B - 1} \left( \frac{1}{\alpha (s+1)B}\sqrt{ \sum_{v \in V} \left\| \sum_{\tau \in  o_{sB+1}(v) \backslash o_{(s-1)B+1}(v)} g_{\tau}(v) - \alpha Bx_{s}(u) \right\|_2^2}  \right)\nonumber
    \\  &\leq \frac{3}{N\sqrt{N}}\sum_{s=1}^{T/B-1} \left( \frac{L}{\alpha (s+1)B}\sqrt{ \sum_{v \in V}  \left( |o_{sB+1}(v)| - |o_{(s-1)B+1}(v)|\right)^2} \right) \nonumber 
    \\& \quad+ \frac{3}{N\sqrt{N}}\sum_{s=1}^{T/B-1} \left( \frac{D}{(s+1)B\alpha}\sqrt{N\alpha^2B^2} \right)\tag{Triangular inequality}
    \\&\leq \frac{3}{N\sqrt{N}}\sum_{s=1}^{T/B - 1} \left( \frac{L}{(s+1)B\alpha}\sqrt{ \sum_{v \in V}  \left( |o_{sB+1}(v)| - |o_{(s-1)B+1}(v)|\right)^2} \right) \nonumber 
    \\& \quad+ \frac{3D}{N}\left(\ln{\left( T/B\right)} +1\right), \label{eq: club_1}
\end{align}
where the last inequality is due to $\sum_{z=1}^{T / B-1} \frac{1}{s+1} \leq \ln (T / B)+1$. 
by definition of $o_t(v)$, we observe that 
\begin{align*}
    |o_{sB+1}(v)| - |o_{(s-1)B+1}(v)| &= |sB| - |m_{sB+1}(v)| - \left(|(s-1)B| - |m_{(s-1)B+1}(v)|\right)
    \\&= B + |m_{(s-1)B+1}(v)| - |m_{sB+1}(v)|. 
\end{align*}
Hence, we obtain
\begin{align}
    \sqrt{ \sum_{v \in V}  \left( |o_{sB+1}(v)| - |o_{(s-1)B+1}(v)|\right)^2}  &\leq \sqrt{ \sum_{v \in V}  \left( B + |m_{(s-1)B+1}(v)| - |m_{sB+1}(v)| \right)^2}\nonumber
    \\&\leq \sqrt{ \sum_{v \in V}  B^2} + \sqrt{ \sum_{v \in V}  \left(|m_{(s-1)B+1}(v)| - |m_{sB+1}(v)| \right)^2}  \tag{triangle inequality}
    \\&\leq B\sqrt{N} + \sum_{v \in V}|m_{(s-1)B+1}(v)| + \sum_{v \in V}|m_{sB+1}(v)| \nonumber
    \\&\leq B\sqrt{N} + 2N \delta_{\max}, \label{eq: club_2}
\end{align}
where the last inequality is due to the definition of $\delta_{\max}$.
Combining \eref{eq: club_1} and \eref{eq: club_2} and using $\sum_{z=1}^{T / B-1} \frac{1}{s} \leq \ln (T / B)+1$, we obtain
\begin{align}
     \sum_{s=1}^{T/B-1} \|x_{s+1}(u) - \bar{x}_{s+1} \|_2 &\leq  \frac{3(\alpha D+L)}{N\alpha}\left(\ln{\left( T/B\right)} +1\right) + \frac{6\delta_{\max}L}{\sqrt{N}\alpha B}\left(\ln{\left( T/B\right)} +1\right). \label{eq:club_3}
\end{align}

% Furthermore, when $d_t(u)=d(u)$ for all $t\in[T]$, due to the definition of $y_l^{+}(v)$, each agent $v \in V$ can receive at most $B$ gradients and actions in any block $s \in T/B$. Hence, we obtain 
% \begin{align}
%         \sum_{s=2}^{T/B} \|x_{s}(v) - \bar{x}_{s} \|_2& \leq \frac{3}{N\sqrt{N}} \sum_{l=1}^{T/B - 1} \left( \frac{1}{(l+1)sB} \sqrt{ \sum_{v \in V} \left\| y_l^{+}(v) \right\|_2^2 } \right) \tag{\eref{ineq:gossip_acc_2_sc}}
%         \\& \leq \frac{3(\alpha D+L)}{N\alpha}\left(\ln{\left( T/B\right)} +1\right). \label{ineq:club_4}
% \end{align}
Finally, we obtain
\begin{align*}
    \Reg_T(u)  &= \sum_{s=1}^{T/B} \sum_{t\in \mathcal{T}_s}  \sum_{v\in V} \left ( \inner{g_t(v), \bar{x}_s  - x^*} -\frac{\alpha}{2}\|x_s(v) - x^* \|_2^2\right) \nonumber
    \\&\quad + 2BL\sum_{s=1}^{ T/B}   \sum_{v\in V} \|x_s(v) -\bar{x}_{s}\|_2  + NBL\sum_{s=1}^{T/B} \| x_s(u) - \bar{x}_s\|_2 \tag{\eref{ineq: regret_decompose_scx}}
    \\&\leq \frac{N\delta_{max}L^2}{\alpha} \left( \ln{\left(T/B\right)}+1 \right) + \frac{2NBL(L +\alpha D)}{\alpha}\ln(2T/B)
    \\&\quad + 2BL\sum_{s=1}^{ T/B}   \sum_{v\in V} \|x_s(v) -\bar{x}_{s}\|_2  + NBL\sum_{s=1}^{T/B} \| x_s(u) - \bar{x}_s\|_2 \tag{\eref{eq:spade_sc}}
    \\&\leq \frac{N\delta_{max}L^2}{\alpha} \left( \ln{\left(T/B\right)}+1 \right) + \frac{2NBL(L +\alpha D)}{\alpha}\ln(2T/B)
    \\&\quad + \frac{9B(\alpha DL+L^2)}{\alpha}\left(\ln{\left( T/B\right)} +1\right) + \frac{18\sqrt{N}\delta_{\max}L^2}{\alpha}\left(\ln{\left( T/B\right)} +1\right) \tag{\eref{eq:club_3}}
    \\ & = {\mathcal{O}}\left(  \frac{N(\alpha DL + L^2)}{\alpha} \left( \delta_{\max} + \frac{\ln(N)}{\sqrt{1-\sigma_2(W)}}\right)  \ln{\left(T\right)}\right). 
\end{align*}
When $d_t(u)=d(u)$ for all $t\in[T]$, we define $\Bar{d} = \frac{1}{N} \sum_{v \in V} d(v)$, then we have $\Bar{d} = \delta_{\max}$ and obtain
\begin{equation*}
    \Reg_T(u) \leq  
    {\mathcal{O}}\left(  \frac{N(\alpha DL + L^2)}{\alpha} \left( \bar{d} + \frac{\ln(N)}{\sqrt{1-\sigma_2(W)}}\right)  \ln{\left(T\right)}\right).
\end{equation*}
\end{proof}

\subsection{Lower bound for the strongly convex case}\label{app:strong_lower_bound}
In this section, we provide the proof for the lower bound for the strongly convex case.
\begin{restatable}{theorem}{lowerboundsc}
\label{th:lowerboundsc} 
Let \( d \) be a constant feedback delay experienced by each agent in the network, and let \( A \) be any algorithm for D-OCO over the domain \( \X \subseteq\R^n\). Then, there exists a graph \( G = ([N], E) \), with $N=2(M+1)$ where $M$ is an even integer and $16(N+d) + 1 \leq T$, and a sequence of $\alpha D $-Lipschitz and $\alpha$-strongly convex loss functions assigned to the agents, denoted by
\[
\left\{ f_1(1, \cdot), \ldots, f_1(N, \cdot) \right\}, \ldots, \left\{ f_T(1, \cdot), \ldots, f_T(N, \cdot) \right\},
\]
such that the regret of algorithm \( A \) satisfies the lower bound:
\[
\Reg_T = {\Omega} \left(\alpha  N D^2 \left(\frac{1}{\sqrt{1 - \sigma_2(W)}}+d\right) \ln\left(\frac{T }{\frac{1}{\sqrt{1 - \sigma_2(W)}}+d}  \right)\right),
\]
where \( W   = I - \frac{1}{\sigma_1(\text{Lap}(G))} \cdot \text{Lap}(G) \).
\end{restatable}

\begin{proof}
This proof is an adaptation of that of \citet[Theorem 4]{wan2024optimalefficientalgorithmsdecentralized}. Specifically, we consider the setting where all delays are fixed and equal to \( d \), and we let \( G \) denote a cycle graph with \( N = 2(M+1) \) nodes where $M$ is even, to simplify, and $V = \{1,2,\dots, N\}$.

For any D-OCO algorithm $A$, we denote the sequence of decisions made by agent $u\in [N]$ as $x_1(u), \dots, x_T(u)$. We divide the total $T$ rounds into the following $Z+1$ blocks:
\begin{align}
\label{no-communication-intervals}
[c_0+1, c_1], [c_1+1, c_2], \dots, [c_Z+1, c_{Z+1}]
\end{align}
where $Z = \lfloor (T-1)/\tau \rfloor$, $\tau =  M/2 + d$, $c_{Z+1} = T$, and $c_i = i\tau$, for $i = 0, \dots, Z$.

At each round $t$, we set:
\begin{align*}
f_t(u,x) = \frac{\alpha}{2} \|x\|_2^2 \quad \text{for } u \in \{1, \ldots, M+1\},
\end{align*}
which is $\alpha$-strongly convex and satisfies Assumption~\ref{asm:Lipschitz} with $L = \alpha D$ over the set $\X = [0, D/\sqrt{n}]^n$.

Let $\mathcal{B}_p$ denote the Bernoulli distribution with success probability $p$ and $\mathcal{B}_p^n$ the distribution of vectors whose coordinates are equal to each other and the value is drawn from $\mathcal{B}_p$. For any $i \in \{0, \dots, Z\}$ and $t \in [c_i+1, c_{i+1}]$, define:
\begin{align*}
f_t(u,x) = \phi_i(x) = \frac{\alpha}{2} \left\|x - \frac{D \mathbf{w}_i}{\sqrt{n}} \right\|_2^2, \quad \text{for } u \in \{ M+2, \dots, N\},
\end{align*}
where $\mathbf{w}_i \in \{\mathbf{0}, \mathbf{1}\}$ is sampled from $\mathcal{B}_p^n$, meaning that with probability $p$,  $\mathbf{w}_i=\mathbf{1}$; otherwise, $\mathbf{w}_i=\mathbf{0}$.

Then, the global loss function at time $t$ is: 
\begin{align*}
\ell_t(x) &\triangleq \sum_{u=1}^{N} f_t(u,x) \\
&= \frac{\alpha(M+1)}{2} \left\|x - \frac{D \mathbf{w}_i}{\sqrt{n}} \right\|_2^2 + \frac{\alpha(M+1)}{2} \|x\|_2^2 \\
&= \frac{\alpha N}{2} \|x\|_2^2 - \frac{\alpha (M+1) D}{\sqrt{n}} \langle x, \mathbf{w}_i \rangle + \frac{\alpha (M+1) D^2}{2n} \|\mathbf{w}_i\|_2^2.
\end{align*}

Taking expectation, we obtain that
\begin{align*}
\mathbb{E}_{\mathbf{w}_i}[\ell_t(x)] &= \frac{\alpha N}{2} \|x\|_2^2 + \frac{\alpha (M+1) D}{\sqrt{n}} \langle x, \mathbf{p} \rangle + \frac{\alpha (M+1) D^2}{2n} \langle \mathbf{1}, \mathbf{p} \rangle \\
&= \frac{\alpha N}{2} \left\| x - \frac{ (M+1) D\mathbf{p}}{N \sqrt{n}} \right\|_2^2 + \frac{\alpha (M+1) D^2}{2n} \left\langle \mathbf{1} - \frac{(M+1) \mathbf{p}}{N}, \mathbf{p} \right\rangle,
\end{align*}
where $\mathbf{p} = p\cdot \mathbf{1}$. Let $F(x) \triangleq \mathbb{E}_{\mathbf{w}_i}[\ell_t(x)]$. Then, direct calculation shows that the minimizer of $F(x)$ has the following form:
\begin{align*}
x^* = \frac{ (M+1) D\cdot  \mathbf{p}}{N \sqrt{n}} = \xi \mathbf{p},
\end{align*}
where  $\xi$ is defined as $\frac{ (M+1) D\mathbf{p}}{N \sqrt{n}}$,
and that for any $x \in \calX$, we have
\begin{align}
\label{extended-eq1}
F(x) - F(x^*) = \frac{\alpha N}{2} \left\| x - \frac{ (M+1) D\mathbf{p}}{N \sqrt{n}} \right\|_2^2  = \frac{\alpha N}{2} \left\| x - \xi \mathbf{p} \right\|_2^2 \geq 0.
\end{align}

Moreover, according to Jensen's inequality, we have
\begin{align}
\label{extended-eq2}
\mathbb{E}_{\mathbf{w}_0, \dots, \mathbf{w}_Z} \left[ \min_{x \in \mathcal{X}} \sum_{i=0}^Z \sum_{t=c_i+1}^{c_{i+1}} \ell_t(x) \right] \leq \sum_{i=0}^Z \sum_{t=c_i+1}^{c_{i+1}} F(x^*).
\end{align}

Because of the feedback delay $d$ and the delay $M/2+1$ induced by communication in the graph, the decisions $x_{c_i+1}(M/2+1), \dots, x_{c_{i+1}}(M/2+1)$ are independent of $\mathbf{w}_i$. Thus:
\begin{align}
\mathbb{E}_{\mathbf{w}_0, \dots, \mathbf{w}_Z}[\Reg_{T}(M/2+1)] &= \mathbb{E}\left[ \sum_{i=0}^Z \sum_{t=c_i+1}^{c_{i+1}} \ell_t(x_t(M/2+1)) - \min_{x \in \mathcal{X}} \sum_{i=0}^Z \sum_{t=c_i+1}^{c_{i+1}} \ell_t(x) \right] \nonumber\\
&
\geq\sum_{i=0}^Z \sum_{t=c_i+1}^{c_{i+1}} \left( \mathbb{E}[F(x_t(M/2+1))] - F(x^*) \right) \tag{using \eref{extended-eq2}}\\
&= \sum_{i=0}^Z \sum_{t=c_i+1}^{c_{i+1}} \mathbb{E}[F(x_t(M/2+1)) - F(x^*)]. \label{extended-eq3}\\
&= \sum_{i=0}^Z \sum_{t=c_i+1}^{c_{i+1}}\frac{\alpha N}{2}  \mathbb{E}\left[\left\| x_t(M/2) - \xi \mathbf{p}\right\|_2^2\right]. \label{extended-eq4}
\end{align}

To achieve a lower bound on \eqref{extended-eq4} , we assume without loss of generality that the D-OCO algorithm is deterministic.
\footnote{This reduction is also used in \citet{wan2024optimalefficientalgorithmsdecentralized} and dates back to \citet{hazan2014beyond}. Specifically, the analysis can be directly generalized to randomized algorithm as discussed in Footnote 3 of \citet{wan2024optimalefficientalgorithmsdecentralized}.}
Recall that given $i\in\{0,1,2,\dots, Z\}$, for each round $t \in [c_i+1, c_{i+1}]$, all local functions $\{ f_t(1,x), \dots, f_t(N,x)\}$  are jointly dependent on the same random vector $\mathbf{w}_i\in \{\mathbf{0},\mathbf{1}\}$ sampled from the Bernoulli distribution $\mathcal{B}_p^n$. Consequently, the decision $x_t(M/2+1)$ made by agent $M/2+1$ at time $t\in[c_i+1,c_{i+1}]$ can be expressed as a deterministic function of a sequence $X \in \{\mathbf{0},\mathbf{1}\}^i$ , where $X$ is sampled from $(\mathcal{B}_p^n)^i$, where $(\mathcal{B}_p^n)^i$ represents the joint probability law of $i$ independent draws from $\mathcal{B}_p^n$ (used to sample the $(\textbf{w}_j)_{j\le i}$). That is, $x_t(M/2+1) = \mathcal{A}_t(X)$ for some mapping $\mathcal{A}_t: \{0,1\}^i \to \mathcal{X}$.

We will use \lref{extended-lem2} below, which has been proven in \cite{wan2024nearly} and which shows that for any algorithm, there exists a $p$ that induces a non-trivial gap from the optimal solution $\xi \mathbf{p}$ at every time step. This gap evolves over time as a function of the epochs, which we define next.

Let $ K = \left\lfloor \log_{16}(15Z + 16) - 1 \right\rfloor $. Assuming that $T$ is sufficiently large such that $ 16(N + d) + 1 \leq T$, we know that $ K \geq 1 $. We partition the first  
$
Z' = \frac{1}{15}(16^{K+1} - 16)
$
blocks into $ K $ epochs, where the $k$-th epoch spans $r_k = 16^k$ blocks for $k = 1, 2, \dots, K$. Specifically, epoch $k$ corresponds to the block indices:
\[
E_k = \left\{ \frac{1}{15}(16^k - 16), \dots, \frac{1}{15}(16^{k+1} - 16) - 1 \right\}.
\]
This means that epoch $k$ covers the time steps between $c_{\frac{1}{15}(16^k - 16)} + 1$ and $c_{\frac{1}{15}(16^{k+1} - 16)}$. 

\begin{lemma}[Lemma~8 in \citet{wan2024optimalefficientalgorithmsdecentralized}]
\label{extended-lem2}
There exists a collection of nested intervals $\left[\frac{1}{4}, \frac{3}{4}\right] \supseteq I_1 \supseteq I_2 \supseteq \dots \supseteq I_K$ such that the length of the $k$-th interval equals to $|I_k| = 4^{-(k+3)}$ and for every $p \in I_k$,
\begin{align*}
\mathbb{E}_X \left[\left\|\mathcal{A}_t(X) - \xi \mathbf{p}\right\|_2^2\right] \geq \frac{16^{-(k+3)} n \xi^2}{8},
\end{align*}
holds for at least half the rounds $t$ in Epoch $k$, where $\xi = \frac{ (M+1) D\mathbf{p}}{N \sqrt{n}}$.
\end{lemma}

 The statement coincides with that of \cite{wan2024nearly}, up to  minor notation changes. Note that, although the statement provides a specific definition of $\xi$, which is the same as in \cite{wan2024nearly}, the lemma in fact holds for any $\xi > 0$. 
 We do not prove 
\lref{extended-lem2}, since it is proven in \cite{wan2024nearly}.
 To provide some intuition on the proof, observe that its first ingredient is a change-of-measure type of argument. Specifically, Lemma~7 in \citet{wan2024optimalefficientalgorithmsdecentralized}, used in the proof, shows that if we fix two distributions $p$ and $p'$, which in turn determine $\mathbf{p}$ and $\mathbf{p'}$, the expected instantaneous regret
of the local learner $M/2 +1 $ on at least one of the two distributions parameterized by appropriate $p$ and $p'$
must be large. The appropriate distance between $p$ and $p'$ is a function of the index $i$ of the block $c_i$ to which the time step $t$ belongs. More precisely, if \( |p - p'| \) is bounded below by \(2\varepsilon\) and above by \(4\varepsilon\), where the parameter \(\varepsilon\) satisfies 
$
\varepsilon \le \frac{1}{32\sqrt{i+1}},
$
then one can show that the minimal regret between the one induced by distribution \(p\) and the one induced by distribution \(p'\) is lower bounded by
$
\frac{\alpha N}{2} \cdot \frac{d(\xi \varepsilon)^2}{4}.
$
Then, \lref{extended-lem2} follows directly from Lemma~7 in \citet{wan2024optimalefficientalgorithmsdecentralized} combined with a simple dichotomy argument. We refer the interested reader to \cite{wan2024optimalefficientalgorithmsdecentralized} for the complete proof. 

From \lref{extended-lem2}, there exists $p \in \cap_{k=1}^K I_k$ such that:
\begin{align}
\mathbb{E}_{\mathbf{w}_1, \dots, \mathbf{w}_Z}[\Reg_{T}(M/2+1)] &\geq \mathbb{E}_{\mathbf{w}_1, \dots, \mathbf{w}_Z}\left[ \sum_{i=0}^{Z} \sum_{t=c_i+1}^{c_{i+1}} \frac{\alpha N}{2} \left\|x_t(M/2+1) - \frac{(M+1) D\mathbf{p}}{N\sqrt{n}} \right\|_2^2\right] \nonumber \\
&\geq \sum_{k=1}^{K} \sum_{i \in E_k} \sum_{t=c_i+1}^{c_{i+1}} \mathbb{E}_X \left[ \frac{\alpha N}{2} \left\| \mathcal{A}_t(X) - \frac{(M+1) D\mathbf{p}}{N\sqrt{n}} \right\|_2^2 \right] \nonumber \\
&\geq \sum_{k=1}^{K} \frac{\left(c_{\frac{1}{15}(16^{k+1}-16)} - c_{\frac{1}{15}(16^k-16)}\right) 16^{-(k+3)} \alpha (M+1)^2 D^2}{64N} \nonumber \\
&= \frac{16^{-4} \alpha K \tau (M+1)^2 D^2}{4N}\nonumber \\
&\ge \frac{16^{-4} \alpha K \tau M^2 D^2}{4N}.\label{extended-eq13}
\end{align}

From the definitions of $K, Z, \tau$, we get:
\begin{align}
\frac{K \tau M^2}{4N} &\geq \frac{(\log_{16}(15(T - 1)/(M/2+d)) - 2)(M/2+d)M^2}{4N} \nonumber\\
&\geq \frac{(\log_{16}(15(T - 1)/(M/2+d)) - 2)(M/2+d)M}{16}\nonumber
\\ &\geq \frac{(\log_{16}(15(T - 1)/(N+d)) - 2)(N+d)N}{16^2}.\label{extended-eq14}
\end{align}

Recall that for the $N$-cycle graph, we have
$
\frac{1}{1 - \sigma_2(W)} \le \frac{N^2}{2}
$
as established in \eref{eq:rho}. Moreover, the second-smallest eigenvalue of the Laplacian satisfies
$
\sigma_{N-1}(\text{Lap}(G)) = 2 - 2\cos\left(\frac{2\pi}{N}\right) \le \frac{4\pi^2}{N^2} \le \frac{40}{N^2},
$
and the largest eigenvalue is $\sigma_1(\text{Lap}(G)) = 4$. Since $W = I - \frac{1}{4} \text{Lap}(G)$, it follows that
$
\sigma_2(W) = 1 - \frac{1}{4} \sigma_{N-1}(\text{Lap}(G)),
\quad \text{so} \quad
\frac{1}{1 - \sigma_2(W)} \ge \frac{N^2}{10}.
$
Combining this estimate with \eref{extended-eq13} and \eref{extended-eq14}, we conclude that for some realization of $\mathbf{w}_0, \dots, \mathbf{w}_Z$,

\begin{align}
\label{for_dis_low}
\Reg_{T}(M/2+1) \geq 16^{-6} \alpha  N D^2 \left(\frac{\sqrt 2}{\sqrt{1 - \sigma_2(W)}}+d\right) \log_{16}\left(\frac{15(T - 1)}{\frac{\sqrt{10}}{\sqrt{1 - \sigma_2(W)}}+d} -2 \right).
\end{align}

\end{proof}

%%%%%%%%%%%%%%%%%%%%%%%%%%%%%%%%%%%%%%%%%%%%%%%%%%%%%%%%%%%%

\end{document}